\documentclass[opre]{informs4}              

\RequirePackage{bm}
\RequirePackage{endnotes}

\OneAndAHalfSpacedXII

\usepackage{algorithm}
\usepackage{algpseudocode}
\usepackage{tikz}

\usepackage{natbib}
 \bibpunct[, ]{(}{)}{,}{a}{}{,}%
 \def\bibfont{\small}%

\definecolor{dark-blue}{RGB}{0,0,191}
\usepackage[hidelinks,colorlinks=true,citecolor=dark-blue,linkcolor=dark-blue,urlcolor=dark-blue, backref=page]{hyperref}

\usepackage{thmtools}
\usepackage{physics}
\usepackage{booktabs}
\usepackage{mathtools}
\usepackage{xcolor}
\usepackage{caption}
\mathtoolsset{showonlyrefs}

\newcommand{\explain}[1]{\tag*{(#1)}}

\definecolor{mydarkblue}{rgb}{0,0.08,0.45}

\newcommand{\fS}{\mathcal{S}}
\newcommand{\fA}{\mathcal{A}}

\newcommand{\R}{\mathbb{R}}
\newcommand{\E}{\mathbb{E}}
\newcommand{\V}{\mathbb{V}}

\newcommand{\na}{{|\fA|}}

\newcommand{\pdisg}{G^{\text{PDIS}}}

\newcommand{\tb}[1]{{\textbf{#1}}}

\definecolor{revision-brown}{RGB}{133,45,15}

\usepackage{color}

\theoremstyle{plain}

\EquationsNumberedThrough

\TheoremsNumberedBySection
\ECRepeatTheorems

\MANUSCRIPTNO{}                   

\DeclareFontShape{OT1}{cmr}{m}{n}{<-10>gen*cmr<10->cmr10}{}
\DeclareFontShape{OT1}{cmr}{m}{it}{<5><6><7>cmti7<8>cmti8<9>cmti9<10->cmti10}{}
\DeclareFontShape{OT1}{cmr}{m}{sc}{<5->cmcsc10}{}
\DeclareFontShape{OT1}{cmr}{bx}{n}{<5><6><7><8><9>gen*cmbx<10->cmbx10}{}
\DeclareFontShape{OT1}{cmr}{bx}{it}{<->cmbxti10}{}
\OneAndAHalfSpacedXI
\let\AppendixFontSize\normalsize

\makeatletter
\def\@maketitle{%
  \newpage
  \theARTICLETITLE
  \theARTICLEAUTHORS\removelastskip
  \theARTICLEABSTRACT}
\def\theARTICLETITLE{%
  \HOOKtop
  \begin{center}%
  \vspace*{0pt}%
  \TITLEfont\HD{24}{0}\theTITLE\HD{0}{15}%
  \end{center}}
\def\AUTHOR#1{\begin{center}\AUTHORfont\HD{15}{0}#1\HD{0}{6}\relax\end{center}}
\def\AFF#1{\begin{center}\AFFfont{#1}\relax\vskip1.6pt\end{center}}
\long\def\ABSTRACT#1{\long\gdef\theABSTRACT{#1}}
\def\theARTICLEABSTRACT{%
  \HOOKb
  \vspace*{18pt}%
  {\centering
   \begin{minipage}{\ABSwidth}\parindent1em \ABSfont
     \noindent\theABSTRACT\endgraf
     \vskip5pt
     \theFUNDING
     \theKEYWORDS
     \theSUBJECTCLASS
     \theAREAOFREVIEW
     \theMSCCLASS
     \theORMSCLASS
     \if@BLINDREV\else\theHISTORY\fi
   \end{minipage}\par}%
  \vspace*{10pt}%
  \noindent\rule{\textwidth}{1pt}%
  \vspace*{20pt}}
\def\KEYWORDSname{{\it Key words\/}{\kern0.7pt}:\enskip}
\def\setoddRH{\hbox to \textwidth{\fs.7.8.\tabcolsep0pt
  \begin{tabular*}{\textwidth}[b]{l@{\extracolsep\fill}r}
  {\theRRHFirstLine}&\raisebox{0pt}[0pt][0pt]{\fs.10.10.\thepage}\\[-4pt]
  \rlap{\VRHDW{0.5pt}{0pt}{\textwidth}}&\\
  \end{tabular*}}}
\def\setevenRH{\hbox to \textwidth{\fs.7.8.\tabcolsep0pt
  \begin{tabular*}{\textwidth}[b]{l@{\extracolsep\fill}r}
  \raisebox{0pt}[0pt][0pt]{\fs.10.10.\thepage}&{\theLRHFirstLine}\\[-4pt]
  \rlap{\VRHDW{0.5pt}{0pt}{\textwidth}}&\\
  \end{tabular*}}}
\def\section{\@startsection{section}{1}{\z@}{-15pt plus -5pt minus -4pt}{0.01pt}%
    {\fs.13.15.\bf\RAGG}}
\def\subsection{\@startsection{subsection}{2}{\z@}{-12pt plus -4pt minus -2pt}{0.01pt}%
    {\fs.11.{13.6}.\bf\RAGG}}
\def\@seccntformat#1{{\upshape\csname the#1\endcsname.}\hskip1em\relax}
\def\TheoremHeaderFont{\rmfamily\mdseries\scshape}
\def\ExampleHeaderFont{\rmfamily\mdseries\scshape}
\makeatother

\begin{document}

\RUNAUTHOR{Liu et al.}

\RUNTITLE{Efficient Policy Evaluation with Offline Data Informed Behavior Policy Design}

\TITLE{Efficient Policy Evaluation with Offline Data Informed Behavior Policy Design}

\ARTICLEAUTHORS{%
\AUTHOR{Shuze Daniel Liu}
\AFF{Department of Computer Science,
University of Virginia, \EMAIL{shuzeliu@virginia.edu}}

\AUTHOR{Claire Chen}
\AFF{The Division of Physics,
Mathematics and Astronomy, California Institute of Technology, \EMAIL{clairechen@caltech.edu}}

\AUTHOR{Will Ma}
\AFF{Graduate School of Business and Data Science Institute, Columbia University, \EMAIL{wm2428@gsb.columbia.edu}}

\AUTHOR{Shangtong Zhang}
\AFF{Department of Computer Science,
University of Virginia, \EMAIL{shangtong@virginia.edu}}

}

\ABSTRACT{%
Online Monte Carlo evaluation is a fundamental tool for assessing policy performance in reinforcement learning and sequential decision-making problems arising in operations research. However, achieving accurate estimates often requires extensive online interaction with the environment, which can be costly or impractical in many real-world settings. In this paper, we develop a framework that improves the sample efficiency of online Monte Carlo estimators while preserving unbiasedness. We first derive a closed-form optimal behavior policy that minimizes estimator variance under unbiasedness constraints. We then propose practical algorithms for learning the proposed behavior policy from previously collected offline data, enabling improved online evaluation without requiring estimation of the environment transition model. We provide theoretical analysis that quantifies the resulting variance reduction and analyzes the impact of approximation errors. Empirical studies across diverse environments demonstrate substantial improvements in online sample efficiency compared with standard on-policy Monte Carlo evaluation and existing baseline methods. Our results provide a unified framework for optimal behavior policy design in off-policy evaluation, with applications to reinforcement learning and operations research.
}%


\maketitle

\section{Introduction}

Many operations research (OR) problems involve sequential decision making under uncertainty, including inventory management, healthcare scheduling, and online stochastic matching. These problems are commonly modeled using Markov decision processes (MDPs) and solved via stochastic dynamic programming and approximate dynamic programming methods \citep{puterman1990markov,powell2007approximate, bertsekas2012dynamic}. Reinforcement learning (RL) can be viewed as an extension of this paradigm by enabling data-driven optimization when system dynamics are unknown or difficult to model explicitly \citep{sutton2018reinforcement}. As a result, RL has increasingly been applied to classical OR domains such as dynamic pricing and resource allocation, providing learning-based approaches to classical stochastic control problems \citep{gijsbrechts2022can, wang2024reliable}.

When applying reinforcement learning algorithms,\textit{ policy evaluation} plays a central role. It enables practitioners to assess the performance of candidate policies prior to deployment and to compare alternative algorithmic designs.
Exact policy evaluation (via dynamic programming) is only reasonable if the MDP is known and state space is small; consequently,
a widely used approach is on-policy Monte Carlo evaluation, in which the target policy is evaluated by testing it many times in the environment and taking the average performance.
Crucially and somewhat surprisingly, directly executing the target policy itself has been shown to be suboptimal \citep{hanna2017data, zhong2022robust, liu2024efficient}, often resulting in high-variance estimates for the expected performance.
It can be better instead to execute a “behavioral policy” that places a higher probability on actions that contribute more strongly to the estimation objective, while correcting for distribution mismatch using importance sampling ratios.

On the other hand, offline reinforcement learning can also be used for policy evaluation by leveraging previously collected samples of interaction with the environment. However, these past decisions may have been generated by a variety of arbitrary policies, and the resulting data distribution may not adequately cover the state space or align with the target policy, which can introduce evaluation bias that is difficult to eliminate \citep{levine2020offline}.
Consequently, policies that perform well under offline evaluation may not translate into strong real-world performance. As a result, both online and offline RL workflows continue to rely heavily on online Monte Carlo evaluation for reliable algorithm comparison and hyperparameter tuning \citep{kalashnikov2018scalable, vinyals2019grandmaster}.

Therefore, improving the sample efficiency of online Monte Carlo estimators while preserving unbiasedness is a central challenge in reinforcement learning. Yet, this perspective has received less attention in sequential decision-making problems within operations research, where directly executing the target policy remains the standard practice. Motivated by this gap,
the focus of this paper is on how to \textit{optimally} evaluate a given policy using Monte Carlo methods, by deriving a variance-minimizing behavior policy that preserves unbiasedness under importance sampling. This approach can substantially reduce the number of online samples required, leading to shorter experimentation durations (e.g., when online samples require real-world A/B testing), lower computational cost (e.g., when samples rely on expensive simulators), and improved overall efficiency of Monte Carlo evaluation. For example, in online stochastic matching, evaluating a candidate allocation policy may require many costly online interactions before its value can be estimated accurately. Our behavior-policy design reduces the variance of this evaluation while preserving unbiasedness, thereby reducing the required online interactions.
Our main contributions are summarized as follows:
\begin{enumerate}
    \item Optimal behavior policy design. We formulate the problem of variance reduction for unbiased Monte Carlo policy evaluation as a constrained optimization problem and derive a closed-form optimal behavior policy that minimizes estimator variance while preserving unbiasedness (Theorem~\ref{lem:rl-optimal}).
    \item Offline-data-informed learning framework.
    We develop a practical method for learning the proposed one-step optimal behavior policy (Theorem~\ref{thm: one step optimal}) from previously collected offline data (Theorem~\ref{lem:hat-q-recursive}). This enables improved online evaluation without requiring estimation of the environment transition model.
    \item Theoretical guarantees. We prove that the exact closed-form behavior policies achieve no greater evaluation variance than standard online Monte Carlo evaluation while preserving unbiasedness. We further provide theoretical analysis that quantifies the resulting variance reduction (Theorem~\ref{lem:var_smaller_stronger_optimal} \& Theorem~\ref{lem:one_step_var_smaller_stronger_optimal}) and characterizes the impact of approximation errors on estimation performance (Theorem~\ref{lem:error analysis}).
    \item Online sample complexity. Under a high-probability relative-accuracy condition on the learned auxiliary value function, we quantify the retained local and trajectory-level variance reduction and derive a high-probability online sample-complexity bound (Theorem~\ref{thm: robust local} \& Theorem~\ref{thm: end to end}).
    \item Empirical validation. We demonstrate the effectiveness of our approach across a diverse set of environments, showing substantial improvements in online sample efficiency compared to standard on-policy Monte Carlo evaluation and prior best-performing methods.
\end{enumerate}

\paragraph{Roadmap.}
The remainder of the paper is organized as follows. Section~\ref{sec related work} discusses related literature and positions our contributions within the broader reinforcement learning and operations research literature. Section~\ref{sec: background} introduces the problem formulation and reviews background on online Monte Carlo policy evaluation and importance sampling. Section~\ref{sec:var-stats} presents the variance reduction formulation and derives the closed-form optimal behavior policy in statistics. Section~\ref{sec:variance} extends the variance analysis to sequential reinforcement learning settings and provides theoretical guarantees for variance reduction and unbiasedness. Section~\ref{sec: learn behavior policies} describes the offline-data-informed learning framework for estimating the optimal behavior policy from previously collected data, and Section~\ref{sec: end to end} quantifies how approximation error in the learned auxiliary value function affects retained variance reduction and online sample complexity. Section~\ref{sec:experiment} reports empirical results demonstrating the improved evaluation efficiency of our proposed method across multiple environments. Finally, Section~\ref{sec: conclusion} concludes and discusses future research directions.

This paper builds upon a conference version presented at the 41st International Conference on Machine Learning (ICML) \citep{liu2024efficient} and substantially extends its theoretical analysis and empirical evaluation. First, we provide a more explicit characterization of how the optimal behavior policies reduce variance relative to on-policy Monte Carlo. For the globally optimal behavior policy, we characterize total variance reduction as the sum of current- and future-step reductions (Theorem~\ref{lem:var_smaller_stronger_optimal}). For the one-step optimal behavior policy, we derive an exact variance-reduction expression relative to on-policy Monte Carlo, whereas the conference version provided only a recursive lower bound (Theorem~\ref{lem:one_step_var_smaller_stronger_optimal}).
Second, we analyze variance reduction in terms of relative estimation error in the learned auxiliary value function. We derive a lower bound on the retained local variance reduction (Theorem~\ref{thm: robust local}) and a sufficient condition for retaining a prescribed fraction of the ideal local variance reduction (Corollary~\ref{cor: retained fraction}). Third, we extend this analysis to the trajectory level (Theorem~\ref{thm: certificate}) and derive a high-probability online sample-complexity bound for achieving a prescribed evaluation accuracy and confidence level (Theorem~\ref{thm: end to end}).
Finally, we extend the empirical evaluation to online stochastic
matching with heterogeneous resources under both single- and
multi-capacity settings, comparing online sample efficiency and
estimator variance against on-policy Monte Carlo and existing
behavior-policy baselines.

\section{Related Work}
\label{sec related work}
\tb{Estimator Design for Off-Policy Evaluation.} Off-policy evaluation (OPE) has been extensively studied as a central problem in reinforcement learning and stochastic control, where the objective is to estimate the performance of a target policy using data collected under a different behavior policy. Early approaches rely on importance sampling (IS), which provides unbiased estimates by correcting distribution mismatch via likelihood ratio weighting \citep{precup2000eligibility}. However, IS estimators often suffer from high variance, particularly in long-horizon settings \citep{thomas2016data, liu2018breaking}.
Subsequent work proposes variance-stabilized variants such as weighted or per-decision importance sampling \citep{thomas2016data, precup2001off}, as well as doubly robust estimators that combine importance weighting with control variates to reduce variance while preserving unbiasedness \citep{jiang2015doubly, liu2024doubly}. More recent research explores marginalized importance sampling methods that reduce variance growth by replacing trajectory-level likelihood ratios with ratios of state visitation distributions, thereby mitigating the exponential variance scaling of standard importance sampling \citep{xie2019towards}. Additional advances develop statistically efficient OPE estimators by combining importance sampling with auxiliary regression functions, such as value-function or density-ratio estimates \citep{kallus2019intrinsically,kallus2020double}.
Despite significant progress in estimator design, most prior work treats the behavior policy as fixed and focuses primarily on improving the estimator itself. Our work instead considers optimizing the data-collecting behavior policy, providing a complementary perspective that can be combined with a broad class of existing OPE estimators.

\tb{Behavior Policy Design for Off-Policy Evaluation.}
In addition to estimator design, another line of research investigates reducing evaluation variance by optimizing the data-collecting policy, often referred to as the behavior policy. Early work by \citet{hanna2017data} formulates behavior policy selection as an optimization problem and proposes gradient-based updates to learn variance-reducing policies. This line of work highlights that the target policy itself is not optimal for data collection, and that carefully chosen behavior policies can substantially improve evaluation efficiency. Subsequent studies explore adaptive behavior policies that focus sampling effort on under-represented regions of the state space to further reduce estimator variance \citep{zhong2022robust}. Complementarily, \citet{mukherjee2022revar} study behavior policies designed to reduce variance for per-decision importance sampling estimators, focusing on tree-structured MDPs where transition dynamics are assumed known.

Taken together, these approaches highlight the potential of optimizing data-collection strategies to improve off-policy evaluation performance. However, existing methods typically rely on iterative policy search procedures \citep{hanna2017data}, assumptions on data coverage or similarity between behavior and target policies \citep{zhong2022robust}, or structural constraints such as tree-structured environments or access to transition dynamics \citep{mukherjee2022revar}. Our work advances this line of research by deriving a closed-form optimal behavior policy that can be learned directly from offline data without requiring transition model estimation. We provide theoretical guarantees on unbiasedness and variance reduction, and empirically demonstrate improved evaluation efficiency.

\section{Background}
\label{sec: background}
We consider a finite horizon Markov Decision Process (MDP, \citet{puterman2014markov}) with a possibly infinite or continuous state space $\fS$, a finite action space $\fA$,
a reward function $r: \fS \times \fA \to \R$,
a transition distribution $p(\cdot\mid s,a)$ on $\fS$,
an initial distribution $p_0$ on $\fS$,
and a constant horizon length $T$. We allow general state spaces and assume a finite action space, following common formulations in off-policy evaluation
\citep{dudik2011doubly,le2019batch}.
For general state spaces, successor-state sums denote expectations under the transition distribution.
To simplify notation, we work with the undiscounted setting without loss of generality.
The results generalize to the discounted setting \citep{puterman2014markov} as long as the horizon is fixed and finite.
For any integer $n$,
we define as shorthand $[n] \doteq \qty{0, 1, \dots, n}$. Throughout, $\E$ and $\V$ denote expectation and variance,
respectively, with subscripts indicating the source of randomness.

At time step 0,
an initial state $S_0$ is sampled from $p_0$.
At time step $t \in [T-1]$,
an action $A_t$ is drawn from $\pi_t(\cdot \mid S_t)$,
where $\pi_t: \fA \times \fS \to [0, 1]$ is the policy at time step $t$.
A finite reward $R_{t+1} \doteq r(S_t, A_t)$ is received, and a successor state $S_{t+1}$ is sampled from $p(\cdot \mid S_t, A_t)$.
We define abbreviations $\pi_{i:j} \doteq \qty{\pi_i, \pi_{i+1}, \dots, \pi_j}$ and $\pi \doteq \pi_{0:T-1}$.
The return at time step $t$ is defined as
$
  G_t \doteq \sum_{i={t+1}}^T R_i,
$
which allows defining the state- and action-value functions as
$
v_{\pi, t}(s) \doteq \E_{\pi}\left[G_t \mid S_t = s\right]$ and $
q_{\pi, t}(s, a) \doteq \E_{\pi}\left[G_t \mid S_t = s, A_t = a\right].
$
We use the total rewards performance metric \citep{puterman2014markov} to measure the performance of the policy $\pi$,
which is defined as $J(\pi)\doteq\E_{S_0\sim p_0}[v_{\pi,0}(S_0)]$.
Here we use the Monte Carlo
methods introduced by \citet{kakutani1945markoff} to estimate the total rewards $ J(\pi)$.
A common and straightforward implementation samples $J(\pi)$ by executing the policy $\pi$ online.
With increasing sample size,
the average return converges to $J(\pi)$.
This idea is called on-policy learning (\citealt{sutton1988learning}) because it estimates a policy $\pi$ by executing itself.

We now turn to off-policy learning:
the total rewards $J(\pi)$ of a policy of interest $\pi$,
called the \textit{target policy},
are estimated by executing another policy $\mu$,
called the \textit{behavior policy}.
In off-policy learning, each trajectory
$
\textstyle \qty{S_0, A_0, R_1, S_1, A_1, R_2, \dots, S_{T-1}, A_{T-1}, R_T}
$
is generated by a behavior policy $\mu$ with
$
  S_0 \sim p_0, A_{t} \sim \mu_{t}(\cdot | S_{t}), \, t  \in [T-1].
$
Let
$$
  \tau^{\mu_{t:T-1}}_{t:T-1} \doteq \qty{S_t, A_t, R_{t+1}, \dots, S_{T-1}, A_{T-1}, R_{T}}
$$
be a shorthand for a segment of a trajectory generated by the behavior policy $\mu$ from the time step $t$ to the time step $T-1$ inclusive.
Off-policy estimation reweights the rewards collected by $\mu$ using importance sampling ratios to estimate $J(\pi)$.
The importance sampling ratio at time step $t$ is defined as
$
\textstyle \rho_t \doteq \frac{\pi_t(A_t \mid S_t)}{\mu_t(A_t \mid S_t)}.
$
The product of importance sampling ratios from time $t$ to $t' \geq t$ is defined as
$
\textstyle \rho_{t:t'} \doteq \prod_{k=t}^{t'} \frac{\pi_k(A_k | S_k)}{\mu_k(A_k | S_k)}.
$
There are various ways to use the importance sampling ratios in off-policy learning \citep{geweke1988antithetic, hesterberg1995weighted, koller2009probabilistic,thomas2015safe}.
We start with the per-decision importance sampling estimator (PDIS, \citet{precup:2000:eto:645529.658134}) here; other estimators remain a direction for future work.
The PDIS Monte Carlo estimator is defined as
\begin{align}
  \label{eq pdis def}
\textstyle \pdisg(\tau^{\mu_{t:T-1}}_{t:T-1}) \doteq \sum_{k=t}^{T-1} \rho_{t:k} R_{k+1}
\end{align}
and is unbiased
for any behavior policy $\mu$  that covers target policy $\pi$ \citep{precup:2000:eto:645529.658134}.
In other words,
when $\forall s$, $\forall a$,
$\mu_t(a|s) = 0 \implies \pi_t(a|s)=0$, we have $\forall t$, $\forall s$,
$$\E[ \pdisg(\tau^{\mu_{t:T-1}}_{t:T-1}) \mid S_t = s ]  = v_{\pi,t}(s).$$
We extensively use the recursive form of the PDIS estimator:
\begin{align}\label{eq:PDIS-recursive}
&\pdisg(\tau^{\mu_{t:T-1}}_{t:T-1}) \\
=&\begin{cases}
\rho_t \left(R_{t+1} + \pdisg(\tau^{\mu_{t+1:T-1}}_{t+1:T-1})\right) & t \in [T-2], \\
\rho_tR_{t+1} & t = T-1.
\end{cases}
\end{align}
Unbiasedness of the PDIS estimator means that
a reduction in variance improves sample efficiency.
We obtain this reduction through the design and learning of behavior policies.

\section{Variance Reduction in Statistics} \label{sec:var-stats}

In this section,
we provide the mathematical foundation for variance reduction with importance sampling ratios.
Its notation is independent of that used elsewhere in the paper.
We use similar notations only for easy interpretation in later sections.

Consider a discrete random variable $A$ taking values from a finite space $\fA$ according to a probability mass function $\pi:\fA \to [0,1]$
and a function $q:\fA \to \R$ mapping a value in $\fA$ to a real number.
Our goal is to estimate
$$\E_{A\sim \pi}[q(A)].$$
The ordinary Monte Carlo methods  sample $\qty{A_1, \dots, A_N}$ from $\pi$ and use the empirical average
$$\frac{1}{N}\sum_{i=1}^N q(A_i)$$
as the estimate, which is unbiased but may suffer from high variance.

In statistics, importance sampling is introduced as a variance reduction technique for Monte Carlo methods (\citealt{Rubinstein1981Simulation}).
The main idea is
to sample $\qty{A_1, A_2,  \dots, A_N}$ from a different distribution $\mu$
and use
 $\frac{1}{N}\sum_{i=1}^N \rho(A_i) q(A_i)$
as the estimate,
where
$
\textstyle  \rho(A) \doteq \frac{\pi(A)}{\mu(A)}
$
is the importance sampling ratio.
Assuming $\mu$ covers $\pi$, i.e.,
\begin{align}
\label{eq stats converage}
\forall a, \mu(a) = 0 \implies \pi(a) = 0,
\end{align}
the importance sampling ratio weighted empirical average is then unbiased,
$$\E_{A\sim \pi}[q(A)] = \E_{A\sim \mu}[\rho(A)q(A)].$$

With importance sampling, if the sampling distribution $\mu$ is carefully designed,
the variance can be reduced while maintaining unbiasedness.
To ensure unbiasedness, one natural choice presented by \citet{Rubinstein1981Simulation} and \citet{sutton2018reinforcement} is to optimize the sampling distribution within
\begin{align}
\Lambda^- \doteq \qty{\mu \in \Delta(\fA) \mid \forall a, \mu(a) = 0 \implies \pi(a) = 0}.  \label{eq stats search space small}
\end{align}
In other words, we aim to find a variance-minimizing sampling distribution among all distributions that cover $\pi$.
We formulate
this problem of searching for a variance-reducing sampling distribution as an optimization problem:
\begin{align}
\text{min}_{\mu}  \quad &
\V_{A\sim \mu}(\rho(A)q(A)) \label{eq:math-optimization1} \\
\text{s.t.}& \quad \mu \in \Lambda^-.
\end{align}

In our work,
we enlarge $\Lambda^-$ to $\Lambda$ defined as
\begin{align}
\!\!\!\!\!\! \Lambda \doteq \qty{\mu \in \Delta(\fA) \mid \forall a, \mu(a) = 0 \implies \pi(a)q(a) = 0}.
  \label{eq stats search space}
\end{align}
The space $\Lambda$ weakens the assumption in \eqref{eq stats search space small}.  \citet{mcbook} proves that any distribution $\mu$ in $\Lambda$ gives unbiased estimation,
though $\mu$ may not cover $\pi$.
\begin{lemma}[Unbiasedness]
\label{lem stats unbiasedness}
$\forall \mu \in \Lambda, \E_{A\sim\mu}\left[\rho(A)q(A)\right] = \E_{A\sim\pi}\left[q(A)\right].$
\end{lemma}
For completeness, its proof is in Appendix~\ref{sec proof lem stats unbiasedness}.
We now consider the variance minimization problem on $\Lambda$, i.e.,
\begin{align}
\text{min}_{\mu}  & \quad
\V_{A\sim \mu}(\rho(A)q(A)) \label{eq:math-optimization} \\
\text{s.t.}& \quad \mu \in \Lambda.
\end{align}
The following lemma
gives an optimal solution $\mu^*$ to this optimization problem.
\begin{lemma}[Variance Reduction]\label{lem:math-optimal}
$\mu^*(a) \propto \pi(a)\abs{q(a)}$ is an optimal solution to \eqref{eq:math-optimization}.
\end{lemma}
For completeness, its proof is detailed in Appendix~\ref{append:math-optimal}.
Here, by $$\mu(a) \propto \pi(a) w(a)$$ with some non-negative $w(a)$,
we mean
\begin{align}
\textstyle \mu(a) \doteq \frac{\pi(a) w(a)} { \sum_b \pi(b) w(b)}.
\end{align}
The reader may notice that if $\pi(a)w(a) = 0$ for all $a$,
this reweighting is not well defined.
We then use the convention to interpret $\mu(a)$ as a uniform distribution, i.e., $\mu(a) = 1/\na$.
\emph{This convention applies to $\propto$ in the rest of the paper to simplify the presentation.}

At a high level, $\mu^*$ is optimal because choosing the sampling distribution proportional to $\pi(a)\lvert q(a)\rvert$ equalizes the magnitude of the importance-weighted term $\rho(a)q(a)$ across samples. Intuitively, $\mu^*$ allocates more probability mass to samples that contribute a large magnitude to the estimator once selected, which is where variance would come from. As a result, among all unbiased sampling distributions in $\Lambda$, this choice minimizes (and in some cases drives to zero) the variance of the importance-sampling estimator.
The following remark gives an example of the optimality of $\mu^*$,
whose proof is in Appendix~\ref{append:math-variance-0}.
\begin{remark}\label{lem:math-variance-0}
If $\forall a \in \fA, q(a) \geq 0$ or $\forall a \in \fA, q(a) \leq 0$, $\mu^*$ defined in Lemma~\ref{lem:math-optimal} ensures unbiasedness and gives
$$\V_{A\sim \mu^*}(\rho(A)q(A))  = 0.$$
\end{remark}

\section{Variance Reduction in Reinforcement Learning} \label{sec:variance}

The previous section shows how to use importance sampling to minimize the variance of the next-step reward in policy evaluation.
In sequential reinforcement learning, rather than minimizing the variance of the next-step reward, we aim to minimize the variance of the whole trajectory.
\subsection{Globally Optimal Behavior Policy}
We seek to reduce the variance $\V\left(\pdisg(\tau^{\mu_{0:T-1}}_{0:T-1})\right)$
by designing a proper behavior policy $\mu$.
Meanwhile,
we need to ensure that the PDIS estimator with this behavior policy is unbiased. We formalize this variance reduction problem as follows.
\begin{align}
\label{eq rl opt problem unbiased}
\textstyle \min_{\mu} & \quad \V\left(\pdisg(\tau^{\mu_{0:T-1}}_{0:T-1})\right) \\
\text{s.t.}& \quad \E\left[\pdisg(\tau^{\mu_{0:T-1}}_{0:T-1})\right] = J(\pi).
\end{align}

To ensure that the off-policy estimator
is unbiased, the conventional reinforcement learning wisdom \citep{precup:2000:eto:645529.658134,maei2011gradient,sutton2016emphatic,zhang2022thesis}
requires that the behavior policy $\mu$ covers the target policy $\pi$.
In other words, they restrict the policy space of $\mu$ to be in a set
\begin{align}
\Lambda^- \doteq &
\{ \mu \mid
\forall t, s, a, \mu_t(a|s) = 0 \implies \pi_t(a|s) = 0\} \\
=& \{ \mu \mid
\forall t, s, a, \pi_t(a|s) \neq 0 \implies \mu_t(a|s) \neq 0\} .
\end{align}
The set $\Lambda^-$ contains all policies that satisfy the policy coverage constraint in off-policy learning (\citealt{sutton2018reinforcement}).
With the constraint of policy coverage, the optimization problem \eqref{eq rl opt problem unbiased} can be formulated as
\begin{align}
\label{eq rl opt problem unbiased lambda -}
\textstyle \min_{\mu} & \quad \V\left(\pdisg(\tau^{\mu_{0:T-1}}_{0:T-1})\right) \\
\text{s.t.}& \quad \mu \in \Lambda^-.
\end{align}

In this paper, rather than relying on the classic policy coverage constraint, we expand the search space of $\mu$ from the conventional policy space $\Lambda^-$ to a large policy space $\Lambda^*$.
To ensure the expanded policy space is large enough and the off-policy estimator is unbiased, we require $\Lambda^*$ to have two properties.

\begin{enumerate}
\item (Broadness) $\Lambda^*$ must be large enough to include all policies satisfying the classic policy coverage constraint  \citep{precup:2000:eto:645529.658134, sutton2018reinforcement}. Mathematically,
\begin{align}\label{eq: Lambda general}
\Lambda^- \subseteq \Lambda^*.
\end{align}

\item (Unbiasedness)
Every behavior policy $\mu \in \Lambda^*$ must be well behaved in the sense that data generated under $\mu$ can be used by the off-policy estimator to yield unbiased estimates at every state $s$ and time step $t$. Formally, for all $\mu \in \Lambda^*$,
\begin{align}\label{eq: rl unbiasedness}
\forall  t, \forall  s, \quad \E\left[\pdisg(\tau^{\mu_{t:T-1}}_{t:T-1}) \mid S_t = s\right] = v_{\pi, t}(s).
\end{align}
\end{enumerate}

The policy space $\Lambda^*$ will be defined shortly.
We now formulate our problem as
\begin{align}
\label{eq rl opt problem}
\textstyle \min_{\mu} &  \quad \V\left(\pdisg(\tau^{\mu_{0:T-1}}_{0:T-1})\right)
 \\
\text{s.t.}& \quad \mu \in \Lambda^*.
\end{align}
By the law of total variance,
for any $\mu \in \Lambda^*$,
we decompose the variance of the PDIS estimator as
\begin{align}
\label{eq:varaince-1}
&\V\left(\pdisg(\tau^{\mu_{0:T-1}}_{0:T-1})\right) \\
=& \E_{S_0}\left[\V\left(\pdisg(\tau^{\mu_{0:T-1}}_{0:T-1}) \mid S_0\right)\right] + \V_{S_0}\left(\E\left[\pdisg(\tau^{\mu_{0:T-1}}_{0:T-1}) \mid S_0\right]\right) \\
=& \E_{S_0}\left[\V\left(\pdisg(\tau^{\mu_{0:T-1}}_{0:T-1}) \mid S_0\right)\right] + \V_{S_0}\left(v_{\pi, 0}(S_0)\right)  \explain{by \eqref{eq: rl unbiasedness}}.
\end{align}
The second term $\V_{S_0}\left(v_{\pi, 0}(S_0)\right)$ is fixed by the target policy $\pi$ and is unrelated to the choice of $\mu$.
In the first term,
the expectation is taken over $S_0$, which is determined by the initial probability distribution $p_0$.
Consequently,
to solve the problem~\eqref{eq rl opt problem},
it is sufficient to solve,
\begin{align}
\label{eq rl opt problem2}
\min_{\mu} & \quad \V\left(\pdisg(\tau^{\mu_{t:T-1}}_{t:T-1})\mid S_t = s\right) \\
\text{s.t.}& \quad \mu \in \Lambda^*
\end{align}
for all $s$ and $t$.
We define the conditional variance of the successor-state value
at the current state-action pair $(s,a)$ as $\omega_{\pi, t}(s, a)$.
We have $\omega_{\pi, t}(s, a) = 0$ for $t = T-1$ and otherwise
\begin{align}\label{def:nu}
&\!\!\!\!\!\!\!\!\!\!\! \omega_{\pi,t}(s, a) \doteq\V_{S_{t+1}}\left(v_{\pi, t+1}(S_{t+1})\mid S_t=s, A_t=a\right).
\end{align}
We now construct a behavior policy $\mu^*$ as
\begin{align}
\label{def mu star}
\mu_t^*(a|s)  \propto \textstyle \pi_{t}(a|s) \sqrt{u_{\pi, t}(s, a)},
\end{align}
where
$u_{\pi, t}(s, a) \doteq q_{\pi, t}^2(s, a)$ for $t=T-1$ and otherwise
\begin{align}\label{def u}
&u_{\pi, t}(s, a) =  q_{\pi, t}^2(s, a) + \omega_{\pi, t}(s, a) + \textstyle{\sum_{s'} p(s'|s, a)\V\left(\pdisg(\tau^{\mu^*_{t+1:T-1}}_{t+1:T-1}) \mid S_{t+1} = s'\right)}.
\end{align}

Notably,
$\mu^*_t$ and $u_{\pi, t}$ are defined backwards and alternatively,
i.e.,
they are defined in the order of $u_{\pi, T-1}, \mu^*_{T-1}, u_{\pi, T-2}, \mu^*_{T-2}, \dots, u_{\pi, 0}, \mu^*_{0}$.
We now break down each term in $u_{\pi, t}(s, a)$.
\begin{enumerate}
\item $q^2_{\pi, t}(s, a)$ is the squared state-action value function. This term is non-negative. Its magnitude is controlled by the conventional state-action value function $q$.
\item $\omega_{\pi,t}(s, a)$ defined in \eqref{def:nu} is the variance of the state value for the next state. This term is always non-negative by the definition of variance. Its magnitude is controlled by the stochasticity of the environment (i.e., transition function $p$) and the difference within the state value function $v$ for next states.
\item $\textstyle{\sum_{s'} p(s'|s, a)\V\left(\pdisg(\tau^{\mu^*_{t+1:T-1}}_{t+1:T-1}) \mid S_{t+1} = s'\right)}$ is the expected future variance given the current state $s$ and action $a$. This term is also non-negative by the definition of variance. Its magnitude is controlled by the variance of the full trajectory in the future.
\end{enumerate}
From the above analysis, $u_{\pi, t}(s, a)$ is non-negative because it is the sum of non-negative terms. Thus, $\sqrt{u_{\pi, t}(s, a)}$ is always well-defined.
In \eqref{def mu star},
$\mu_t^*(a|s) \propto \pi_t(a|s) \sqrt{u_{\pi, t}(s, a)}$
means $$
\mu_t^*(a|s)  = \frac{\pi_t(a|s) \sqrt{u_{\pi, t}(s, a)}}{\sum_b \pi_t(b|s) \sqrt{u_{\pi, t}(s, b)}}.$$
If $\forall a, \pi_t(a|s) \sqrt{u_{\pi, t}(s, a)} = 0$,
the denominator is zero.
In this case, we use the convention to interpret $\mu_t^*(a|s)$ as a uniform distribution, i.e., $\forall a, \mu_t^*(a|s) = 1/\na$.
Now, we define the expanded space $\Lambda^*$ as
\begin{align}\label{def: Lambda star}
\Lambda^* \doteq& \qty{\mu \mid \forall t, s, a, \mu_t(a|s) = 0 \implies \pi_t(a|s)u_{\pi, t}(s, a)  = 0}.
\end{align}
We prove that this policy space $\Lambda^*$ satisfies the broadness  \eqref{eq: Lambda general} and the unbiasedness \eqref{eq: rl unbiasedness} properties.
\begin{lemma}[Broadness]\label{lemma: Lambda broadness u}
$\Lambda^- \subseteq \Lambda^*.$
\end{lemma}
Its proof is in Appendix \ref{appendix: Lambda broadness u}.
\begin{restatable}[Unbiasedness]{lemma}{reOOLambdaOOstarOOunbiasedness}
\label{lemma: Lambda star unbiasedness}
$\forall \mu \in \Lambda^*$, $\forall  t, \forall  s$,
 $\E\left[\pdisg(\tau^{\mu_{t:T-1}}_{t:T-1}) \mid S_t = s\right] = v_{\pi, t}(s).$
\end{restatable}
Its proof is in Appendix \ref{appendix: Lambda star unbiasedness}. After confirming the broadness and unbiasedness of the space $\Lambda^*$, we now prove that the behavior policy $\mu^*$ is the optimal solution for the optimization problem \eqref{eq rl opt problem2}.

\begin{restatable}[Optimal Behavior Policy]{theorem}{restaterloptimal}
\label{lem:rl-optimal}
For any $t$ and $s$,
the behavior policy $\mu^*$ defined in \eqref{def mu star} is an optimal solution to the following problem
\begin{align}
\min_{\mu} & \quad \V\left(\pdisg(\tau^{\mu_{t:T-1}}_{t:T-1})\mid S_t = s\right) \\
\text{s.t.}& \quad \mu \in \Lambda^*
\end{align}
\end{restatable}
\begin{proof}{Proof Sketch of Theorem \ref{lem:rl-optimal}.} Details are deferred in Appendix~\ref{append:rl-optima}.
The proof relies on backward induction and the recursive structure of variance.
First, using the law of total variance, we establish a recursive relationship for the variance of the estimator at time $t$ (Lemma~\ref{lem:recursive-var}),
\begin{align}
&\V\left(\pdisg(\tau^{\mu_{t:T-1}}_{t:T-1})\mid S_t\right) \\
=& \E_{A_t\sim \mu_t}\left[\rho_t^2 \left(\E_{S_{t+1}}\left[\V\left(\pdisg(\tau^{\mu_{t+1:T-1}}_{t+1:T-1})\mid S_{t+1}\right) \mid S_t, A_t\right] + \omega_{\pi,t}(S_t, A_t) + q_{\pi, t}^2(S_t, A_t)\right) \mid S_t\right] \\
&- v_{\pi, t}^2(S_t).
\end{align}

With this recursive expression, we then proceed via induction. When $t = T-1$, we have
\begin{align}
    \V\left(\pdisg(\tau^{\mu_{T-1:T-1}}_{T-1:T-1}) \mid S_{T-1} = s\right)
    &= \V_{A_{T-1}}\left(\rho_{T-1} r(s, A_{T-1})\mid S_{T-1} =s \right) \\
    &= \V_{A_{T-1}}\left(\rho_{T-1} q_{\pi, T-1}(s, A_{T-1})\mid S_{T-1} =s \right).
\end{align}
The definition of $\mu^*_{T-1}$ in~\eqref{def mu star} and Lemma~\ref{lem:math-optimal} ensure that $\mu^*_{T-1}$ is an optimal solution to
\begin{align}
    \min_{\mu_{T-1} \in \Lambda^*_{T-1}}  \quad \V\left(\pdisg\left(\tau^{\mu_{T-1}}_{T-1}\right) \mid S_{T-1} = s\right).
\end{align}
Now, suppose for some $t \in [T-2]$ and every successor state $s'$, $\mu^*_{t+1:T-1}$ is an optimal solution to
\begin{align}
    \min_{\mu_{t+1} \in \Lambda^*_{t+1}, \dots, \mu_{T-1} \in \Lambda^*_{T-1}}  \quad \V\left(\pdisg\left(\tau^{\mu_{t+1:T-1}}_{t+1:T-1}\right)\mid S_{t+1}=s'\right).
\end{align}
To complete induction, we proceed to proving that $\mu^*_{t:T-1}$ is an optimal solution to
\begin{align}
    \label{eq induction opt problm}
    \min_{\mu_t \in \Lambda^*_t, \dots, \mu_{T-1} \in \Lambda^*_{T-1}}  \quad \V\left(\pdisg\left(\tau^{\mu_{t:T-1}}_{t:T-1}\right) \mid S_t = s\right).
\end{align}
In the rest of this proof, we omit the domain $\Lambda^*_t, \dots, \Lambda^*_{T-1}$ for simplifying notations. For any $\mu_{t:T-1}$, we have
\begin{align}
    &\V\left(\pdisg(\tau^{\mu_{t:T-1}}_{t:T-1})\mid S_{t}\right) \\
    =& \E_{A_{t}}\left[\rho_{t}^2 \left(\E_{S_{t+1}}\left[\V\left(\pdisg(\tau^{\mu_{{t+1}:T-1}}_{{t+1}:T-1}) \mid S_{t+1}\right) \mid S_{t}, A_{t}\right] + \omega_{t}(S_{t}, A_{t}) + q_{\pi, t}^2(S_{t}, A_{t}) \right) \mid S_{t}\right] \\
    & -  v_{\pi, t}^2(S_{t})  \explain{By Lemma~\ref{lem:recursive-var}}  \\
    \stackrel{(a)}{\geq} & \E_{A_{t}}\left[\rho_{t}^2 \left(\E_{S_{t+1}}\left[\min_{\mu'_{t+1:T-1}}\V\left(\pdisg(\tau^{\mu'_{{t+1}:T-1}}_{{t+1}:T-1}) \mid S_{t+1}\right) \mid S_{t}, A_{t}\right] + \omega_{t}(S_{t}, A_{t}) + q_{\pi, t}^2(S_{t}, A_{t}) \right) \mid S_{t}\right] \\
    & -  v_{\pi, t}^2(S_{t})  \explain{Monotonicity in the continuation variance}  \\
    =& \E_{A_{t}}\left[\rho_{t}^2 \left(\E_{S_{t+1}}\left[\V\left(\pdisg(\tau^{\mu^*_{{t+1}:T-1}}_{{t+1}:T-1}) \mid S_{t+1}\right) \mid S_{t}, A_{t}\right] + \omega_{t}(S_{t}, A_{t}) + q_{\pi, t}^2(S_{t}, A_{t}) \right) \mid S_{t}\right] \\
    & -  v_{\pi, t}^2(S_{t}) \explain{Inductive hypothesis} \\
    =& \E_{A_{t}}\left[\rho_{t}^2 u_{\pi,t}(S_t,A_t)\mid S_{t}\right]  -  v_{\pi, t}^2(S_{t}) \explain{By \eqref{def u}} \\
    =& \V_{A_{t}}\left(\rho_{t} \sqrt{u_{\pi,t}(S_t,A_t)}\mid S_{t}\right) + \E_{A_t}\left[\rho_t \sqrt{u_{\pi, t}(S_t, A_t)} \mid S_t\right]^2 -  v_{\pi, t}^2(S_{t}) \explain{Definition of variance} \\
    =& \V_{A_t}\left(\rho_{t} \sqrt{u_{\pi,t}(S_t,A_t)}\mid S_{t}\right) + \E_{A_t\sim\pi_t(\cdot \mid S_t)}\left[\sqrt{u_{\pi, t}(S_t, A_t)} \mid S_t\right]^2 -  v_{\pi, t}^2(S_{t}) \explain{Lemma~\ref{lem stats unbiasedness} and $\mu_t \in \Lambda^*_t$} \\
    \stackrel{(b)}{\geq}& \E_{A_t\sim\pi_t(\cdot \mid S_t)}\left[\sqrt{u_{\pi, t}(S_t, A_t)} \mid S_t\right]^2 -  v_{\pi, t}^2(S_{t}) \explain{Non-negativity of variance}.
\end{align}
According to the inductive hypothesis, the equality in $(a)$ can be achieved when $\mu_{t+1:T-1} = \mu^*_{t+1:T-1}$. By the construction of $\mu^*_t$ in~\eqref{def mu star} and Remark~\ref{lem:math-variance-0}, the equality in $(b)$ can be achieved when $\mu_t = \mu^*_t$. This suggests that $\mu^*_{t:T-1}$ achieves the lower bound and is thus an optimal solution to~\eqref{eq induction opt problm}, which completes the induction and thus completes the proof.
\end{proof}

Theorem~\ref{lem:rl-optimal} establishes the optimality of the closed-form behavior policy $\mu^*$. Next, we quantify the variance reduction in
reinforcement learning with this $\mu^*$. Through a recursive analysis, we show that the variance reduction compounds over each step, bringing substantial efficiency improvement. Specifically, we provide a theoretical comparison of our method with the on-policy Monte Carlo estimator, which uses the target policy $\pi$ to interact with the environments.
\begin{restatable}[Optimal Variance Reduction]{theorem}
{revarsmallerstrongeroptimal}
\label{lem:var_smaller_stronger_optimal}
$\forall t, s$,
\begin{align}
&\V\left(\pdisg(\tau^{\pi_{t:T-1}}_{t:T-1})\mid S_t = s\right) - \V\left(\pdisg(\tau^{\mu^*_{t:T-1}}_{t:T-1})\mid S_t = s\right) \\
=&\textstyle \underbrace{\V_{A_t \sim \pi_t}\qty(\sqrt{u_t(S_t,A_t)} \mid S_t=s) }_{\hypertarget{4.1}{(4.1)}}  + \underbrace{\delta^{\pi, \mu^*}_t(s) }_{\hypertarget{4.2}{(4.2)}},
\end{align}
where $\delta^{\pi, \mu^*}_t(s)\doteq0$ for $t=T-1$ and $\forall t\in [T-2]$,
\begin{align}
 \delta^{\pi, \mu^*}_t(s)\doteq  \E_{A_t \sim \pi_t, S_{t+1}}\left[ \V\left(\pdisg(\tau^{\pi_{t+1:T-1}}_{t+1:T-1}) \mid S_{t+1}\right) - \V\left(\pdisg(\tau^{\mu^*_{t+1:T-1}}_{t+1:T-1}) \mid S_{t+1}\right)  \mid S_t = s\right] \geq 0.
\end{align}
\end{restatable}
Its proof is in Appendix \ref{append:var_smaller_stronger_optimal}.
In Theorem \ref{lem:var_smaller_stronger_optimal}, we show that the variance reduction of the optimal behavior policy $\mu^*$ comes from two sources. First, the variance term \hyperlink{4.1}{(4.1)} is eliminated by choosing an optimal behavior policy $\mu^*$. Second, the variance reduction in the future step \hyperlink{4.2}{(4.2)}  is compounded into the current step.

Unfortunately, approximating $\mu^*_t$ requires knowledge of $u_{\pi, t}$ \eqref{def u} that contains the transition function $p$.
Estimating these transitions is difficult with substantial stochasticity and function approximation (cf. model-based RL \citep{sutton1990integrated,sutton2012dyna,deisenroth2011pilco,chua2018deep}).
We therefore seek an alternative policy $\hat{\mu}$ that can be easily implemented without direct knowledge of the transition function $p$ (cf. model-free RL \citep{sutton1988learning,watkins1989learning}).

\subsection{One-Step Optimal Behavior Policy}
In this section, we propose a behavior policy $\hat{\mu}$ that is easy to approximate.
\emph{We achieve this by aiming at one-step optimality instead of global optimality.}
We try to find the best $\mu_t$
assuming in the future we follow $\pi_{t+1}, \dots, \pi_{T-1}$, instead of $\mu^*_{t+1}, \dots, \mu^*_{T-1}$.
Formally, we aim to find a policy $\mu$ which solves the optimization problem
\begin{align}\label{eq: one step optimization}
\min_{\mu}& \quad \V\left(\pdisg(\tau^\qty{\mu_t,\pi_{t+1},\dots, \pi_{T-1}}_{t:T-1})\mid S_t = s\right) \\
\text{s.t.}& \quad \mu \in \hat{\Lambda}
\end{align}
for all $t$ and $s$. Similarly, we require the policy space $\hat{\Lambda}$ to satisfy the broadness  \eqref{eq: Lambda general} and the unbiasedness \eqref{eq: rl unbiasedness} requirements.
To this end, we define
\begin{align}\label{eq def q hat last step}
\hat q_{\pi, t}(s, a) \doteq q_{\pi, t}^2(s, a)
\end{align}
for $t = T-1$ and otherwise
\begin{align}\label{eq def q hat}
&\hat q_{\pi, t}(s, a) \doteq q_{\pi, t}^2(s, a) + \omega_{\pi,t}(s, a) +\sum_{s'} p(s'|s, a)\V\left(\pdisg(\tau^{\pi_{t+1:T-1}}_{t+1:T-1}) \mid S_{t+1} = s'\right).
\end{align}
$\hat q_{\pi, t}(s, a)$ is similar to $u_{\pi, t}(s, a)$ by sharing the same first two terms. They only differ by the third term, the expected variance of the future trajectory. This is because when finding the global optimal behavior policy, we follow an optimal policy $\mu^*$ in the future, giving us $\textstyle{\sum_{s'} p(s'|s, a)\V\left(\pdisg(\tau^{\mu^*_{t+1:T-1}}_{t+1:T-1}) \mid S_{t+1} = s'\right)}$. While in this section, we aim to solve the one-step optimal optimization problem by assuming future variance is given by the target policy $\pi$, introducing $\textstyle{\sum_{s'} p(s'|s, a)\V\left(\pdisg(\tau^{\pi_{t+1:T-1}}_{t+1:T-1}) \mid S_{t+1} = s'\right)}$ as the third term. Now, we define the policy space $\hat{\Lambda}$ as
\begin{align}\label{def: Lambda hat}
\hat{\Lambda} \doteq& \qty{\mu \mid \forall t, s, a, \mu_t(a|s) = 0 \implies \pi_t(a|s)\hat{q}_{\pi, t}(s, a)  = 0}.
\end{align}
This space $\hat{\Lambda} $ is smaller than the optimal search space $\Lambda^*$, yielding a class of efficient behavior policies that are easier to learn. We also show that this policy space $\hat{\Lambda}$ is larger than the conventional policy space $\Lambda^-$
 with the coverage assumption
that previous works \citep{precup:2000:eto:645529.658134,maei2011gradient,sutton2016emphatic,sutton2018reinforcement,zhang2022thesis} consider.
\begin{lemma}[Broadness]\label{lemma: Lambda broadness hat q}
$\Lambda^- \subseteq \hat{\Lambda}.$
\end{lemma}
Its proof is in Appendix \ref{appendix: Lambda broadness hat q}. Broadness of $\hat{\Lambda}$ ensures the optimization space is large enough to find a good behavior policy. At the same time, behavior policies in this policy space give an unbiased estimation by the following lemma.
\begin{restatable}[Unbiasedness]{lemma}{reOOLambdaOOunbiasednessOOhat}
\label{lemma: Lambda unbiasedness hat}
$\forall \mu \in \hat{\Lambda}$, $\forall  t, \forall  s$,
 $\E\left[\pdisg(\tau^{\mu_{t:T-1}}_{t:T-1}) \mid S_t = s\right] = v_{\pi, t}(s).$
\end{restatable}
Its proof is in Appendix \ref{appendix: Lambda hat unbiasedness}.
Define $\forall t, s, a,$
\begin{align}
\label{def hat mu}
\hat \mu_t(a|s) \propto \textstyle \pi_t(a|s)\sqrt{\hat q_{\pi, t}(s, a)}.
\end{align}
The following theorem shows $\hat \mu_t(a|s)$ is an optimal solution to  \eqref{eq: one step optimization}.
\begin{theorem}[One-step Optimal Policy]
\label{thm: one step optimal}
For any $t$ and $s$,
the behavior policy $\hat \mu$ defined in \eqref{def hat mu} is an optimal solution to the following problem
\begin{align}
\min_{\mu}& \quad \V\left(\pdisg(\tau^\qty{\mu_t,\pi_{t+1},\dots, \pi_{T-1}}_{t:T-1})\mid S_t = s\right) \\
\text{s.t.}& \quad \mu \in \hat{\Lambda}.
\end{align}
\end{theorem}

\begin{proof}
According to the recursive expression of the variance in Lemma~\ref{lem:recursive-var} in
Appendix~\ref{append:rl-optima}, we rewrite the objective of \eqref{eq: one step optimization} as
\begin{align}
&\V\left(\pdisg(\tau^\qty{\mu_t,\pi_{t+1},\dots, \pi_{T-1}}_{t:T-1})\mid S_t = s\right)  \\
=& \E_{A_t\sim \mu_t}\left[\rho_t^2 \left(\E_{S_{t+1}}\left[ \V\left(\pdisg(\tau^{{\pi}_{t+1:T-1}}_{t+1:T-1}) \mid S_{t+1}\right) \mid S_t, A_t\right] + \omega_{\pi,t}(S_t, A_t) + q_{\pi, t}^2(S_t, A_t) \right) \mid S_t\right] \\
&- v_{\pi, t}^2(S_t) \explain{By Lemma \ref{lem:recursive-var}} \\
=& \underbrace{\E_{A_t\sim\mu_t}\left[\rho_t^2 \hat q_{\pi, t}(S_t, A_t) \mid S_t\right] }_{\hypertarget{1.1}{(1.1)}}  - v_{\pi, t}^2(S_t) \explain{By Definition of $\hat{q}$ \eqref{eq def q hat}}\\
\end{align}
We can then use the definition of variance and the unbiasedness to rewrite the first term \hyperlink{1.1}{(1.1)}.
\begin{align}
&\E_{A_t\sim\mu_t}\left[\rho_t^2 \hat q_{\pi, t}(S_t, A_t) \mid S_t\right]  \\
=&   \V_{A_t\sim\mu_t}\left(\rho_t \sqrt{\hat q_{\pi, t}(S_t, A_t)} \mid S_t\right) + \E_{A_t\sim\mu_t}^2\left[ \rho_t \sqrt{\hat q_{\pi, t}(S_t, A_t)} \mid S_t\right] \explain{Definition of Variance}  \\
=&   \V_{A_t\sim\mu_t}\left(\rho_t \sqrt{\hat q_{\pi, t}(S_t, A_t)} \mid S_t\right) + \E_{A_t\sim\pi_t}^2\left[\sqrt{\hat q_{\pi, t}(S_t, A_t)} \mid S_t\right] \explain{Lemma~\ref{lem stats unbiasedness} and $\mu \in \hat \Lambda$}.
\end{align}
Therefore, we have shown that the objective function \eqref{eq: one step optimization} can be simplified by
\begin{align}
&\V\left(\pdisg(\tau^\qty{\mu_t,\pi_{t+1},\dots, \pi_{T-1}}_{t:T-1})\mid S_t = s\right) \\
=&\V_{A_t\sim\mu_t}\left(\rho_t \sqrt{\hat q_{\pi, t}(S_t, A_t)} \mid S_t = s\right) + \E_{A_t\sim\pi_t}^2\left[\sqrt{\hat q_{\pi, t}(S_t, A_t)} \mid S_t\right] - v_{\pi, t}^2(S_t).
\end{align}
Since the second term and the third term don't depend on $\mu$, to solve \eqref{eq: one step optimization}, it is equivalent to solving the following optimization problem.
\begin{align}
\min_{\mu}& \quad \V_{A_t\sim\mu_t}\left(\rho_t \sqrt{\hat q_{\pi, t}(S_t, A_t)} \mid S_t = s\right) \\
\text{s.t.}& \quad \mu \in \hat{\Lambda}.
\end{align}
for any $s$ and $t$.
According to Lemma~\ref{lem:math-optimal},
\begin{align}
\hat \mu_t(a|s) \propto \textstyle \pi_t(a|s)\sqrt{\hat q_{\pi, t}(s, a)}.
\end{align}
is an optimal solution to~\eqref{eq: one step optimization}.

\end{proof}

We now present our main result that $\hat \mu$ \emph{provably} reduces variance and quantify the variance reduction with the following theorem.
\begin{restatable}[One-step Optimal Variance Reduction]{theorem}
{reonestepvarsmallerstrongeroptimal}
\label{lem:one_step_var_smaller_stronger_optimal}
$\forall t, s$,
\begin{align}
&\V\left(\pdisg(\tau^{\pi_{t:T-1}}_{t:T-1})\mid S_t = s\right) - \V\left(\pdisg(\tau^{\hat{\mu}_{t:T-1}}_{t:T-1})\mid S_t = s\right) \\
= &\textstyle \underbrace{\V_{A_t \sim \pi_t}\qty(\sqrt{\hat{q}_t(S_t,A_t)} \mid S_t=s) }_{\hypertarget{5.1}{(5.1)}}  + \underbrace{\delta^{\pi, \hat{\mu}}_t(s) }_{\hypertarget{5.2}{(5.2)}},
\end{align}
where $\delta^{\pi, \hat{\mu}}_t(s)\doteq0$ for $t=T-1$ and $\forall t\in [T-2]$, $\delta^{\pi, \hat{\mu}}_t(s)\doteq $
\begin{align}
\! \! \! \! \! \! \!\textstyle  \E_{A_t \sim \hat{\mu}_t, S_{t+1}}\left[ \rho_t^2 \qty(\V\left(\pdisg(\tau^{\pi_{t+1:T-1}}_{t+1:T-1}) \mid S_{t+1}\right) - \V\left(\pdisg(\tau^{\hat{\mu}_{t+1:T-1}}_{t+1:T-1}) \mid S_{t+1}\right))  \mid S_t = s\right].
\end{align}
Moreover, we prove $\forall t, s, \delta^{\pi, \hat{\mu}}_t(s) \geq 0$, meaning the variance reduction in future steps is compounded into the current step.
\end{restatable}
Its proof is in Appendix \ref{append:one_step_var_smaller_stronger_optimal}.
In Theorem \ref{lem:one_step_var_smaller_stronger_optimal}, we show that the variance reduction of our method includes two sources.
First, the variance term \hyperlink{5.1}{(5.1)} is eliminated by choosing a behavior policy $\hat{\mu}$.
$\V_{A_t \sim \pi_t}\qty(\sqrt{\hat{q}_t(S_t,A_t)} \mid S_t=s)  = 0$ occurs only when all actions with $\pi_t(a|s) > 0$ have the same $\hat q_{\pi,t}$ on the state $s$.
It is reasonable to conjecture that this is rare in practice.
The magnitude of the eliminated variance term depends on a specific target policy and the environment. We empirically show that the variance reduction is significant in the numerical experimentation Section~\ref{sec:experiment}.
Second, \hyperlink{5.2}{(5.2)} shows that the variance reduction in the future step is compounded into the current step.

Directly approximating $\hat{q}$ from \eqref{eq def q hat} is challenging, because it requires approximating the transition probability $p$ and the variance of future trajectories. In the following section, we transform $\hat{q}$ into a standard Bellman equation. Therefore, we can leverage off-the-shelf model-free reinforcement learning methods to approximate it.

\section{Learning Closed-Form Behavior Policies}
\label{sec: learn behavior policies}

Despite that $\hat q_{\pi, t}$ in \eqref{eq def q hat} has a complicated definition,
we reduce it to a concise Bellman equation.
We show that it is the state action value function of the policy $\pi$ with the same transition function $p$ but a different reward function $\hat r$.
\begin{theorem}[Bellman Equation for Variance Reduction]\label{lem:hat-q-recursive}
Define
\begin{align}\label{def: hat r}
\hat{r}_{\pi,t}(s,a) \doteq 2r(s, a) q_{\pi, t}(s, a)- r^2(s, a).
\end{align}
Then $\hat{q}_{\pi, t}(s, a) = \hat r_{\pi, t}(s, a)$ for $t=T-1$
and otherwise
\begin{align}\label{eq: hat q bellman}
\textstyle &\hat{q}_{\pi, t}(s, a)
=  \hat{r}_{\pi,t}(s,a) +  \sum_{s', a'} p(s'|s, a) \pi_{t+1}(a'|s') \hat{q}_{\pi, t+1}(s', a').
\end{align}
\end{theorem}

The proof uses the following auxiliary reward and value functions, which express the on-policy variance in the form of a Bellman equation \citep{variance2016Tamar,o2017uncertainty,sherstan2018directly}. Define
\begin{align}
 &\tilde r_{\pi,t}(s, a) \doteq
\omega_{\pi,t}(s, a) + q^2_{\pi, t}(s, a) - v^2_{\pi, t}(s) \quad \forall t \in [T-1], \label{def:tilde-r} \\
 &\tilde q_{\pi, t}(s, a)  \doteq
\begin{cases}
\tilde r_{\pi,t}(s, a) + \sum_{s', a'} p(s'|s, a) \pi_{t+1}(a'|s') \tilde{q}_{\pi, t+1}(s', a')  & \text{ if }  t  \in [T-2] \\
\tilde r_{\pi,t}(s, a)   & \text{ if } t = T-1
\end{cases}.\label{def:tilde-q}
\end{align}
That is, $\tilde q_{\pi, t}$ is the state-action value function of the target policy $\pi$ under the reward $\tilde r_{\pi, t}$. Appendix~\ref{append: variance equality} proves that
\begin{align}\label{eq: variance equality main}
\V\left(\pdisg(\tau^{\pi_{t:T-1}}_{t:T-1}) \mid S_t = s\right) = \sum_a \pi_t(a|s)\tilde q_{\pi, t}(s, a) \quad \forall t, s,
\end{align}
and consequently that
\begin{align}\label{eq hat and tilde q}
\hat q_{\pi, t}(s, a) = \tilde q_{\pi, t}(s, a) + v_{\pi, t}^2(s).
\end{align}

\begin{proof}
For $t = T-1$, we have
\begin{align}
 \hat{q}_{\pi, t}(s, a)  &=  q^2_{ \pi, t}( s, a) \explain{Definition of $\hat q_{\pi, t}$~\eqref{eq def q hat last step}} \\
&= \hat{r}_{\pi,t}(s,a). \explain{By $q_{\pi, T-1}(s, a) = r(s, a)$ and the definition of $\hat r_{\pi, t}$ \eqref{def: hat r}}
\end{align}
For $t\in [T-2]$, we have
\begin{align}
&\hat{q}_{\pi,t}(s,a) \\
=& \tilde{q}_{ \pi, t}( s, a) + v^2_{\pi, t}(s) \explain{By~\eqref{eq hat and tilde q}} \\
=&  \tilde r_{\pi,t}(s, a) + v^2_{\pi, t}(s)  + \sum_{s', a'} p(s'|s, a) \pi_{t+1}(a'|s') \tilde{q}_{\pi, t+1}(s', a')  \explain{Definition of $\tilde{q}$ \eqref{def:tilde-q}} \\
=& \tilde r_{\pi,t}(s, a) + v^2_{\pi, t}(s)  + \sum_{s', a'} p(s'|s, a) \pi_{t+1}(a'|s') (\tilde{q}_{\pi, t+1}(s', a') + v^2_{\pi, t+1}(s') -  v^2_{\pi, t+1}(s')) \\
=& \tilde r_{\pi,t}(s, a) + v^2_{\pi, t}(s)  + \sum_{s', a'} p(s'|s, a) \pi_{t+1}(a'|s') (\hat{q}_{\pi, t+1}(s', a') -  v^2_{\pi, t+1}(s')) \explain{By~\eqref{eq hat and tilde q}}\\
=& \omega_{\pi,t}(s, a) + q^2_{\pi, t}(s, a) - \sum_{s'} p(s'|s, a)   v^2_{\pi, t+1}(s') +  \sum_{s', a'} p(s'|s, a) \pi_{t+1}(a'|s') \hat{q}_{\pi, t+1}(s', a') \explain{Definition of $\tilde{r}$ \eqref{def:tilde-r}}\\
=& -(\E [ v_{\pi, t+1}(S_{t+1})\mid S_t=s, A_t=a ])^2 + q^2_{\pi, t}(s, a) + \sum_{s', a'} p(s'|s, a) \pi_{t+1}(a'|s') \hat{q}_{\pi, t+1}(s', a')  \explain{Definition of $\omega$ \eqref{def:nu}} \\
=& -(q_{\pi, t}(s, a) - r(s, a) )^2 + q^2_{\pi, t}(s, a) + \sum_{s', a'} p(s'|s, a) \pi_{t+1}(a'|s') \hat{q}_{\pi, t+1}(s', a') \\
=&  2r(s, a) q_{\pi, t}(s, a)- r^2(s, a)  + \sum_{s', a'} p(s'|s, a) \pi_{t+1}(a'|s') \hat{q}_{\pi, t+1}(s', a') \\
=&  \hat r_{\pi, t}(s, a)  + \sum_{s', a'} p(s'|s, a) \pi_{t+1}(a'|s') \hat{q}_{\pi, t+1}(s', a'), \explain{By the definition of $\hat r_{\pi, t}$ \eqref{def: hat r}}
\end{align}
which completes the proof.
\end{proof}

\emph{Theorem~\ref{lem:hat-q-recursive} allows any off-the-shelf offline policy evaluation method to estimate $\hat q$,}
after which the behavior policy $\hat \mu$ can be computed easily with~\eqref{def hat mu}.
For generality, we adopt the behavior policy agnostic offline learning setting \citep{nachum2019dualdice},
where the offline data is in the form of
\begin{align}
  \qty{(t_i,s_i,a_i,r_i,s_i')}_{i=1}^m
\end{align}
with $m$ previously logged offline data tuples.
In the $i$-th data tuple,
$t_i$ is the time step,
$s_i$ is the state at time step $t_i$,
$a_i$ is the action executed on state $s_i$,
$r_i$ is the sampled reward,
and $s_i'$ is the successor state.
Those tuples can be generated by one or more, known or unknown behavior policies, and they do not need to come from a complete trajectory.

\begin{algorithm}
\caption{Offline Data Informed (ODI) algorithm}
\label{alg: ODI algorithm}
\begin{algorithmic}[1]
\State {\bfseries Input:}
a target policy $\pi$, \\
an offline dataset $\mathcal{D} = \qty{(t_i,s_i,a_i,r_i,s'_i)}_{i=1}^m$
\State {\bfseries Output:} a behavior policy $\hat{\mu}$
\State Approximate $q_{\pi,t}$ from $\mathcal{D}$ using  offline RL methods  (e.g. Fitted Q-Evaluation)
\State Construct $\mathcal{D}_{\hat{r}} \doteq \qty{(t_i,s_i,a_i,\hat{r}_i,s'_i)}_{i=1}^m$ by computing $\hat{r}$ \eqref{def: hat r}
\State Approximate $\hat{q}_{\pi,t}$ from $\mathcal{D}_{\hat{r}}$ using  offline RL methods  (e.g. Fitted Q-Evaluation)
\State \textbf{Return:} $\hat \mu_t(a|s) \propto \pi_t(a|s)\sqrt{\hat{q}_{\pi, t}(s, a)}$
\end{algorithmic}
\end{algorithm}

In this paper, we adopt Fitted Q-Evaluation (FQE; \citet{le2019batch}) as a concrete instantiation of our framework. In fact, the proposed framework is general and can readily incorporate any state-of-the-art offline policy evaluation method to approximate $\hat q$.
The framework does not require a tabular state representation;
$q_{\pi,t}$ and $\hat q_{\pi,t}$ may be estimated with function
approximation on continuous or otherwise large state spaces.

The approximation of $\hat q$ proceeds in two stages. First, we approximate the action-value function $q$ using FQE under the original reward. Second, we apply FQE again, this time using the modified reward defined in \eqref{def: hat r}. A detailed description of the resulting procedure is provided in Algorithm~\ref{alg: ODI algorithm}.
Since $\hat q_{\pi, t}(s, a)$ is proved to be always non-negative (cf. \eqref{eq def q hat}),
we use positive function classes for FQE in approximating $\hat q$,
e.g.,
a neural network with softplus as the last activation function.

Next, we theoretically analyze how the error in approximating $\hat q$ affects the amount of reduced variance in Theorem~\ref{lem:one_step_var_smaller_stronger_optimal}.
Considering \eqref{eq def q hat} is a summation of non-negative terms, including the future trajectory variance, we assume $\hat q_{\pi, t}(s, a)$ to be positive.
We use $\hat q^+_{\pi, t}(s, a) > 0$ to denote our approximation to $\hat q_{\pi, t}(s, a)$.
We express the approximation error through the ratio
\begin{align}\label{def: eta}
\eta_{\pi, t}(s, a) \doteq \hat q_{\pi, t}^+(s, a) / \hat q_{\pi, t}(s, a) > 0.
\end{align}

If $\eta_{\pi, t}(s, a)$ is $1$, there is no approximation error for $(s, a, t)$.
The actual learned behavior policy is then denoted by
\begin{align}\label{def: mu +}
\hat \mu_t^+(a|s) \propto \pi_t(a|s) \sqrt{\hat q_{\pi, t}^+(s, a)}.
\end{align}
Then, we generalize Theorem~\ref{lem:one_step_var_smaller_stronger_optimal} to the following theorem.
\begin{theorem}[Error Analysis]
\label{lem:error analysis}
For any $t$ and $s$,
\begin{align}
&\textstyle \mathbb{V}(\pdisg(\tau_{t:T-1}^{\hat \mu_{t:T-1}^+})\mid S_t = s) \\
\leq& \mathbb{V}(\pdisg(\tau_{t:T-1}^{\pi_{t:T-1}})\mid S_t = s) - \epsilon^+_t(s).
\end{align}
To define $\epsilon^+_t(s)$, first define
\begin{align}
c^+_t(s) \doteq&  \textstyle  \sum_a \pi_t(a|s) \hat q_{\pi, t}(s, a) - \\
&\textstyle \left(\sum_a \pi_t(a|s)\sqrt{\eta_{\pi, t}(s, a)} \sqrt{\hat q_{\pi, t}(s, a)}\right)  \textstyle \left(\sum_a \pi_t(a|s) \frac{1}{\sqrt{\eta_{\pi, t}(s, a)}} \sqrt{\hat q_{\pi, t}(s, a)}\right).
\end{align}
Then we define $\epsilon^+_t(s) \doteq c^+_t(s)$ for $t = T-1$ and otherwise
\begin{align}\label{def: epsilon +}
&\epsilon^+_t(s)\\
\doteq& c^+_t(s) + \mathbb E_{A_t \sim \hat \mu_t^+}[\rho_t^2 \mathbb E_{S_{t+1}}[\epsilon^+_{t+1}(S_{t+1})| s, A_t]].
\end{align}
\end{theorem}
Its proof is in Appendix~\ref{append:error analysis}.
When there is no estimation error,
i.e., $\eta_{\pi, t}(s, a) = 1$,
$c_t^+(s)$ reduces to the variance term \hyperlink{5.1}{(5.1)} of Theorem~\ref{lem:one_step_var_smaller_stronger_optimal}, which is non-negative by Jensen's inequality.
If the ideal local gain in \hyperlink{5.1}{(5.1)} is positive, then $c^+_t(s)$ remains positive for sufficiently small relative estimation errors.
Because the sign of $c^+_t$ only depends on the current $\eta_{\pi,t}$, the estimation error in the future step does not affect the current $c^+_t$.
Notably, even if some $\epsilon^+_{t+1}(S_{t+1}) < 0$, $\epsilon^+_t(S_t)$ can still be positive.
This is because $\epsilon^+_t(s)$ depends on the expectation of the $\epsilon^+_{t+1}(S_{t+1})$, rather than a single realization,
and $c^+_t$ can still be positive.
This provides robustness to approximation error.
\emph{The PDIS estimator using $\hat \mu^+_t(a|s)$ is always unbiased,
regardless of the approximation error $\eta$.}

Theorem~\ref{lem:error analysis} links offline data to the resulting variance reduction.
For FQE, for example, the analyses of \citet{munos2003error,antos2008learning,munos2008finite,chen2019information} connect offline data to the approximation error $\eta$.
Theorem~\ref{lem:error analysis} then directly relays the approximation error to the amount of reduced variance. We refer interested readers to these works for detailed analyses of approximation errors, as a comprehensive treatment of value function approximation is beyond the scope of this paper. Section~\ref{sec: end to end} further relates approximation error
and retained variance reduction to online sample complexity.

\section{Variance Reduction and Online Sample Complexity under Approximation Error}
\label{sec: end to end}

In this section, we quantify the effect of relative estimation error on the local
and trajectory-level variance reduction and derive an online
sample requirement for a prescribed evaluation accuracy and
confidence level.

We write $\mathcal{D}_m$ for the offline dataset of $m$ logged tuples, $\varepsilon > 0$ for the target accuracy of the estimate of $J(\pi)$, $\delta_{\text{off}}$ for the failure probability of the offline learning stage, and $\delta_{\text{on}}$ for the failure probability of the online Monte Carlo stage. We assume rewards are bounded, $R_{\max} \doteq \sup_{s, a} \abs{r(s, a)} < \infty$, and abbreviate the remaining horizon by $L_t \doteq T - t$. For a behavior policy $\mu$, we write
\begin{align}\label{def: Y sigma}
Y^\mu \doteq \pdisg(\tau^{\mu_{0:T-1}}_{0:T-1}), \qquad \sigma_\mu^2 \doteq \V(Y^\mu),
\end{align}
for the estimator built from a single fresh online trajectory and its variance, and
\begin{align}\label{def: J hat m N}
\hat J_{m, N} \doteq \frac{1}{N} \sum_{i=1}^N Y_i^{\hat \mu^+}
\end{align}
for the average of $N$ fresh trajectories collected with the learned behavior policy $\hat \mu^+$ of \eqref{def: mu +}.

\subsection{The Auxiliary Value Function Is a Conditional Second Moment.}
The auxiliary value function $\hat q_{\pi,t}$ admits the
following conditional second-moment representation.
\begin{lemma}[Second-Moment Interpretation]\label{lem: second moment}
For all $t$, $s$, and $a$,
\begin{align}\label{eq: second moment}
\hat q_{\pi, t}(s, a) = \E_\pi\left[G_t^2 \mid S_t = s, A_t = a\right].
\end{align}
Consequently $0 \leq \hat q_{\pi, t}(s, a) \leq L_t^2 R_{\max}^2$.
\end{lemma}
Recall that $L_t=T-t$ is the remaining horizon at time $t$.
The proof of Lemma~\ref{lem: second moment} is in
Appendix~\ref{append: second moment}.
The one-step optimal behavior policy in \eqref{def hat mu}
therefore reweights the target policy by the square root of
the conditional second moment of its return, assigning relatively more probability to actions whose future returns tend to have larger magnitude.

\subsection{Mean-Squared Error of a Behavior Policy.} By Lemma~\ref{lemma: Lambda unbiasedness hat}, the PDIS estimator
is unbiased for every $\mu\in\hat\Lambda$. We use this property
to relate its single-trajectory variance to the mean-squared error
of the online average. Recall from \eqref{def: Y sigma} that
$Y^\mu$ is the estimator built from a single fresh online trajectory
collected under $\mu$, $\sigma_\mu^2=\V(Y^\mu)$ is its variance,
and $N$ is the number of such trajectories.
\begin{corollary}[Mean-Squared Error of a Behavior Policy]\label{cor: online complexity}
Let $\mu \in \hat \Lambda$ satisfy $\E[(Y^\mu)^2] < \infty$, and let $Y_1^\mu, \dots, Y_N^\mu$ be i.i.d.\ copies of $Y^\mu$. Then $\E[\frac{1}{N}\sum_i Y_i^\mu] = J(\pi)$ and
\begin{align}\label{eq: mse identity}
\E\left[\left(\frac{1}{N}\sum_{i=1}^N Y_i^\mu - J(\pi)\right)^2\right] = \frac{\sigma_\mu^2}{N}.
\end{align}
Hence, for $\sigma_\mu^2 > 0$, the smallest $N$ achieving mean squared error at most $\varepsilon^2$ is exactly $\lceil \sigma_\mu^2 / \varepsilon^2 \rceil$. In particular, the optimal behavior policy $\mu^*$ of Theorem~\ref{lem:rl-optimal} minimizes the mean squared error over $\Lambda^*$.
\end{corollary}
Its proof is in Appendix~\ref{append: online complexity}. This identity links the trajectory-level variance to the mean-squared error of the online average.

\subsection{Variance Reduction under Relative Approximation Error.}
For a state $s$ at time $t$, define
\begin{align}\label{def: Z Delta}
Z_{\pi, t}(s) \doteq \E_{A \sim \pi_t(\cdot|s)}\left[\sqrt{\hat q_{\pi, t}(s, A)}\right], \qquad
\Delta_{\pi, t}(s) \doteq \V_{A \sim \pi_t(\cdot|s)}\left(\sqrt{\hat q_{\pi, t}(s, A)}\right).
\end{align}
The quantity $\Delta_{\pi,t}(s)$ is the local variance-reduction
term of Theorem~\ref{lem:one_step_var_smaller_stronger_optimal}
when $\hat q_{\pi,t}$ is known exactly. We next bound the local
variance reduction retained under estimation error. Let
\begin{align}\label{def: active set}
\mathcal{K} \doteq \qty{(t, s, a) \mid \pi_t(a|s) > 0}
\end{align}
denote the set of \emph{target-supported} triples; these are the only triples that enter $\hat \mu$, $\hat \mu^+$, or any of the variance expressions above. Recall that $\mathcal{D}_m$ is the offline dataset $\qty{(t_i, s_i, a_i, r_i, s_i')}_{i=1}^m$ of Section~\ref{sec: learn behavior policies}, and that $\hat q^+_{\pi, t}$ is defined in \eqref{def: eta} as the approximation to $\hat q_{\pi, t}$ that Algorithm~\ref{alg: ODI algorithm} learns from $\mathcal{D}_m$, from which the behavior policy $\hat \mu^+$ is built by \eqref{def: mu +}.

\begin{assumption}[Relative Accuracy of the Learned Auxiliary Value Function]\label{assum: relative error}
Suppose $\hat q_{\pi,t}(s,a)>0$ for every $(t,s,a)\in\mathcal K$. There exist $\alpha_m \in [0, 1)$ and an event $\mathcal{E}_m$, determined by $\mathcal{D}_m$ alone, with $\Pr(\mathcal{E}_m) \geq 1 - \delta_{\text{off}}$ and such that on $\mathcal{E}_m$,
\begin{align}\label{eq: relative error event}
(1 - \alpha_m)\, \hat q_{\pi, t}(s, a) \; \leq \; \hat q^+_{\pi, t}(s, a) \; \leq \; (1 + \alpha_m)\, \hat q_{\pi, t}(s, a) \qquad \forall (t, s, a) \in \mathcal{K},
\end{align}
equivalently $\eta_{\pi, t}(s, a) \in [1 - \alpha_m, 1 + \alpha_m]$ for all $(t, s, a) \in \mathcal{K}$.
\end{assumption}
Assumption~\ref{assum: relative error} parameterizes the estimation
error on the target-supported set by the relative-error level
$\alpha_m$, without prescribing how $\hat q_\pi^+$ is obtained.
To quantify how this error affects the local variance reduction, define
\begin{align}\label{def: lambda alpha}
\lambda(\alpha)
\doteq
\frac{1}{2}\left(\frac{1}{\sqrt{1-\alpha^2}}-1\right),
\qquad \alpha\in[0,1).
\end{align}
The following result gives a lower bound on the retained local
variance reduction in terms of $\alpha_m$.

\begin{theorem}[Robust Local Variance Reduction]\label{thm: robust local}
On the event $\mathcal{E}_m$ of Assumption~\ref{assum: relative error}, for every $t$ and $s$,
\begin{align}\label{eq: robust local}
c^+_t(s) \; \geq \; \Delta_{\pi, t}(s) - \lambda(\alpha_m)\, Z_{\pi, t}(s)^2,
\end{align}
where $c^+_t(s)$ is the local term of Theorem~\ref{lem:error analysis}.
\end{theorem}
Its proof is in Appendix~\ref{append: robust local} and follows from
the Kantorovich inequality \citep{horn2012matrix} applied to the
distribution $w_a \propto \pi_t(a|s)\sqrt{\hat q_{\pi,t}(s,a)}$.
The key feature of \eqref{eq: robust local} is that the approximation
penalty has no first-order term. Since
$\lambda(\alpha_m)=\alpha_m^2/4+O(\alpha_m^4)$ as
$\alpha_m\to0$, the guaranteed loss in the ideal local variance
reduction is second order in the relative-error level.

The bound in Theorem~\ref{thm: robust local} is statewise.
Before propagating this bound through the horizon, we characterize the relative accuracy sufficient to retain a prescribed fraction of the ideal local variance reduction. Where $Z_{\pi,t}(s)>0$, define
the squared coefficient of variation of
$\sqrt{\hat q_{\pi,t}(s,\cdot)}$ under $\pi_t(\cdot|s)$,
\begin{align}\label{def: kappa}
\kappa^2_{\pi, t}(s) \doteq \frac{\Delta_{\pi, t}(s)}{Z_{\pi, t}(s)^2}.
\end{align}
The following corollary gives a sufficient condition on $\alpha_m$
for retaining any prescribed fraction of the ideal local gain.
\begin{corollary}[Retained Fraction of the Ideal Local Gain]\label{cor: retained fraction}
Fix $t$ and $s$ with $Z_{\pi, t}(s) > 0$, and let $\zeta \in (0, 1]$ denote the tolerated fraction of the ideal local gain that may be lost. On the event $\mathcal{E}_m$,
\begin{align}\label{eq: retained fraction}
\alpha_m^2 \; \leq \; 1 - \frac{1}{\left(1 + 2 \zeta \kappa^2_{\pi, t}(s)\right)^2} \qquad \Longrightarrow \qquad c^+_t(s) \; \geq \; (1 - \zeta)\, \Delta_{\pi, t}(s).
\end{align}
In particular, taking $\zeta = 1$ gives a sufficient condition for $c^+_t(s) \geq 0$.
\end{corollary}
Its proof is in Appendix~\ref{append: retained fraction}.

\subsection{From Local to Trajectory-Level Variance Reduction.}
Theorem~\ref{thm: robust local} gives a lower bound on the local variance-reduction term $c^+_t(s)$. Define this local lower bound as
\begin{align}\label{def: underline c}
\underline{c}_{m, t}(s) \doteq \Delta_{\pi, t}(s) - \lambda(\alpha_m) Z_{\pi, t}(s)^2,
\end{align}
the ideal local gain $\Delta_{\pi, t}(s)$ less the penalty $\lambda(\alpha_m) Z_{\pi, t}(s)^2$ charged for estimating $\hat q_{\pi, t}$, so that $c^+_t(s) \geq \underline{c}_{m, t}(s)$ on $\mathcal{E}_m$.

We next propagate this local bound through the horizon using the same recursion as $\epsilon^+_t$ in Theorem~\ref{lem:error analysis}. Set $\underline{\epsilon}_{m, T-1}(s) \doteq \underline{c}_{m, T-1}(s)$ and, for $t \in [T-2]$,
\begin{align}\label{def: underline epsilon}
\underline{\epsilon}_{m, t}(s) \doteq \underline{c}_{m, t}(s) + \E_{A_t \sim \hat \mu^+_t(\cdot|s)}\left[\rho_t^2 \, \E_{S_{t+1}}\left[\underline{\epsilon}_{m, t+1}(S_{t+1}) \mid s, A_t\right]\right],
\end{align}
so that $\underline{\epsilon}_{m, t}(s)$ lower-bounds the variance reduction accumulated from time $t$ onward. Finally, define the trajectory-level lower bound
\begin{align}\label{def: Gamma m}
\Gamma_m \doteq \E_{S_0 \sim p_0}\left[\underline{\epsilon}_{m, 0}(S_0)\right].
\end{align}
Because $\hat \mu^+$ is constructed from $\hat q^+_\pi$ learned using $\mathcal{D}_m$, the propagated quantities $\underline{\epsilon}_{m, t}$ and $\Gamma_m$ are also data-dependent.
\begin{theorem}[Trajectory-Level Variance Reduction Bound]\label{thm: certificate}
On the event $\mathcal{E}_m$ of Assumption~\ref{assum: relative error}, for every $t$ and $s$,
\begin{align}
\V\left(\pdisg(\tau^{\hat \mu^+_{t:T-1}}_{t:T-1}) \mid S_t = s\right) \; \leq \; \V\left(\pdisg(\tau^{\pi_{t:T-1}}_{t:T-1}) \mid S_t = s\right) - \underline{\epsilon}_{m, t}(s),
\end{align}
and consequently
\begin{align}\label{eq: sigma certificate}
\sigma_{\hat \mu^+}^2 \; \leq \; \sigma_\pi^2 - \Gamma_m.
\end{align}
In particular, if $\Gamma_m \geq 0$, the learned behavior policy has no greater variance than on-policy Monte Carlo; if $\Gamma_m > 0$, the variance is strictly lower.
\end{theorem}
Its proof is in Appendix~\ref{append: certificate}. A simple sufficient condition for $\Gamma_m \geq 0$ is that $\underline{c}_{m, t}(s) \geq 0$ at every reachable state and time.

\subsection{High-Probability Online Sample Complexity.}
We next translate the trajectory-level variance bound into a high-probability
online sample requirement.
Let $\hat q_{\min}\doteq
\inf_{(t,s,a)\in\mathcal K}\hat q_{\pi,t}(s,a)$.
When $\hat q_{\min}>0$, Lemma~\ref{lem: ratio bound} provides a bound
$b_m$ on the deviation of a single-trajectory estimator from $J(\pi)$,
with a dependence on $\alpha_m$ and the problem parameters. Combining this bound
with Theorem~\ref{thm: certificate} and Bernstein's inequality \citep{boucheron2013concentration}
yields the following result.

\begin{theorem}[Online Sample Complexity]\label{thm: end to end}
Let Assumption~\ref{assum: relative error} hold, suppose $\hat q_{\min} > 0$, and let $b_m$ be the range bound of Lemma~\ref{lem: ratio bound}. Let $\Gamma_m$ be as in \eqref{def: Gamma m}, let $N = N(\mathcal{D}_m)$, and let the $N$ fresh online trajectories be i.i.d.\ conditionally on $\mathcal{D}_m$. Then
\begin{align}\label{eq: conditional mse}
\E\left[\left(\hat J_{m, N} - J(\pi)\right)^2 \;\middle|\; \mathcal{D}_m\right] \; \leq \; \frac{\sigma_\pi^2 - \Gamma_m}{N} \qquad \text{on } \mathcal{E}_m,
\end{align}
and if
\begin{align}\label{eq: N requirement}
N \; \geq \; \left[\frac{2\left(\sigma_\pi^2 - \Gamma_m\right)}{\varepsilon^2} + \frac{2 b_m}{3 \varepsilon}\right] \log \frac{2}{\delta_{\text{on}}},
\end{align}
then $\Pr\left(\abs{\hat J_{m, N} - J(\pi)} \leq \varepsilon\right) \geq 1 - \delta_{\text{off}} - \delta_{\text{on}}$.
\end{theorem}

Its proof is in Appendix~\ref{append: end to end}.
The relative-error level $\alpha_m$ enters
\eqref{eq: N requirement} through both $\Gamma_m$,
which captures the retained variance reduction, and $b_m$,
which controls the range of the single-trajectory estimator. As $\alpha_m \to 0$ we have $\lambda(\alpha_m) \to 0$ and $\hat \mu^+ \to \hat \mu$, so $\underline{c}_{m, t} \to \Delta_{\pi, t}$ and $\Gamma_m$ converges to the exact one-step-optimal lower bound implied by Theorem~\ref{lem:one_step_var_smaller_stronger_optimal}.

\begin{table*}[ht]
  \begin{center}
  \begin{small}
  \setlength\tabcolsep{3pt}
\begin{tabular}
{llllll}
\toprule
& MDP & Data to learn $\mu$ & $\pi$ Parameterization  & Gridworld size &  Other environments \\
\midrule
ODI (our method) & \textbf{general} & \textbf{offline data} & \tb{no assumption} &  \textbf{27,000} & \textbf{Online Stochastic }\\
& &  & &  & \textbf{Matching,}\\
& &  & &  & \textbf{MuJoCo robotics}\\
BPG \citep{hanna2017data} & \tb{general} &  online data & need to be known & 1,600 & CartPole, Acrobot\\
ROS \citep{zhong2022robust}& \tb{general} &  online data & need to be known & 1,600&  CartPole \\
ReVar \citep{mukherjee2022revar}  & tree &  \tb{offline data} & \tb{no assumption} & 1,600 & 15 states tree-MDP \\
\bottomrule
\end{tabular}
\end{small}
\end{center}
\caption{
Our methods impose weaker assumptions on the data, and our empirical study covers more challenging tasks.
}
\label{table: compare}
\end{table*}

\section{Numerical Simulations}\label{sec:experiment}
We evaluate our methods empirically using three baselines for comparison:
\tb{(1)} the canonical on-policy Monte Carlo estimator,
\tb{(2)} off-policy Monte Carlo estimator with behavior policy search (BPG, \citet{hanna2017data}),
and \tb{(3)} robust on-policy sampling (ROS, \citet{zhong2022robust}).
We do not include ReVar \citep{mukherjee2022revar} because it assumes a tree-structured MDP, whereas the environments considered in our experiments do not impose this structural assumption.
As our work focuses on behavior policy design for unbiased Monte Carlo evaluation, we compare primarily against baselines within this class of unbiased estimators. Offline evaluation methods, which rely on value-function approximation and temporal-difference targets and may introduce bias \citep{levine2020offline}, operate in a different design space and are therefore outside the scope of our empirical comparison.
We first use the offline data to learn a behavior policy through Algorithm \ref{alg: ODI algorithm}, then the PDIS Monte Carlo estimator~\eqref{eq pdis def} is used to estimate the performance of the target policy,
using interactions generated by the learned behavior policy.
We call our method Offline Data Informed (ODI) algorithm.
Our method is superior in data requirements and applicability than related baselines, as summarized
in Table \ref{table: compare}.

We demonstrate the performance of our algorithm across three representative environments: Gridworld, a canonical reinforcement learning environment {\hypersetup{citecolor=black}\citep{sutton2018reinforcement}}; Online Stochastic Matching, a classical operations research problem {\hypersetup{citecolor=black}\citep{ma2020algorithms,aouad2026nonparametric}}; and MuJoCo, a standard benchmark suite for robotic control tasks {\hypersetup{citecolor=black}\citep{todorov2012mujoco}}.

\textbf{Gridworld:}
We first conduct experiments with linear function approximation in Gridworld   with $n^3$ states,
i.e.,
it is an $n \times n$ grid with the time horizon also being $n$.
Specifically, we use Gridworld with $n^3 = 1,000$ and $n^3 = 27,000$.
We use 30 target policies with a wide range of performance.
The offline data is generated by selecting random actions on a uniformly random state distribution.

We report the \emph{normalized estimation error} of the four methods against the number of environment interactions (steps).
Intuitively, normalized estimation error expresses an estimator's error relative to that of on-policy Monte Carlo.
More precisely,
we define the \emph{estimation error} at step $t$ as the absolute difference between an estimator and the ground truth divided by the ground truth.
The \emph{normalized estimation error} divides this error by the on-policy Monte Carlo estimator's average error after the first episode.
Consequently, the normalized estimation error of the on-policy Monte Carlo estimator starts at $1$.
\begin{figure}[H]
\includegraphics[width=0.65\textwidth]{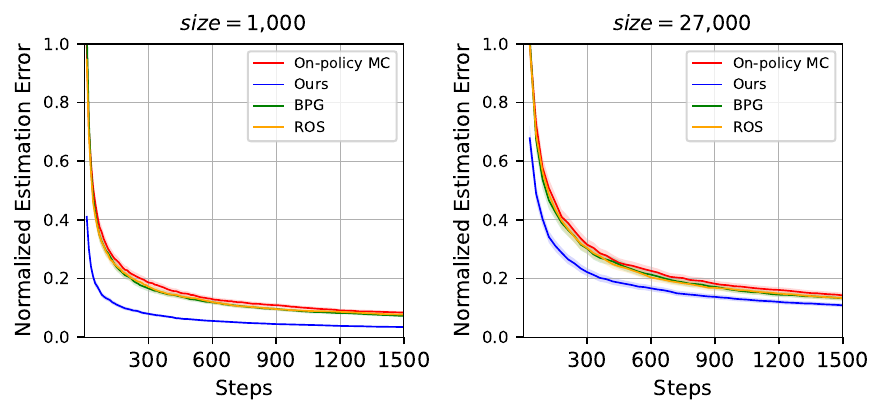}
\centering
\caption{
Results on Gridworld.
The curves are averaged over 900 trials (30 target policies, each having 30 independent runs).
Shading shows standard errors; some are too small to be visible.
}
\label{fig:pseudo_tabular}
\end{figure}

As shown in Figure \ref{fig:pseudo_tabular}, our method achieves substantially lower error than the baselines.
In particular,
the dotted line shows that, for Gridworld of size $1,000$,
to match the on-policy Monte Carlo estimation error at $250$ steps,
our methods only need around $50$ steps.
In Gridworld with size $27,000$,
to match the estimation error reached by on-policy Monte Carlo in $750$ steps,
our methods only need around $400$ steps, saving more than $40\%$ of online interactions.
The improvement in environments with $\text{size}=27,000$ is smaller than in environments with $\text{size}=1,000$ because the same amount of offline data is used despite the much larger set of possible transition tuples.
In fact, for the Gridworld with $\text{size}=27,000$, the offline sample size corresponds to approximately $23.15\%$ of the number of possible transition tuples.
More experiment details are in Appendix~\ref{append:gridworld}.

\begin{table*}[t]
\begin{center}
\begin{small}
  \setlength\tabcolsep{5pt}
\begin{tabular}
{lllll}
\toprule
 On-policy MC & Ours with $23.15\%$	& Ours with $46.3\%$ &	BPG	 & ROS \\
  & sample-to-support ratio & sample-to-support ratio & & \\
\midrule
	300&	150&	90&	300&	300 \\
	600&	330&	180&	540&	540 \\
	1200&	540&	420&	990&	990 \\
\bottomrule
\end{tabular}
\end{small}
\end{center}
\caption{The above table is an extension of Figure \ref{fig:pseudo_tabular} comparing two offline sample sizes, corresponding to sample-to-support ratios of approximately $23.15\%$ and $46.3\%$, for our algorithm in Gridworld with $\text{size}=27,000$. Each number represents the number of steps needed to achieve the same estimation accuracy that on-policy Monte Carlo achieves with $300/600/1200$ steps. All numbers are averaged from $900$ different runs over a wide range of policies. Standard errors are visualized in Figure \ref{fig:pseudo_tabular}  and are invisible for some algorithm curves because they are too small.
}
\label{table: gridworld scale}
\end{table*}

We also show our algorithm scales with offline data. As we increase the amount of offline data in the Gridworld with $\text{size}=27,000$ using additional offline data from many different distributions, the sample savings increase from $55\%=(1200-540)/1200$ to $65\% = (1200-420)/1200$ in the last row of Table \ref{table: gridworld scale}. By comparison, the best algorithm among all previous state-of-the-art methods only saves $17.5\% = (1200-990)/1200$ samples and does not have a mechanism to use offline data because they can only utilize online trajectory.

\textbf{Online Stochastic Matching}

\begin{figure}[ht]
\begin{minipage}{0.48\textwidth}
\centering
\includegraphics[width=1\textwidth]{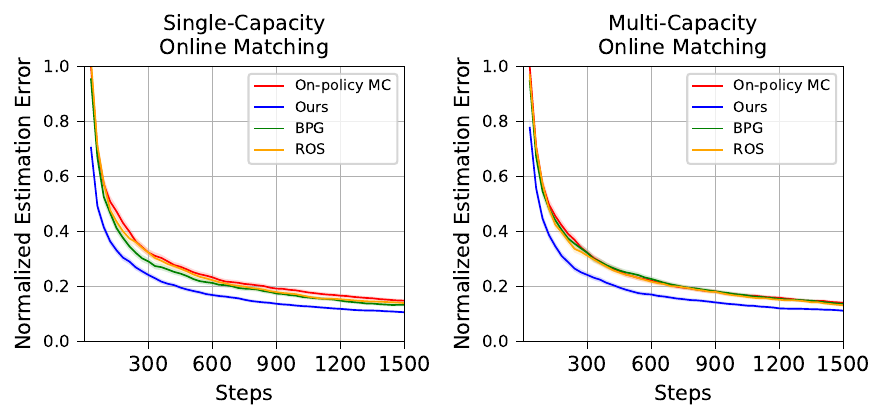}
\caption{Results on Online Stochastic Matching.
Each curve is averaged over 900 runs (30 target policies, each having 30 independent runs).
Shaded regions denote standard errors and are invisible for some curves because they are too small.}
\label{fig:resource allocation}
\end{minipage}
\hfill
\begin{minipage}{0.48\textwidth}
\begin{table}[H]
\centering
\small
\begin{tabular}{lrlrr}
\toprule
  & On-policy MC & \textbf{Ours} & BPG & ROS \\
\midrule
Single-Capacity& 1.000 & \textbf{0.513} & 0.878 & 0.958 \\
Multi-Capacity& 1.000 & \textbf{0.585} & 0.921 & 0.983 \\
\bottomrule
\end{tabular}
\caption{Relative variance of estimators on Online Stochastic Matching.
The relative variance is defined as the variance of each estimator divided by the variance of the on-policy Monte Carlo estimator.
Numbers are averaged over 900 independent runs (30 target policies, each having 30 independent runs).}
\label{table: resource allocation variance ratio}
\end{table}
\end{minipage}
\end{figure}

We evaluate two settings that differ in resource capacity, where $\kappa_j \in \mathbb{Z}_{\ge0}$ denotes the initial capacity (i.e., number of available units) of resource type $j$: (i) single-capacity ($\kappa_j=1$) and (ii) multi-capacity ($\kappa_j=2$ for all $j$). While both settings share the same dynamics, increasing capacity significantly enlarges the state space, which scales as $|\mathcal{S}| \propto T \cdot m \cdot \prod_{j=1}^n (\kappa_j+1)$. In particular, with horizon $T=30$, customer types $m=10$, and resource types $n=10$, the capacity-$2$ environment contains approximately $17.7$ million $\approx 30 \cdot 10\cdot  3^{10}$ states, highlighting the combinatorial complexity of heterogeneous resource allocation.

As shown in Figure~\ref{fig:resource allocation}, our method requires fewer samples to achieve the same estimation error compared with the on-policy Monte Carlo method. It also consistently outperforms other baseline methods across both single- and multi-capacity settings. Table~\ref{table: resource allocation variance ratio} shows that our method reduces variance by approximately $42\%$--$49\%$ relative to the on-policy Monte Carlo method. These results demonstrate the improved efficiency achieved by the proposed behavior policy while preserving unbiased estimation.

\begin{figure*}[t]
  \includegraphics[width=1.02\textwidth]{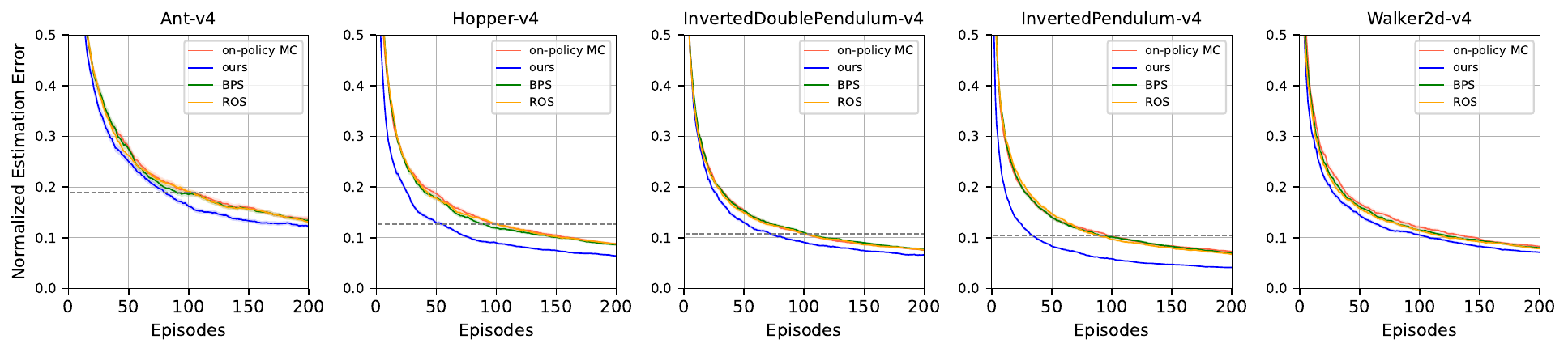}
  \centering
  \caption{
  Results on Mujoco environments. Each curve averages $900$ trials (30 target policies, each having 30 independent runs). Shading indicates standard errors, which are too small to see for some curves.
  }
  \label{fig:mujoco}
\end{figure*}

\begin{table*}[h]
\begin{center}
\begin{small}
  \setlength\tabcolsep{5pt}
\begin{tabular}
{llllll}
\toprule
 &On-policy MC & Ours &   BPG   &  ROS & Saved-episode ratio to best baseline\\
\midrule
Ant   & 100 & \textbf{81} & 91 & 103 & (100-81)/(100-91)$\approx$211.1\%\\
Hopper & 100 & \textbf{54} & 89 & 100 & (100-54)/(100-89)$\approx$418.2\% \\
I. D. Pendulum & 100 & \textbf{72} & 103 & 99 & (100-72)/(100-99)=2800\% \\
I. Pendulum & 100 & \textbf{35} & 95 & 90 &(100-35)/(100-90)=650\% \\
Walker & 100 & \textbf{70} & 92 & 91 &  (100-70)/(100-91)$\approx$333.3\% \\
\bottomrule
\end{tabular}
\end{small}
\end{center}
\caption{Episodes needed to achieve the same estimation accuracy that on-policy Monte Carlo achieves with $100$ episodes.
}
\label{table: compare num}
\end{table*}

\textbf{MuJoCo:}
We then conduct experiments with neural network function approximation in MuJoCo robot simulation tasks \citep{todorov2012mujoco}.
We report the normalized estimator error in Figure \ref{fig:mujoco},
where our methods consistently outperform the baselines.
In particular,
as shown by the dotted line in Figure \ref{fig:mujoco} and Table~\ref{table: compare num},
our methods substantially reduce the required episodes (saving up to $65\%$ episodes) to achieve the estimation error that the on-policy Monte Carlo estimator achieves with $100$ episodes.
Because MuJoCo episodes can differ in length,
we also provide in Appendix~\ref{append:mujoco} a version of Figure~\ref{fig:mujoco} with the $x$-axis being steps,
where our methods are still consistently better.

All hyperparameters used by our methods to learn $\hat{\mu}$ are tuned offline and are the same across all MuJoCo, Online Stochastic Matching, and Gridworld experiments.

\section{Conclusion}
\label{sec: conclusion}

Monte Carlo methods remain a fundamental tool for policy evaluation in reinforcement learning and sequential decision-making problems studied in operations research. They are widely used for performance assessment, algorithm comparison, and model selection. However, standard on-policy Monte Carlo evaluation often requires extensive online interaction, which can be costly or impractical in real-world settings. In this paper, we improve the data efficiency of unbiased online Monte Carlo evaluation by designing an optimal behavior policy for data collection. Specifically, we derive an optimal closed-form behavior policy that provably reduces estimator variance (Theorem~\ref{lem:rl-optimal}). We also develop practical algorithms for learning a one-step optimal behavior policy (Theorem~\ref{thm: one step optimal}) from previously collected offline data (Theorem~\ref{lem:hat-q-recursive}). We further quantify how relative error in the learned auxiliary value function affects the retained variance reduction and the resulting high-probability online sample requirement (Theorem~\ref{thm: robust local} \& Theorem~\ref{thm: end to end}). Empirical results demonstrate consistent improvements in evaluation efficiency across diverse environments compared with traditional on-policy Monte Carlo evaluation and other baseline methods. These results provide a comprehensive study for behavior policy optimization in off-policy evaluation, with implications for reinforcement learning and operations research. Possible future directions include extending this framework to temporal-difference learning and related policy improvement methods.

\section*{Acknowledgments}
This work is supported in part by the U.S. National Science Foundation under Awards III-2128019, SLES-2331904, and CAREER-2442098; the Commonwealth Cyber Initiative's Central Virginia Node under Award VV-1Q26-001; and a Cisco Faculty Research Award.

\clearpage
\bibliographystyle{informs2014}
\bibliography{bibliography}

\begin{thebibliography}{71}
\providecommand{\natexlab}[1]{#1}
\providecommand{\url}[1]{\texttt{#1}}
\providecommand{\urlprefix}{URL }

\bibitem[{Antos et~al.(2008)Antos, Szepesv{\'{a}}ri, \protect\BIBand{}
  Munos}]{antos2008learning}
Antos A, Szepesv{\'{a}}ri C, Munos R (2008) Learning near-optimal policies with
  bellman-residual minimization based fitted policy iteration and a single
  sample path. \emph{Machine Learning} .

\bibitem[{Aouad \protect\BIBand{} Ma(2026)}]{aouad2026nonparametric}
Aouad A, Ma W (2026) A nonparametric framework for online stochastic matching
  with correlated arrivals. \emph{Management Science} .

\bibitem[{Bertsekas(2012)}]{bertsekas2012dynamic}
Bertsekas D (2012) \emph{Dynamic programming and optimal control: Volume I},
  volume~4 (Athena scientific).

\bibitem[{Boucheron et~al.(2013)Boucheron, Lugosi, \protect\BIBand{}
  Massart}]{boucheron2013concentration}
Boucheron S, Lugosi G, Massart P (2013) \emph{Concentration Inequalities: A
  Nonasymptotic Theory of Independence} (Oxford University Press), ISBN
  9780199535255,
  \urlprefix\url{http://dx.doi.org/10.1093/acprof:oso/9780199535255.001.0001}.

\bibitem[{Chen et~al.(2025)Chen, Liu, \protect\BIBand{}
  Zhang}]{chen2024efficient}
Chen C, Liu S, Zhang S (2025) Efficient policy evaluation with safety
  constraint for reinforcement learning. \emph{Proceedings of the International
  Conference on Learning Representations}.

\bibitem[{Chen et~al.(2026{\natexlab{a}})Chen, Liu, Luo, Chandra, Jiang,
  \protect\BIBand{} Zhang}]{chen2026robust}
Chen C, Liu SD, Luo L, Chandra R, Jiang N, Zhang S (2026{\natexlab{a}}) Robust
  data-collection policy learning for low-variance online policy evaluation.
  \emph{arXiv preprint arXiv:2608.24146} .

\bibitem[{Chen \protect\BIBand{} Zhang(2026{\natexlab{a}})}]{chen2026fast}
Chen C, Zhang Y (2026{\natexlab{a}}) Fast rates in $\alpha$-potential games via
  regularized mirror descent. \emph{ArXiv Preprint arXiv:2605.00268} .

\bibitem[{Chen \protect\BIBand{} Zhang(2026{\natexlab{b}})}]{chen2026pessimism}
Chen C, Zhang Y (2026{\natexlab{b}}) Pessimism-free offline learning in
  general-sum games via {KL} regularization. \emph{ArXiv Preprint
  arXiv:2605.00264} .

\bibitem[{Chen et~al.(2026{\natexlab{b}})Chen, Zhang, Liu, Xie, Liu,
  \protect\BIBand{} Jiang}]{chen2026offline}
Chen C, Zhang Y, Liu X, Xie Z, Liu SD, Jiang N (2026{\natexlab{b}}) Offline
  two-player zero-sum markov games with {KL} regularization. \emph{Forty-third
  International Conference on Machine Learning},
  \urlprefix\url{https://openreview.net/forum?id=cQD2wxFhFG}.

\bibitem[{Chen \protect\BIBand{} Jiang(2019)}]{chen2019information}
Chen J, Jiang N (2019) Information-theoretic considerations in batch
  reinforcement learning. \emph{Proceedings of International Conference on
  Machine Learning}.

\bibitem[{Chua et~al.(2018)Chua, Calandra, McAllister, \protect\BIBand{}
  Levine}]{chua2018deep}
Chua K, Calandra R, McAllister R, Levine S (2018) Deep reinforcement learning
  in a handful of trials using probabilistic dynamics models. \emph{Advances in
  Neural Information Processing Systems}.

\bibitem[{Deisenroth \protect\BIBand{} Rasmussen(2011)}]{deisenroth2011pilco}
Deisenroth MP, Rasmussen CE (2011) {PILCO:} {A} model-based and data-efficient
  approach to policy search. \emph{Proceedings of the International Conference
  on Machine Learning}.

\bibitem[{Dud{\'\i}k et~al.(2011)Dud{\'\i}k, Langford, \protect\BIBand{}
  Li}]{dudik2011doubly}
Dud{\'\i}k M, Langford J, Li L (2011) Doubly robust policy evaluation and
  learning. \emph{arXiv preprint arXiv:1103.4601} .

\bibitem[{Geweke(1988)}]{geweke1988antithetic}
Geweke J (1988) Antithetic acceleration of monte carlo integration in bayesian
  inference. \emph{Journal of Econometrics} .

\bibitem[{Gijsbrechts et~al.(2022)Gijsbrechts, Boute, Van~Mieghem,
  \protect\BIBand{} Zhang}]{gijsbrechts2022can}
Gijsbrechts J, Boute RN, Van~Mieghem JA, Zhang DJ (2022) Can deep reinforcement
  learning improve inventory management? performance on lost sales,
  dual-sourcing, and multi-echelon problems. \emph{Manufacturing \& Service
  Operations Management} 24(3):1349--1368.

\bibitem[{Hanna et~al.(2017)Hanna, Thomas, Stone, \protect\BIBand{}
  Niekum}]{hanna2017data}
Hanna JP, Thomas PS, Stone P, Niekum S (2017) Data-efficient policy evaluation
  through behavior policy search. \emph{Proceedings of the International
  Conference on Machine Learning}.

\bibitem[{Hesterberg(1995)}]{hesterberg1995weighted}
Hesterberg T (1995) Weighted average importance sampling and defensive mixture
  distributions. \emph{Technometrics} .

\bibitem[{Horn \protect\BIBand{} Johnson(2012)}]{horn2012matrix}
Horn RA, Johnson CR (2012) \emph{Matrix analysis (2nd Edition)} (Cambridge
  university press).

\bibitem[{Huang et~al.(2022)Huang, Dossa, Ye, Braga, Chakraborty, Mehta,
  \protect\BIBand{} Ara{\'u}jo}]{huang2022cleanrl}
Huang S, Dossa RFJ, Ye C, Braga J, Chakraborty D, Mehta K, Ara{\'u}jo JG (2022)
  Cleanrl: High-quality single-file implementations of deep reinforcement
  learning algorithms. \emph{Journal of Machine Learning Research} .

\bibitem[{Jiang \protect\BIBand{} Li(2016)}]{jiang2015doubly}
Jiang N, Li L (2016) marginalized robust off-policy value evaluation for
  reinforcement learning. \emph{Proceedings of the International Conference on
  Machine Learning}.

\bibitem[{Kakutani(1945)}]{kakutani1945markoff}
Kakutani S (1945) Markoff process and the dirichlet problem. \emph{Proceedings
  of the Japan Academy}.

\bibitem[{Kalashnikov et~al.(2018)Kalashnikov, Irpan, Pastor, Ibarz, Herzog,
  Jang, Quillen, Holly, Kalakrishnan, Vanhoucke
  et~al.}]{kalashnikov2018scalable}
Kalashnikov D, Irpan A, Pastor P, Ibarz J, Herzog A, Jang E, Quillen D, Holly
  E, Kalakrishnan M, Vanhoucke V, et~al. (2018) Scalable deep reinforcement
  learning for vision-based robotic manipulation. \emph{Proceedings of
  Conference on Robot Learning}.

\bibitem[{Kallus \protect\BIBand{} Uehara(2019)}]{kallus2019intrinsically}
Kallus N, Uehara M (2019) Intrinsically efficient, stable, and bounded
  off-policy evaluation for reinforcement learning. \emph{Advances in neural
  information processing systems} 32.

\bibitem[{Kallus \protect\BIBand{} Uehara(2020)}]{kallus2020double}
Kallus N, Uehara M (2020) Double reinforcement learning for efficient
  off-policy evaluation in markov decision processes. \emph{Journal of Machine
  Learning Research} 21(167):1--63.

\bibitem[{Kingma \protect\BIBand{} Ba(2015)}]{kingma2014adam}
Kingma DP, Ba J (2015) Adam: {A} method for stochastic optimization.
  \emph{Proceedings of the International Conference on Learning
  Representations}.

\bibitem[{Koller \protect\BIBand{} Friedman(2009)}]{koller2009probabilistic}
Koller D, Friedman N (2009) \emph{Probabilistic Graphical Models: Principles
  and Techniques} (Mit Press).

\bibitem[{Le et~al.(2019)Le, Voloshin, \protect\BIBand{} Yue}]{le2019batch}
Le HM, Voloshin C, Yue Y (2019) Batch policy learning under constraints.
  \emph{Proceedings of the International Conference on Machine Learning}.

\bibitem[{Levine et~al.(2020)Levine, Kumar, Tucker, \protect\BIBand{}
  Fu}]{levine2020offline}
Levine S, Kumar A, Tucker G, Fu J (2020) Offline reinforcement learning:
  Tutorial, review, and perspectives on open problems. \emph{arXiv preprint
  arXiv:2005.01643} .

\bibitem[{Liu et~al.(2018)Liu, Li, Tang, \protect\BIBand{}
  Zhou}]{liu2018breaking}
Liu Q, Li L, Tang Z, Zhou D (2018) Breaking the curse of horizon:
  Infinite-horizon off-policy estimation. \emph{Advances in Neural Information
  Processing Systems}.

\bibitem[{Liu et~al.(2025{\natexlab{a}})Liu, Chen, \protect\BIBand{}
  Zhang}]{liu2024doubly}
Liu S, Chen C, Zhang S (2025{\natexlab{a}}) Doubly optimal policy evaluation
  for reinforcement learning. \emph{Proceedings of the International Conference
  on Learning Representations}.

\bibitem[{Liu et~al.(2025{\natexlab{b}})Liu, Chen, \protect\BIBand{}
  Zhang}]{liu2024efficientmul}
Liu S, Chen Y, Zhang S (2025{\natexlab{b}}) Efficient multi-policy evaluation
  for reinforcement learning. \emph{Proceedings of the AAAI Conference on
  Artificial Intelligence}.

\bibitem[{Liu \protect\BIBand{} Zhang(2024)}]{liu2024efficient}
Liu S, Zhang S (2024) Efficient policy evaluation with offline data informed
  behavior policy design. \emph{Proceedings of the International Conference on
  Machine Learning}.

\bibitem[{Liu(2025)}]{liu2025efficientrobust}
Liu SD (2025) \emph{Efficient and Robust Policy Evaluation for Reinforcement
  Learning}. Ph.D. thesis, University of Virginia.

\bibitem[{Liu et~al.(2025{\natexlab{c}})Liu, Chen, \protect\BIBand{}
  Zhang}]{liu2025ode}
Liu SD, Chen S, Zhang S (2025{\natexlab{c}}) The ode method for stochastic
  approximation and reinforcement learning with markovian noise. \emph{Journal
  of Machine Learning Research} 26(24):1--76.

\bibitem[{Liu et~al.(2026)Liu, Simchi-Levi, Chen, Gao, \protect\BIBand{}
  Zhang}]{liu2026or}
Liu SD, Simchi-Levi D, Chen C, Gao C, Zhang S (2026) Or-transformer: Scaling
  real-time decision-making to 1,000 items. \emph{arXiv preprint
  arXiv:2609.01933} .

\bibitem[{Ma \protect\BIBand{} Simchi-Levi(2020)}]{ma2020algorithms}
Ma W, Simchi-Levi D (2020) Algorithms for online matching, assortment, and
  pricing with tight weight-dependent competitive ratios. \emph{Operations
  Research} 68(6):1787--1803.

\bibitem[{Maei(2011)}]{maei2011gradient}
Maei HR (2011) \emph{Gradient temporal-difference learning algorithms}. Ph.D.
  thesis, University of Alberta.

\bibitem[{Mahadevan et~al.(2026)Mahadevan, Chen, Liu, \protect\BIBand{}
  Zhang}]{mahadevan2026convergence}
Mahadevan V, Chen C, Liu SD, Zhang S (2026) Convergence of two-timescale
  {Markovian} stochastic approximations with applications in reinforcement
  learning. \emph{Proceedings of the 43rd International Conference on Machine
  Learning}.

\bibitem[{Mukherjee et~al.(2022)Mukherjee, Hanna, \protect\BIBand{}
  Nowak}]{mukherjee2022revar}
Mukherjee S, Hanna JP, Nowak RD (2022) Revar: Strengthening policy evaluation
  via reduced variance sampling. \emph{Proceedings of the Conference in
  Uncertainty in Artificial Intelligence}.

\bibitem[{Munos(2003)}]{munos2003error}
Munos R (2003) Error bounds for approximate policy iteration. \emph{Proceedings
  of the International Conference on Machine Learning}.

\bibitem[{Munos \protect\BIBand{} Szepesv{\'a}ri(2008)}]{munos2008finite}
Munos R, Szepesv{\'a}ri C (2008) Finite-time bounds for fitted value iteration.
  \emph{Journal of Machine Learning Research} .

\bibitem[{Nachum et~al.(2019)Nachum, Chow, Dai, \protect\BIBand{}
  Li}]{nachum2019dualdice}
Nachum O, Chow Y, Dai B, Li L (2019) Dualdice: Behavior-agnostic estimation of
  discounted stationary distribution corrections. \emph{Advances in Neural
  Information Processing Systems}.

\bibitem[{O'Donoghue et~al.(2018)O'Donoghue, Osband, Munos, \protect\BIBand{}
  Mnih}]{o2017uncertainty}
O'Donoghue B, Osband I, Munos R, Mnih V (2018) The uncertainty bellman equation
  and exploration. \emph{Proceedings of the International Conference on Machine
  Learning}.

\bibitem[{Owen(2013)}]{mcbook}
Owen AB (2013) \emph{Monte Carlo theory, methods and examples}.

\bibitem[{Powell(2007)}]{powell2007approximate}
Powell WB (2007) \emph{Approximate Dynamic Programming - Solving the Curses of
  Dimensionality}.

\bibitem[{Precup et~al.(2001)Precup, Sutton, \protect\BIBand{}
  Dasgupta}]{precup2001off}
Precup D, Sutton RS, Dasgupta S (2001) Off-policy temporal difference learning
  with function approximation. \emph{Proceedings of the International
  Conference on Machine Learning}.

\bibitem[{Precup et~al.(2000{\natexlab{a}})Precup, Sutton, \protect\BIBand{}
  Singh}]{precup2000eligibility}
Precup D, Sutton RS, Singh S (2000{\natexlab{a}}) Eligibility traces for
  off-policy policy evaluation .

\bibitem[{Precup et~al.(2000{\natexlab{b}})Precup, Sutton, \protect\BIBand{}
  Singh}]{precup:2000:eto:645529.658134}
Precup D, Sutton RS, Singh SP (2000{\natexlab{b}}) Eligibility traces for
  off-policy policy evaluation. \emph{Proceedings of the International
  Conference on Machine Learning}.

\bibitem[{Puterman(1990)}]{puterman1990markov}
Puterman ML (1990) Markov decision processes. \emph{Handbooks in operations
  research and management science} 2:331--434.

\bibitem[{Puterman(2014)}]{puterman2014markov}
Puterman ML (2014) \emph{Markov decision processes: discrete stochastic dynamic
  programming} (John Wiley \& Sons).

\bibitem[{Rubinstein(1981)}]{Rubinstein1981Simulation}
Rubinstein RY (1981) \emph{Simulation and the Monte Carlo Method} (Wiley).

\bibitem[{Schulman et~al.(2017)Schulman, Wolski, Dhariwal, Radford,
  \protect\BIBand{} Klimov}]{schulman2017proximal}
Schulman J, Wolski F, Dhariwal P, Radford A, Klimov O (2017) Proximal policy
  optimization algorithms. \emph{arXiv preprint arXiv:1707.06347} .

\bibitem[{Sherstan et~al.(2018)Sherstan, Bennett, Young, Ashley, White, White,
  \protect\BIBand{} Sutton}]{sherstan2018directly}
Sherstan C, Bennett B, Young K, Ashley DR, White A, White M, Sutton RS (2018)
  Directly estimating the variance of the $\lambda$-return using
  temporal-difference methods. \emph{arXiv preprint arXiv:1801.08287} .

\bibitem[{Sutton(1988)}]{sutton1988learning}
Sutton RS (1988) Learning to predict by the methods of temporal differences.
  \emph{Machine Learning} .

\bibitem[{Sutton(1990)}]{sutton1990integrated}
Sutton RS (1990) Integrated architectures for learning, planning, and reacting
  based on approximating dynamic programming. \emph{Proceedings of the
  International Conference on Machine Learning}.

\bibitem[{Sutton \protect\BIBand{} Barto(2018)}]{sutton2018reinforcement}
Sutton RS, Barto AG (2018) \emph{Reinforcement Learning: An Introduction (2nd
  Edition)} (MIT press).

\bibitem[{Sutton et~al.(2016)Sutton, Mahmood, \protect\BIBand{}
  White}]{sutton2016emphatic}
Sutton RS, Mahmood AR, White M (2016) An emphatic approach to the problem of
  off-policy temporal-difference learning. \emph{Journal of Machine Learning
  Research} .

\bibitem[{Sutton et~al.(2008)Sutton, Szepesv{\'{a}}ri, Geramifard,
  \protect\BIBand{} Bowling}]{sutton2012dyna}
Sutton RS, Szepesv{\'{a}}ri C, Geramifard A, Bowling MH (2008) Dyna-style
  planning with linear function approximation and prioritized sweeping.
  \emph{Proceedings of the Conference in Uncertainty in Artificial
  Intelligence}.

\bibitem[{Tamar et~al.(2016)Tamar, Castro, \protect\BIBand{}
  Mannor}]{variance2016Tamar}
Tamar A, Castro DD, Mannor S (2016) Learning the variance of the reward-to-go.
  \emph{Journal of Machine Learning Research} .

\bibitem[{Thomas \protect\BIBand{} Brunskill(2016)}]{thomas2016data}
Thomas P, Brunskill E (2016) Data-efficient off-policy policy evaluation for
  reinforcement learning. \emph{International conference on machine learning},
  2139--2148 (PMLR).

\bibitem[{Thomas(2015)}]{thomas2015safe}
Thomas PS (2015) \emph{Safe reinforcement learning}. Ph.D. thesis, University
  of Massachusetts Amherst.

\bibitem[{Todorov et~al.(2012)Todorov, Erez, \protect\BIBand{}
  Tassa}]{todorov2012mujoco}
Todorov E, Erez T, Tassa Y (2012) Mujoco: {A} physics engine for model-based
  control. \emph{Proceedings of the International Conference on Intelligent
  Robots and Systems}.

\bibitem[{Vinyals et~al.(2019)Vinyals, Babuschkin, Czarnecki, Mathieu, Dudzik,
  Chung, Choi, Powell, Ewalds, Georgiev, Oh, Horgan, Kroiss, Danihelka, Huang,
  Sifre, Cai, Agapiou, Jaderberg, Vezhnevets, Leblond, Pohlen, Dalibard,
  Budden, Sulsky, Molloy, Paine, G{\"{u}}l{\c{c}}ehre, Wang, Pfaff, Wu, Ring,
  Yogatama, W{\"{u}}nsch, McKinney, Smith, Schaul, Lillicrap, Kavukcuoglu,
  Hassabis, Apps, \protect\BIBand{} Silver}]{vinyals2019grandmaster}
Vinyals O, Babuschkin I, Czarnecki WM, Mathieu M, Dudzik A, Chung J, Choi DH,
  Powell R, Ewalds T, Georgiev P, Oh J, Horgan D, Kroiss M, Danihelka I, Huang
  A, Sifre L, Cai T, Agapiou JP, Jaderberg M, Vezhnevets AS, Leblond R, Pohlen
  T, Dalibard V, Budden D, Sulsky Y, Molloy J, Paine TL, G{\"{u}}l{\c{c}}ehre
  {\c{C}}, Wang Z, Pfaff T, Wu Y, Ring R, Yogatama D, W{\"{u}}nsch D, McKinney
  K, Smith O, Schaul T, Lillicrap TP, Kavukcuoglu K, Hassabis D, Apps C, Silver
  D (2019) Grandmaster level in starcraft {II} using multi-agent reinforcement
  learning. \emph{Nature} .

\bibitem[{Wang et~al.(2024)Wang, Gao, \protect\BIBand{} Zha}]{wang2024reliable}
Wang J, Gao R, Zha H (2024) Reliable off-policy evaluation for reinforcement
  learning. \emph{Operations Research} .

\bibitem[{Wang et~al.(2026)Wang, Srinivasa, Chen, Liu, Payani,
  \protect\BIBand{} Zhang}]{wang2026predicting}
Wang J, Srinivasa J, Chen C, Liu SD, Payani A, Zhang S (2026) Predicting
  plasticity in deep continual learning: A theoretical perspective. \emph{arXiv
  preprint arXiv:2605.09044} .

\bibitem[{Watkins(1989)}]{watkins1989learning}
Watkins CJCH (1989) \emph{Learning from delayed rewards}. Ph.D. thesis, King's
  College, Cambridge.

\bibitem[{Xie et~al.(2019)Xie, Ma, \protect\BIBand{} Wang}]{xie2019towards}
Xie T, Ma Y, Wang Y (2019) Towards optimal off-policy evaluation for
  reinforcement learning with marginalized importance sampling. \emph{Advances
  in Neural Information Processing Systems}.

\bibitem[{Xie et~al.(2026)Xie, Liu, Chen, Liu, Chandra, \protect\BIBand{}
  Zhang}]{xie2026beyond}
Xie Z, Liu X, Chen C, Liu SD, Chandra R, Zhang S (2026) Beyond linear
  attention: {Softmax} transformers implement in-context reinforcement
  learning. \emph{arXiv preprint arXiv:2605.07333} .

\bibitem[{Zhang(2022)}]{zhang2022thesis}
Zhang S (2022) \emph{Breaking the deadly triad in reinforcement learning}.
  Ph.D. thesis, University of Oxford.

\bibitem[{Zhang et~al.(2026)Zhang, Chen, \protect\BIBand{}
  Jiang}]{zhang2026pessimism}
Zhang Y, Chen C, Jiang N (2026) Beyond pessimism: Offline learning in
  kl-regularized games. \emph{arXiv preprint arXiv:2604.06738} .

\bibitem[{Zhong et~al.(2022)Zhong, Zhang, Sch{\"a}fer, Albrecht,
  \protect\BIBand{} Hanna}]{zhong2022robust}
Zhong R, Zhang D, Sch{\"a}fer L, Albrecht SV, Hanna JP (2022) Robust on-policy
  sampling for data-efficient policy evaluation in reinforcement learning.
  \emph{Advances in Neural Information Processing Systems}.

\end{thebibliography}

\newpage
\ECSwitch
\ECDisclaimer

\section{Proofs}
\subsection{Proof of Lemma~\ref{lem stats unbiasedness}}
\label{sec proof lem stats unbiasedness}
\begin{proof}
\begin{align}
  \E_{A\sim\mu}\left[\rho(A)q(A)\right] =& \sum_{a \in \qty{a|\mu(a) > 0}} \mu(a) \frac{\pi(a)}{\mu(a)} q(a) \\
  =& \sum_{a \in \qty{a|\mu(a) > 0}} \pi(a) q(a) \\
  =& \sum_{a \in \qty{a|\mu(a) > 0}} \pi(a) q(a) + \sum_{a \in \qty{a | \mu(a) = 0}} \pi(a)q(a) \explain{$\mu \in \Lambda$} \\
  =&\sum_a \pi(a)q(a) \\
  =&\E_{A\sim\pi}\left[q(A)\right].
\end{align}

The intuition in the third equation is that the sample $a$ where $\mu$ does not cover $\pi$ must satisfy $q(a) = 0$,
i.e.,
this sample does not contribute to the expectation anyway.
\end{proof}
\subsection{Proof of Lemma~\ref{lem:math-optimal}}\label{append:math-optimal}

\begin{proof}\hspace{1cm}\\
For a given $\pi$ and $q$,
define
\begin{align}
  \fA_+ \doteq \qty{a \mid \pi(a)q(a) \neq 0}.
\end{align}
For any $\mu \in \Lambda$,
we expand the variance as
\begin{align}
\label{eq:math-variance}
&\V_{A\sim \mu}(\rho(A)q(A)) \\
=& \E_{A\sim \mu}[(\rho(A)q(A))^2] - \E_{A\sim \mu}^2[\rho(A)q(A)] \\
=& \E_{A\sim \mu}[(\rho(A)q(A))^2] - \E_{A\sim \pi}^2[q(A)]  \explain{Lemma~\ref{lem stats unbiasedness}} \\
=& \sum_{a \in \qty{a \mid \mu(a) > 0}} \frac{\pi^2(a)q^2(a)}{\mu(a)}- \E_{A\sim \pi}^2[q(A)] \\
=& \sum_{a \in \qty{a \mid \mu(a) > 0} \cap \fA_+} \frac{\pi^2(a)q^2(a) }{\mu(a)}- \E_{A\sim \pi}^2[q(A)] \explain{$\pi(a)q(a) = 0, \forall a \notin \fA_+$} \\
=& \sum_{a \in \fA_+} \frac{\pi^2(a)q^2(a) }{\mu(a)}- \E_{A\sim \pi}[q(A)]^2 \explain{$\mu \in \Lambda$}
\end{align}
The second term is a constant and is unrelated to $\mu$.
Solving the optimization problem \eqref{eq:math-optimization} is,
therefore, equivalent to solving
\begin{align}
\text{min}_{\mu \in \Lambda}  \quad &
\sum_{a \in \fA_+} \frac{\pi^2(a)q^2(a) }{\mu(a)} \label{eq:math-optimization-2}.
\end{align}
\textbf{Case 1: $\abs{\fA_+} = 0$} \\
In this case,
the variance is always $0$ so any $\mu \in \Lambda$ is optimal.
In particular, $\mu^*(a) = \frac{1}{\abs{\fA}}$ is optimal. \\
\textbf{Case 2: $\abs{\fA_+} > 0$} \\
The definition of $\Lambda$ in~\eqref{eq stats search space} can be equivalently expressed, using contraposition, as
\begin{align}
  \Lambda = \qty{\mu \in \Delta(\fA) \mid \forall a, a \in \fA_+ \implies \mu(a) > 0}.
\end{align}
The optimization problem~\eqref{eq:math-optimization-2} can then be equivalently written as
\begin{align}
 \text{min}_{\mu \in \Delta(\fA)}  \quad &
\sum_{a \in \fA_+} \frac{\pi^2(a)q^2(a) }{\mu(a)} \label{eq:math-optimization-3} \\
\text{s.t.} \quad &
\mu(a) > 0 \quad \forall a \in \fA_+.
\end{align}
If for some $\mu$ we have
$\sum_{a\in \fA_+} \mu(a) < 1$,
then there must exist some $a_0 \notin \fA_+$ such that $\mu(a_0) > 0$.
Since $a_0$ does not contribute to the summation in the objective function of~\eqref{eq:math-optimization-3},
we can move the probability mass on $a_0$ to some other $a_1 \in \fA_+$ to increase $\mu(a_1)$ to further decrease the objective.
In other words,
any optimal solution $\mu$ to~\eqref{eq:math-optimization-3} must put all its mass on $\fA_+$.
This motivates the following problem
\begin{align}
\text{min}_{z \in \Delta(\fA_+)}  \quad &
\sum_{a \in \fA_+} \frac{\pi^2(a)q^2(a)}{z(a)} \label{eq:math-optimization-4} \\
\text{s.t.} \quad &
z(a) > 0 \quad \forall a \in \fA_+.
\end{align}
In particular, if $z_*$ is an optimal solution to~\eqref{eq:math-optimization-4},
then an optimal solution to~\eqref{eq:math-optimization-3} can be constructed as
\begin{align}
  \label{eq opt mu construnction}
  \mu_*(a) = \begin{cases}
    z_*(a) & a \in \fA_+ \\
    0 & \text{otherwise}.
  \end{cases}
\end{align}
Let $\R_{++} \doteq (0, +\infty)$.

According to the Cauchy-Schwarz inequality,
for any $z \in \R_{++}^\abs{\fA_+}$,
we have
\begin{align}
 \left(\sum_{a \in \fA_+} \frac{\pi^2(a)q^2(a)}{z(a)}\right)\left(\sum_{a \in \fA_+} z(a)\right) \geq \left(\sum_{a\in\fA_+} \frac{\pi(a)\abs{q(a)}}{\sqrt{z(a)}} \sqrt{z(a)}\right)^2 = \left(\sum_{a\in\fA_+} \pi(a)\abs{q(a)}\right)^2.
\end{align}
It can be easily verified that the equality holds for
\begin{align}
  z^*(a) \doteq \frac{\pi(a)\abs{q(a)}}{\sum_{b} \pi(b)\abs{q(b)}} > 0.
\end{align}
Since $\sum_{a\in\fA_+} z^*(a) = 1$,
we conclude that $z^*$ is an optimal solution to~\eqref{eq:math-optimization-4}.
An optimal solution $\mu_*$ to~\eqref{eq:math-optimization} can then be constructed according to~\eqref{eq opt mu construnction}.
Making use of the fact that $\pi(a)\abs{q(a)} = 0$ for $a \notin \fA_+$,
this $\mu_*$ can be equivalently expressed as
\begin{align}
  \mu_*(a) = \frac{\pi(a)\abs{q(a)}}{\sum_{b \in \fA} \pi(b)\abs{q(b)}},
\end{align}
which completes the proof.
\end{proof}

\subsection{Proof of Remark~\ref{lem:math-variance-0}}\label{append:math-variance-0}
\begin{proof}
When $\forall a \in \fA, q(a) \geq 0$, if $\exists a_0, \pi(a_0)  q(a_0) \neq 0$,  we have $\forall a \in \fA$
\begin{align}
\mu^*(a) = \frac{\pi(a) \abs{q(a)}}{c}
\end{align}
and $c>0$ is a normalizing constant. Plugging $\mu^*$ to $\rho(A)q(A)$, we get  $\forall a$ with $\mu^*(a) > 0$
\begin{align}
\rho(a)q(a) = \frac{\pi(a)}{\mu^*(a)}q(a) =  \frac{\pi(a)}{\frac{\pi(a) \abs{q(a)}}{c} }q(a)  = c.
\end{align}
This means in this setting, with the optimal distribution $\mu^*$, the random variable $\rho(\cdot)q(\cdot)$ is a constant function.
Thus,
\begin{align}
\V_{A\sim \mu^*}(\rho(A)q(A))  = 0.
\end{align}

When $\forall a \in \fA, q(a) \geq 0$, if $\forall a_0, \pi(a_0)  q(a_0) = 0$,  we have $\forall a \in \fA$
\begin{align}
\mu^*(a) = \frac{1}{|\fA|}.
\end{align}
 Plugging $\mu^*$ to $\rho(A)q(A)$,  we get  $\forall a \in \fA$
\begin{align}
\rho(a)q(a) = \frac{\pi(a)}{\mu^*(a)}q(a) =  \frac{\pi(a)q(a)}{ \frac{1}{|\fA|}}  = 0.
\end{align}
This shows $\rho(A)q(A)$ is also a constant.
Thus,
\begin{align}
\V_{A\sim \mu^*}(\rho(A)q(A))  = 0.
\end{align}
The proof is similar for $\forall a \in \fA, q(a) \leq 0$ and is thus omitted.

\end{proof}

\subsection{Proof of Lemma~\ref{lemma: Lambda broadness u}}\label{appendix: Lambda broadness u}

\begin{proof}
$\forall \mu \in \Lambda^-$, $\forall t, s, a$,
\begin{align}
&\mu_t(a|s) = 0 \\
\implies& \pi_t(a|s) = 0 \explain{Definition of $\Lambda^-$} \\
\implies& \pi_t(a|s)u_{\pi, t}(s, a) = 0.
\end{align}

This shows $\mu \in \Lambda^*$. Thus, $\Lambda^- \subseteq \Lambda^*$.
\end{proof}

\subsection{Proof of Lemma~\ref{lemma: Lambda star unbiasedness}}
\label{appendix: Lambda star unbiasedness}

In this section, we first prove a more general lemma. Define
\begin{align}\label{def: Lambda}
\Lambda \doteq& \qty{\mu \mid \forall t, s, a, \mu_t(a|s) = 0 \implies \pi_t(a|s)q_{\pi, t}(s, a)  = 0}.
\end{align}

\begin{lemma}
\label{lemma: Lambda unbiasedness}
$\forall \mu \in \Lambda$, $\forall  t, \forall  s$,
 $\E\left[\pdisg(\tau^{\mu_{t:T-1}}_{t:T-1}) \mid S_t = s\right] = v_{\pi, t}(s).$
\end{lemma}

\begin{proof}
We proceed via induction. $\forall \mu \in \Lambda$, if $t = T-1$,
we have
\begin{align}
\E_{A_t \sim \mu_t} \left[\pdisg(\tau^{\mu_{t:T-1}}_{t:T-1}) \mid S_t\right] =& \E\left[\rho_t R_{t+1} \mid S_t\right] = \E\left[\rho_t q_{\pi, t}(S_t, A_t) \mid S_t \right] \\
=& \E_{A_t \sim \pi_t(\cdot \mid S_t)}\left[q_{\pi, t}(S_t, A_t) \mid S_t\right] \explain{Lemma~\ref{lem stats unbiasedness}} \\
=& v_{\pi, t}(S_t).
\end{align}
For $t \in [T-2]$,
we have
\begin{align}
&\E\left[\pdisg(\tau^{\mu_{t:T-1}}_{t:T-1}) \mid S_t\right] \\
=& \E\left[\rho_t R_{t+1} + \rho_t\pdisg(\tau^{\mu_{t+1:T-1}}_{t+1:T-1}) \mid S_t\right] \\
=& \E\left[\rho_t R_{t+1} \mid S_t\right] + \E\left[\rho_t\pdisg(\tau^{\mu_{t+1:T-1}}_{t+1:T-1}) \mid S_t\right] \\
\explain{Law of total expectation}
=& \E\left[\rho_t R_{t+1} \mid S_t\right] + \E_{A_t \sim \mu_t(\cdot \mid S_t), S_{t+1} \sim p(\cdot \mid S_t, A_t)}\left[ \E\left[\rho_t\pdisg(\tau^{\mu_{t+1:T-1}}_{t+1:T-1}) \mid S_t, A_t, S_{t+1}\right] \mid S_t \right] \\
\explain{Conditional independence and Markov property}
=& \E\left[\rho_t R_{t+1} \mid S_t\right] + \E_{A_t \sim \mu_t(\cdot \mid S_t), S_{t+1} \sim p(\cdot \mid S_t, A_t)}\left[ \rho_t \E\left[\pdisg(\tau^{\mu_{t+1:T-1}}_{t+1:T-1}) \mid S_{t+1}\right] \mid S_t \right] \\
=& \E\left[\rho_t R_{t+1} \mid S_t\right] + \E_{A_t \sim \mu_t(\cdot \mid S_t), S_{t+1} \sim p(\cdot \mid S_t, A_t)}\left[ \rho_t v_{\pi, t+1}(S_{t+1}) \mid S_t \right]
\explain{Inductive hypothesis} \\
=& \E_{A_t \sim \mu_t(\cdot \mid S_t)}\left[\rho_t q_{\pi, t}(S_t, A_t) \mid S_t\right] \explain{Definition of $q_{\pi, t}$} \\
=& \E_{A_t \sim \pi_t(\cdot \mid S_t)}\left[q_{\pi, t}(S_t, A_t) \mid S_t\right] \explain{Lemma~\ref{lem stats unbiasedness}} \\
=& v_{\pi, t}(S_t),
\end{align}
  which completes the proof.

\end{proof}

We restate and present the  proof of Lemma \ref{lemma: Lambda star unbiasedness}
\reOOLambdaOOstarOOunbiasedness*
\begin{proof}
$\forall \mu \in \Lambda^*$, $\forall t, s, a$, suppose $\mu_t(a|s) = 0$. Then
\begin{align}
&\pi_t(a|s)\, u_{\pi, t}(s, a) = 0 \explain{Definition of $\Lambda^*$ \eqref{def: Lambda star}} \\
\implies& \pi_t(a|s) \qty(q_{\pi, t}^2(s, a) + \omega_{\pi, t}(s, a) + \textstyle{\sum_{s'} p(s'|s, a)\V\left(\pdisg(\tau^{\mu^*_{t+1:T-1}}_{t+1:T-1}) \mid S_{t+1} = s'\right)}) = 0 \explain{Definition of $u_{\pi, t}$ \eqref{def u}} \\
\implies& \pi_t(a|s)\, q_{\pi, t}^2(s, a) = 0 \explain{Every summand is non-negative} \\
\implies& \pi_t(a|s)\, q_{\pi, t}(s, a) = 0.
\end{align}

By the definition of $\Lambda$ in \eqref{def: Lambda}, this shows $\mu \in \Lambda$. Thus, $\Lambda^* \subseteq \Lambda$.
By Lemma \ref{lemma: Lambda unbiasedness},
$\forall \mu \in \Lambda^*$, $\forall  t, \forall  s$,  $\E\left[\pdisg(\tau^{\mu_{t:T-1}}_{t:T-1}) \mid S_t = s\right] = v_{\pi, t}(s).$

\end{proof}

\subsection{Proof of Theorem~\ref{lem:rl-optimal}}\label{append:rl-optima}

To prove Theorem~\ref{lem:rl-optimal},
we rely on a recursive expression of the PDIS Monte Carlo estimator summarized by the following lemma.

\begin{lemma}[Recursive Expression of Variance] \label{lem:recursive-var}
For any $\mu \in \Lambda$,
for $t = T-1$,
\begin{align}
\V\left(\pdisg(\tau^{\mu_{t:T-1}}_{t:T-1})\mid S_t\right) = \E_{A_t \sim \mu_t}\left[\rho_t^2 q_{\pi, t}^2(S_t, A_t) \mid S_t\right] -  v_{\pi, t}^2(S_t),
\end{align}
for $t \in [T-2]$,
\begin{align}
&\V\left(\pdisg(\tau^{\mu_{t:T-1}}_{t:T-1})\mid S_t\right) \\
=& \E_{A_t\sim \mu_t}\left[\rho_t^2 \left(\E_{S_{t+1}}\left[\V\left(\pdisg(\tau^{\mu_{t+1:T-1}}_{t+1:T-1})\mid S_{t+1}\right) \mid S_t, A_t\right] + \omega_{\pi,t}(S_t, A_t) + q_{\pi, t}^2(S_t, A_t)\right) \mid S_t\right] \\
&- v_{\pi, t}^2(S_t).
\end{align}
\end{lemma}
\begin{proof}
When $t\in [T-2]$, we have
\begin{align}
\label{eq tmp4}
&\V\left(\pdisg(\tau^{\mu_{t:T-1}}_{t:T-1})\mid S_t\right) \\
=& \E_{A_t}\left[ \V\left(\pdisg(\tau^{\mu_{t:T-1}}_{t:T-1})\mid S_t, A_t\right)\mid S_t\right] + \V_{A_t}\left(\E\left[\pdisg(\tau^{\mu_{t:T-1}}_{t:T-1}) \mid S_t, A_t\right]\mid S_t\right)
\explain{Law of total variance} \\
=& \E_{A_t}\left[ \rho_t^2 \V\left(r(S_t,A_t) + \pdisg(\tau^{\mu_{t+1:T-1}}_{t+1:T-1}) \mid S_t, A_t\right)\mid S_t\right] \\
&+ \V_{A_t}\left(\rho_t \E\left[r(S_t,A_t) + \pdisg(\tau^{\mu_{t+1:T-1}}_{t+1:T-1}) \mid S_t, A_t\right]\mid S_t\right)
\explain{Using \eqref{eq:PDIS-recursive}}
\\
=& \E_{A_t}\left[ \rho_t^2 \V\left(\pdisg(\tau^{\mu_{t+1:T-1}}_{t+1:T-1}) \mid S_t, A_t\right)\mid S_t\right] + \V_{A_t}\left(\rho_t \E\left[r(S_t,A_t) + \pdisg(\tau^{\mu_{t+1:T-1}}_{t+1:T-1}) \mid S_t, A_t\right]\mid S_t\right) \explain{Deterministic reward $r$} \\
=& \E_{A_t}\left[ \rho_t^2 \V\left(\pdisg(\tau^{\mu_{t+1:T-1}}_{t+1:T-1}) \mid S_t, A_t\right)\mid S_t\right] + \V_{A_t}\left(\rho_t q_{\pi, t}(S_t, A_t)\mid S_t\right).
\end{align}

Further decomposing the first term, we have
\begin{align}
\label{eq tmp3}
&\V\left(\pdisg(\tau^{\mu_{t+1:T-1}}_{t+1:T-1}) \mid S_t, A_t\right) \\
=& \E_{S_{t+1}}\left[\V\left(\pdisg(\tau^{\mu_{t+1:T-1}}_{t+1:T-1}) \mid S_t, A_t, S_{t+1}\right) \mid S_t, A_t\right] \\
&+ \V_{S_{t+1}}\left(\E\left[\pdisg(\tau^{\mu_{t+1:T-1}}_{t+1:T-1}) \mid S_t, A_t, S_{t+1}\right]\mid S_t, A_t\right)
\explain{Law of total variance}
\\
=& \E_{S_{t+1}}\left[\V\left(\pdisg(\tau^{\mu_{t+1:T-1}}_{t+1:T-1}) \mid S_{t+1}\right) \mid S_t, A_t\right] + \V_{S_{t+1}}\left(\E\left[\pdisg(\tau^{\mu_{t+1:T-1}}_{t+1:T-1}) \mid S_{t+1}\right]\mid S_t, A_t\right) \explain{Markov property} \\
=& \E_{S_{t+1}}\left[\V\left(\pdisg(\tau^{\mu_{t+1:T-1}}_{t+1:T-1}) \mid S_{t+1}\right) \mid S_t, A_t\right] + \V_{S_{t+1}}\left(v_{\pi, t+1}(S_{t+1})\mid S_t, A_t\right). \explain{Lemma~\ref{lemma: Lambda unbiasedness}}
\end{align}
With $\omega_{\pi, t}$ defined in~\eqref{def:nu},
plugging~\eqref{eq tmp3} back to~\eqref{eq tmp4} yields
\begin{align}
&\V\left(\pdisg(\tau^{\mu_{t:T-1}}_{t:T-1})\mid S_t\right) \\
=&\E_{A_t}\left[\rho_t^2 \left(\E_{S_{t+1}}\left[\V\left(\pdisg(\tau^{\mu_{t+1:T-1}}_{t+1:T-1}) \mid S_{t+1}\right) \mid S_t, A_t\right] + \omega_{\pi, t}(S_t, A_t)\right) \mid S_t\right] \\
&+ \V_{A_t}\left(\rho_t q_{\pi, t}(S_t, A_t)\mid S_t\right) \\
=&\E_{A_t}\left[\rho_t^2 \left(\E_{S_{t+1}}\left[\V\left(\pdisg(\tau^{\mu_{t+1:T-1}}_{t+1:T-1}) \mid S_{t+1}\right) \mid S_t, A_t\right] + \omega_{\pi, t}(S_t, A_t)\right) \mid S_t\right] \\
&+ \E_{A_t}\left[\rho_t^2 q_{\pi, t}^2(S_t, A_t)\mid S_t\right] - \left(\E_{A_t}\left[\rho_t q_{\pi, t}(S_t, A_t) \mid S_t\right]\right)^2 \\
=&\E_{A_t}\left[\rho_t^2 \left(\E_{S_{t+1}}\left[\V\left(\pdisg(\tau^{\mu_{t+1:T-1}}_{t+1:T-1}) \mid S_{t+1}\right) \mid S_t, A_t\right] + \omega_{\pi, t}(S_t, A_t)\right) \mid S_t\right] \\
&+ \E_{A_t}\left[\rho_t^2 q_{\pi, t}^2(S_t, A_t)\mid S_t\right] -  v_{\pi, t}^2(S_t). \explain{Lemma~\ref{lem stats unbiasedness}}
\end{align}
When $t =  T-1$, we have
\begin{align}
\V\left(\pdisg(\tau^{\mu_{t:T-1}}_{t:T-1})\mid S_t\right) =& \V\left(\rho_t r(S_t, A_t)\mid S_t\right) \\
=& \V\left(\rho_t q_{\pi, t}(S_t, A_t)\mid S_t\right) \\
=&  \E_{A_t}\left[\rho_t^2 q_{\pi, t}^2(S_t, A_t) \mid S_t\right] -  v_{\pi, t}^2(S_t),
\end{align}
which completes the proof.
\end{proof}

We restate and present the main proof of Theorem~\ref{lem:rl-optimal}.
\restaterloptimal*
\begin{proof}
  We proceed via induction.
  When $t = T-1$,
  we have
  \begin{align}
    &\V\left(\pdisg(\tau^{\mu_{T-1:T-1}}_{T-1:T-1}) \mid S_{T-1} = s\right) \\
    =& \V_{A_{T-1}}\left(\rho_{T-1} r(s, A_{T-1})\mid S_{T-1} =s \right) \\
    =& \V_{A_{T-1}}\left(\rho_{T-1} q_{\pi, T-1}(s, A_{T-1})\mid S_{T-1} =s \right).
  \end{align}
  The definition of $\mu^*_{T-1}$ in~\eqref{def mu star} and Lemma~\ref{lem:math-optimal} ensure that $\mu^*_{T-1}$ is an optimal solution to
  \begin{align}
    \min_{\mu_{T-1} \in \Lambda^*_{T-1}}  \quad \V\left(\pdisg\left(\tau^{\mu_{T-1}}_{T-1}\right) \mid S_{T-1} = s\right).
  \end{align}
  Now,
  suppose for some $t \in [T-2]$ and every successor state $s'$,
  $\mu^*_{t+1:T-1}$ is an optimal solution to
  \begin{align}
    \min_{\mu_{t+1} \in \Lambda^*_{t+1}, \dots, \mu_{T-1} \in \Lambda^*_{T-1}}  \quad \V\left(\pdisg\left(\tau^{\mu_{t+1:T-1}}_{t+1:T-1}\right)\mid S_{t+1}=s'\right).
  \end{align}
  To complete induction,
  we proceed to proving that $\mu^*_{t:T-1}$ is an optimal solution to
  \begin{align}
    \label{eq induction opt problm ap}
    \min_{\mu_t \in \Lambda^*_t, \dots, \mu_{T-1} \in \Lambda^*_{T-1}}  \quad \V\left(\pdisg\left(\tau^{\mu_{t:T-1}}_{t:T-1}\right) \mid S_t = s\right).
  \end{align}
  In the rest of this proof,
  we omit the domain $\Lambda^*_t, \dots, \Lambda^*_{T-1}$ for simplifying notations.
  For any $\mu_{t:T-1}$, we have
\begin{align}
&\V\left(\pdisg(\tau^{\mu_{t:T-1}}_{t:T-1})\mid S_{t}\right) \\
=& \E_{A_{t}}\left[\rho_{t}^2 \left(\E_{S_{t+1}}\left[\V\left(\pdisg(\tau^{\mu_{{t+1}:T-1}}_{{t+1}:T-1}) \mid S_{t+1}\right) \mid S_{t}, A_{t}\right] + \omega_{t}(S_{t}, A_{t}) + q_{\pi, t}^2(S_{t}, A_{t}) \right) \mid S_{t}\right] \\
& -  v_{\pi, t}^2(S_{t})  \explain{By Lemma~\ref{lem:recursive-var}}  \\
\stackrel{(a)}{\geq} & \E_{A_{t}}\left[\rho_{t}^2 \left(\E_{S_{t+1}}\left[\min_{\mu'_{t+1:T-1}}\V\left(\pdisg(\tau^{\mu'_{{t+1}:T-1}}_{{t+1}:T-1}) \mid S_{t+1}\right) \mid S_{t}, A_{t}\right] + \omega_{t}(S_{t}, A_{t}) + q_{\pi, t}^2(S_{t}, A_{t}) \right) \mid S_{t}\right] \\
& -  v_{\pi, t}^2(S_{t})  \explain{Monotonicity in the continuation variance}  \\
=& \E_{A_{t}}\left[\rho_{t}^2 \left(\E_{S_{t+1}}\left[\V\left(\pdisg(\tau^{\mu^*_{{t+1}:T-1}}_{{t+1}:T-1}) \mid S_{t+1}\right) \mid S_{t}, A_{t}\right] + \omega_{t}(S_{t}, A_{t}) + q_{\pi, t}^2(S_{t}, A_{t}) \right) \mid S_{t}\right] \\
& -  v_{\pi, t}^2(S_{t}) \explain{Inductive hypothesis} \\
=& \E_{A_{t}}\left[\rho_{t}^2 u_{\pi,t}(S_t,A_t)\mid S_{t}\right]  -  v_{\pi, t}^2(S_{t}) \explain{By \eqref{def u}} \\
=& \V_{A_{t}}\left(\rho_{t} \sqrt{u_{\pi,t}(S_t,A_t)}\mid S_{t}\right) + \E_{A_t}\left[\rho_t \sqrt{u_{\pi, t}(S_t, A_t)} \mid S_t\right]^2 -  v_{\pi, t}^2(S_{t}) \explain{Definition of variance} \\
=& \V_{A_t}\left(\rho_{t} \sqrt{u_{\pi,t}(S_t,A_t)}\mid S_{t}\right) + \E_{A_t\sim\pi_t(\cdot \mid S_t)}\left[\sqrt{u_{\pi, t}(S_t, A_t)} \mid S_t\right]^2 -  v_{\pi, t}^2(S_{t}) \explain{Lemma~\ref{lem stats unbiasedness} and $\mu_t \in \Lambda^*_t$} \\
\stackrel{(b)}{\geq}& \E_{A_t\sim\pi_t(\cdot \mid S_t)}\left[\sqrt{u_{\pi, t}(S_t, A_t)} \mid S_t\right]^2 -  v_{\pi, t}^2(S_{t}) \explain{Non-negativity of variance}.
\end{align}
According to the inductive hypothesis,
the equality in $(a)$ can be achieved when $\mu_{t+1:T-1} = \mu^*_{t+1:T-1}$.
According to the construction of $\mu^*_t$ in~\eqref{def mu star} and Remark~\ref{lem:math-variance-0},
the equality in $(b)$ can be achieved when $\mu_t = \mu^*_t$.
This suggests that $\mu^*_{t:T-1}$ achieves the lower bound and is thus an optimal solution to~\eqref{eq induction opt problm ap},
which completes the induction and thus completes the proof.
\end{proof}

\subsection{Proof of Theorem \ref{lem:var_smaller_stronger_optimal}}\label{append:var_smaller_stronger_optimal}

\begin{proof}

The variance of $\pdisg(\tau^{\mu^*_{t:T-1}}_{t:T-1})$ has
$\forall s$, for $t=T-1$,
\begin{align}
&\V\left(\pdisg(\tau^{\mu^*_{t:T-1}}_{t:T-1})\mid S_t = s\right) \\   =&\E_{A_t \sim \mu^*_t}\left[\rho_t^2 q^2_{\pi,t}(S_t,A_t)\mid S_t\right] -  v^2_{\pi,t}(S_t) \explain{Lemma~\ref{lem:recursive-var}}\\
=&  \E_{A_t \sim \mu^*_t}\left[\rho_t^2 u_t(S_t,A_t) \mid S_t\right]  -  v^2_{\pi,t}(S_t)  \explain{By \eqref{def u}}
\end{align}

For the first term, we have $\forall t$,
\begin{align}
&\E_{A_t \sim \mu^*_t}\left[\rho_t^2 u_{\pi,t}(S_t,A_t) \mid S_t\right] \\
=&\sum_{a \, : \, \pi_t(a|S_t) u_{\pi,t}(S_t,a) > 0} \frac{\pi_t(a|S_t)^2}{\mu^*_t(a|S_t)} u_{\pi,t}(S_t,a) \\
=&\sum_{a \, : \, \pi_t(a|S_t) u_{\pi,t}(S_t,a) > 0} \pi_t(a|S_t)\sqrt{u_{\pi,t}(S_t,a) }\sum_b \pi_t(b|S_t)\sqrt{u_{\pi,t}(S_t,b) } \explain{By \eqref{def mu star}}\\
=&\E_{A_t \sim \pi_t}\left[ \sqrt{u_{\pi,t}(S_t,A_t)} \mid S_t\right]^2.\label{eq: mu star sqrt u}
\end{align}

Thus,
$t=T-1$,
\begin{align}
\label{eq: variance recursive doubly T-1}
&\V\left(\pdisg(\tau^{\mu^*_{t:T-1}}_{t:T-1})\mid S_t = s\right) \\
=& \E_{A_t \sim \pi_t}\left[ \sqrt{u_{\pi,t}(S_t,A_t)} \mid S_t\right]^2  -  v^2_{\pi,t}(S_t)
\end{align}

For $t\in[T-2]$,
\begin{align}
\label{eq: variance recursive doubly T-2}
&\V\left(\pdisg(\tau^{\mu^*_{t:T-1}}_{t:T-1})\mid S_t = s\right) \\
=&  \E_{A_t \sim \mu^*_t}\left[\rho_t^2 \left(\E_{S_{t+1}}\left[\V\left(\pdisg(\tau^{\mu^*_{t+1:T-1}}_{t+1:T-1}) \mid S_{t+1}\right) \mid S_t, A_t\right] + \omega_{\pi, t}(S_t, A_t)+q^2_{\pi,t}(S_t,A_t)\right) \mid S_t\right] \explain{Lemma \ref{lem:recursive-var}}\\
&- v^2_{\pi, t}(S_t) \\
=&  \E_{A_t \sim \mu^*_t}\left[\rho_t^2 u_t(S_t,A_t) \mid S_t\right]  - v^2_{\pi, t}(S_t) \explain{By \eqref{def u}}
\\ =&\E_{A_t \sim \pi_t}\left[\sqrt{u_t(S_t,A_t)} \mid S_t\right]^2 - v^2_{\pi, t}(S_t).\explain{By \eqref{eq: mu star sqrt u}}
\end{align}

The variance of $\pdisg(\tau^{\pi_{t:T-1}}_{t:T-1})$ has
$\forall s$, for $t=T-1$,
\begin{align}
\label{eq: variance recursive onpolicy T-1}
&\V\left(\pdisg(\tau^{\pi_{t:T-1}}_{t:T-1})\mid S_t = s\right) \\
=&\E_{A_t \sim \pi_t}\left[q_{\pi,t}(S_t,A_t)^2\mid S_t\right] -  v_{\pi,t}(S_t)^2\explain{Lemma \ref{lem:recursive-var}}\\
=&  \E_{A_t \sim \pi_t}\left[u_t(S_t,A_t) \mid S_t\right]-  v_{\pi,t}(S_t)^2.\explain{By \eqref{def u}}
\end{align}

For $t\in[T-2]$,
\begin{align}
\label{eq: variance recursive onpolicy T-2}
&\V\left(\pdisg(\tau^{\pi_{t:T-1}}_{t:T-1})\mid S_t = s\right) \\
=& \E_{A_t \sim \pi_t}\left[ \E_{S_{t+1}}\left[\V\left(\pdisg(\tau^{\pi_{t+1:T-1}}_{t+1:T-1}) \mid S_{t+1}\right) \mid S_t, A_t\right] + \omega_{\pi, t}(S_t, A_t)+q_{\pi,t}(S_t,A_t)^2 \mid S_t\right] \\
&- v_{\pi, t}(S_t)^2 \explain{Lemma \ref{lem:recursive-var}}\\
=& \E_{A_t \sim \pi_t}\left[ \E_{S_{t+1}}\left[\V\left(\pdisg(\tau^{\mu^*_{t+1:T-1}}_{t+1:T-1}) \mid S_{t+1}\right) \mid S_t, A_t\right] + \omega_{\pi, t}(S_t, A_t) +  q_{\pi,t}(S_t,A_t)^2 \mid S_t\right] \\
&- v_{\pi, t}(S_t)^2 \\
&+ \E_{A_t \sim \pi_t}\left[ \E_{S_{t+1}}\left[\V\left(\pdisg(\tau^{\pi_{t+1:T-1}}_{t+1:T-1}) \mid S_{t+1}\right) - \V\left(\pdisg(\tau^{\mu^*_{t+1:T-1}}_{t+1:T-1}) \mid S_{t+1}\right) \mid S_t, A_t\right]  \mid S_t\right] \\
=&  \E_{A_t \sim \pi_t}\left[u_t(S_t,A_t) \mid S_t\right] \\
&- v_{\pi, t}(S_t)^2 \\
&+ \E_{A_t \sim \pi_t}\left[ \E_{S_{t+1}}\left[\V\left(\pdisg(\tau^{\pi_{t+1:T-1}}_{t+1:T-1}) \mid S_{t+1}\right) - \V\left(\pdisg(\tau^{\mu^*_{t+1:T-1}}_{t+1:T-1}) \mid S_{t+1}\right) \mid S_t, A_t\right]  \mid S_t\right] .\explain{By \eqref{def u}}
\end{align}
Thus, for $t=T-1$, their difference is
\begin{align}
&\V\left(\pdisg(\tau^{\pi_{t:T-1}}_{t:T-1})\mid S_t = s\right) - \V\left(\pdisg(\tau^{\mu^*_{t:T-1}}_{t:T-1})\mid S_t = s\right) \\
=& \E_{A_t \sim \pi_t}\left[u_t(S_t,A_t) \mid S_t\right] - \E_{A_t \sim \pi_t}\left[\sqrt{u_t(S_t,A_t)} \mid S_t\right]^2
\explain{By \eqref{eq: variance recursive doubly T-1} and \eqref{eq: variance recursive onpolicy T-1}}\\
=& \V_{A_t \sim \pi_t}\qty(\sqrt{u_t(S_t,A_t)} \mid S_t) .
\end{align}
For $t\in[T-2]$,
\begin{align}
&\V\left(\pdisg(\tau^{\pi_{t:T-1}}_{t:T-1})\mid S_t = s\right) - \V\left(\pdisg(\tau^{\mu^*_{t:T-1}}_{t:T-1})\mid S_t = s\right)\\
=& \E_{A_t \sim \pi_t}\left[u_t(S_t,A_t) \mid S_t\right] - \E_{A_t \sim \pi_t}\left[\sqrt{u_t(S_t,A_t)} \mid S_t\right]^2\\
&+ \E_{A_t \sim \pi_t}\left[ \E_{S_{t+1}}\left[\V\left(\pdisg(\tau^{\pi_{t+1:T-1}}_{t+1:T-1}) \mid S_{t+1}\right) - \V\left(\pdisg(\tau^{\mu^*_{t+1:T-1}}_{t+1:T-1}) \mid S_{t+1}\right) \mid S_t, A_t\right]  \mid S_t\right]  \explain{By \eqref{eq: variance recursive doubly T-2} and \eqref{eq: variance recursive onpolicy T-2}}  \\
=& \V_{A_t \sim \pi_t}\qty(\sqrt{u_t(S_t,A_t)} \mid S_t) \\
&+ \E_{A_t \sim \pi_t}\left[ \E_{S_{t+1}}\left[\V\left(\pdisg(\tau^{\pi_{t+1:T-1}}_{t+1:T-1}) \mid S_{t+1}\right) - \V\left(\pdisg(\tau^{\mu^*_{t+1:T-1}}_{t+1:T-1}) \mid S_{t+1}\right) \mid S_t, A_t\right]  \mid S_t\right]. \\
\end{align}

We use induction to prove $\forall t, s, \delta^{\pi, \mu^*}_t(s) \geq 0$.
For $t=T-1$,
\begin{align}
\delta^{\pi, \mu^*}_t(s) = 0 \geq 0.
\end{align}

For $t \in [T-2]$, the induction hypothesis is $\forall s$,
\begin{align}
\delta^{\pi, \mu^*}_{t+1}(s) \geq 0.
\end{align}
This implies $\forall s$,
\begin{align}
&\V\left(\pdisg(\tau^{\pi_{t+1:T-1}}_{t+1:T-1})\mid S_{t+1} = s\right) - \V\left(\pdisg(\tau^{\mu^*_{t+1:T-1}}_{t+1:T-1})\mid S_{t+1} = s\right) \\
=& \V_{A_{t+1} \sim \pi_{t+1}}\qty(\sqrt{u_{t+1}(S_{t+1},A_{t+1})} \mid S_{t+1}=s)  + \delta^{\pi, \mu^*}_{t+1}(s) \\
\geq& 0 \label{eq: v pdis - v geq 0}.
\end{align}

Thus, $\forall s$,
\begin{align}
&\delta^{\pi, \mu^*}_t(s) \\
=&\E_{A_t \sim \pi_t, S_{t+1}}\left[ \V\left(\pdisg(\tau^{\pi_{t+1:T-1}}_{t+1:T-1}) \mid S_{t+1}\right) - \V\left(\pdisg(\tau^{\mu^*_{t+1:T-1}}_{t+1:T-1}) \mid S_{t+1}\right)  \mid S_t = s\right]\\
\geq&0.\explain{by \eqref{eq: v pdis - v geq 0}}
\end{align}

Thus, $\forall t, s, \delta^{\pi, \mu^*}_t(s) \geq 0$.

\end{proof}

\subsection{Proof of Lemma~\ref{lemma: Lambda broadness hat q}}\label{appendix: Lambda broadness hat q}

\begin{proof}
$\forall \mu \in \Lambda^-$, $\forall t, s, a$,
\begin{align}
&\mu_t(a|s) = 0 \\
\implies& \pi_t(a|s) = 0 \explain{Definition of $\Lambda^-$} \\
\implies& \pi_t(a|s)\hat{q}_{\pi, t}(s, a) = 0.
\end{align}

This shows $\mu \in \hat{\Lambda}$. Thus, $\Lambda^- \subseteq \hat{\Lambda}$.
\end{proof}

\subsection{Proof of Lemma~\ref{lemma: Lambda unbiasedness hat}}
\label{appendix: Lambda hat unbiasedness}

\begin{proof}
$\forall \mu \in \hat{\Lambda}$, $\forall t, s, a$, suppose $\mu_t(a|s) = 0$. Then
\begin{align}
&\pi_t(a|s)\, \hat q_{\pi, t}(s, a) = 0 \explain{Definition of $\hat\Lambda$ \eqref{def: Lambda hat}} \\
\implies& \pi_t(a|s) \qty(q_{\pi, t}^2(s, a) + \omega_{\pi, t}(s, a) + \textstyle{\sum_{s'} p(s'|s, a)\V\left(\pdisg(\tau^{\pi_{t+1:T-1}}_{t+1:T-1}) \mid S_{t+1} = s'\right)}) = 0 \explain{Definition of $\hat{q}_{\pi, t}$ \eqref{eq def q hat}} \\
\implies& \pi_t(a|s)\, q_{\pi, t}^2(s, a) = 0 \explain{Every summand is non-negative} \\
\implies& \pi_t(a|s)\, q_{\pi, t}(s, a) = 0.
\end{align}

By the definition of $\Lambda$ in \eqref{def: Lambda}, this shows $\mu \in \Lambda$. Thus, $\hat{\Lambda} \subseteq \Lambda$.
By Lemma \ref{lemma: Lambda unbiasedness},
$\forall \mu \in \hat{\Lambda}$, $\forall  t, \forall  s$,  $\E\left[\pdisg(\tau^{\mu_{t:T-1}}_{t:T-1}) \mid S_t = s\right] = v_{\pi, t}(s).$

\end{proof}

\subsection{Proof of Theorem \ref{lem:one_step_var_smaller_stronger_optimal}}\label{append:one_step_var_smaller_stronger_optimal}

\begin{proof}

The variance of $\pdisg(\tau^{\hat{\mu}_{t:T-1}}_{t:T-1})$ has
$\forall s$, for $t=T-1$,
\begin{align}
&\V\left(\pdisg(\tau^{\hat{\mu}_{t:T-1}}_{t:T-1})\mid S_t = s\right) \\   =&\E_{A_t \sim \hat{\mu}_t}\left[\rho_t^2 q^2_{\pi,t}(S_t,A_t)\mid S_t\right] -  v^2_{\pi,t}(S_t) \explain{Lemma~\ref{lem:recursive-var}}\\
=&  \E_{A_t \sim \hat{\mu}_t}\left[\rho_t^2 \hat{q}_t(S_t,A_t) \mid S_t\right]  -  v^2_{\pi,t}(S_t)  \explain{By \eqref{eq def q hat}}
\end{align}

For the first term, we have $\forall t$,
\begin{align}
&\E_{A_t \sim \hat{\mu}_t}\left[\rho_t^2 \hat{q}_{\pi,t}(S_t,A_t) \mid S_t\right] \\
=&\sum_{a \, : \, \pi_t(a|S_t)\hat{q}_{\pi,t}(S_t,a) > 0} \frac{\pi_t(a|S_t)^2}{\hat{\mu}_t(a|S_t)} \hat{q}_{\pi,t}(S_t,a) \\
=&\sum_{a \, : \, \pi_t(a|S_t)\hat{q}_{\pi,t}(S_t,a) > 0} \pi_t(a|S_t)\sqrt{\hat{q}_{\pi,t}(S_t,a) }\sum_b \pi_t(b|S_t)\sqrt{\hat{q}_{\pi,t}(S_t,b) } \explain{By \eqref{def hat mu}}\\
=&\E_{A_t \sim \pi_t}\left[ \sqrt{\hat{q}_{\pi,t}(S_t,A_t)} \mid S_t\right]^2.\label{eq: hat q square}
\end{align}

When the normalizing constant in \eqref{def hat mu} is positive,
the restricted sums run over the support of $\hat\mu_t$,
and the excluded actions satisfy
$\pi_t(a|s)\sqrt{\hat q_{\pi,t}(s,a)}=0$.
When the normalizing constant is zero, the uniform convention
applies and $\pi_t(a|s)\hat q_{\pi,t}(s,a)=0$ for every $a$,
so both sides of \eqref{eq: hat q square} vanish.
Thus the identity holds in both cases.
The same argument applies to the corresponding step in
Theorem~\ref{lem:rl-optimal}, with $u_{\pi,t}$ replacing
$\hat q_{\pi,t}$.

Thus,
$t=T-1$,
\begin{align}
\label{eq: variance comparison hat q T-1}
&\V\left(\pdisg(\tau^{\hat{\mu}_{t:T-1}}_{t:T-1})\mid S_t = s\right) \\
=& \E_{A_t \sim \pi_t}\left[ \sqrt{\hat{q}_{\pi,t}(S_t,A_t)} \mid S_t\right]^2  -  v^2_{\pi,t}(S_t)
\end{align}

For $t=T-1$, their difference is
\begin{align}
&\V\left(\pdisg(\tau^{\pi_{t:T-1}}_{t:T-1})\mid S_t = s\right) - \V\left(\pdisg(\tau^{\hat{\mu}_{t:T-1}}_{t:T-1})\mid S_t = s\right) \\
= & \E_{A_t \sim \pi_t}\left[\hat{q}_t(S_t,A_t) \mid S_t\right] - \E_{A_t \sim \pi_t}\left[\sqrt{\hat{q}_t(S_t,A_t)} \mid S_t\right]^2
\explain{By  \eqref{eq: variance recursive onpolicy T-1} and \eqref{eq: variance comparison hat q T-1}}\\
= & \V_{A_t \sim \pi_t}\qty(\sqrt{\hat{q}_t(S_t,A_t)} \mid S_t) .
\end{align}

We also have, for $t=T-1$,
\begin{align}
\delta^{\pi, \hat{\mu}}_t(s) = 0 \geq 0.
\end{align}

For $t\in[T-2]$, we prove by induction. Assume Theorem \ref{lem:one_step_var_smaller_stronger_optimal} hold for $t+1$,
\begin{align}
\label{eq: variance comparison hat q T-2}
&\V\left(\pdisg(\tau^{\hat{\mu}_{t:T-1}}_{t:T-1})\mid S_t = s\right) \\
=&  \E_{A_t \sim \hat{\mu}_t}\left[\rho_t^2 \left(\E_{S_{t+1}}\left[\V\left(\pdisg(\tau^{\hat{\mu}_{t+1:T-1}}_{t+1:T-1}) \mid S_{t+1}\right) \mid S_t, A_t\right] + \omega_{\pi, t}(S_t, A_t)+q^2_{\pi,t}(S_t,A_t)\right) \mid S_t\right] \explain{Lemma \ref{lem:recursive-var}}\\
&- v^2_{\pi, t}(S_t) \\
= &  \E_{A_t \sim \hat{\mu}_t}\left[\rho_t^2 \left(\E_{S_{t+1}}\left[\V\left(\pdisg(\tau^{\pi_{t+1:T-1}}_{t+1:T-1}) \mid S_{t+1}\right) \mid S_t, A_t\right] + \omega_{\pi, t}(S_t, A_t)+q^2_{\pi,t}(S_t,A_t)\right) \mid S_t\right] \\
&- v^2_{\pi, t}(S_t) - \delta^{\pi, \hat{\mu}}_t(s) \explain{Definition of $\delta^{\pi, \hat{\mu}}_t$ in Theorem~\ref{lem:one_step_var_smaller_stronger_optimal}}\\
=&  \E_{A_t \sim \hat{\mu}_t}\left[\rho_t^2 \hat{q}_t(S_t,A_t) \mid S_t\right]  - v^2_{\pi, t}(S_t) - \delta^{\pi, \hat{\mu}}_t(s) \explain{By \eqref{eq def q hat}}
\\ =&\E_{A_t \sim \pi_t}\left[\sqrt{\hat{q}_t(S_t,A_t)} \mid S_t\right]^2 - v^2_{\pi, t}(S_t) - \delta^{\pi, \hat{\mu}}_t(s).\explain{By \eqref{eq: hat q square}}
\end{align}

We also need the on-policy counterpart. For $t\in[T-2]$ and all $s$,
\begin{align}
&\V\left(\pdisg(\tau^{\pi_{t:T-1}}_{t:T-1})\mid S_t = s\right) \\
=& \E_{A_t \sim \pi_t}\left[\rho_t^2\left(\E_{S_{t+1}}\left[\V\left(\pdisg(\tau^{\pi_{t+1:T-1}}_{t+1:T-1}) \mid S_{t+1}\right) \mid S_t, A_t\right] + \omega_{\pi, t}(S_t, A_t)+q^2_{\pi,t}(S_t,A_t)\right) \mid S_t\right] - v^2_{\pi, t}(S_t) \explain{Lemma~\ref{lem:recursive-var}} \\
=& \E_{A_t \sim \pi_t}\left[\E_{S_{t+1}}\left[\V\left(\pdisg(\tau^{\pi_{t+1:T-1}}_{t+1:T-1}) \mid S_{t+1}\right) \mid S_t, A_t\right] + \omega_{\pi, t}(S_t, A_t)+q^2_{\pi,t}(S_t,A_t) \mid S_t\right] - v^2_{\pi, t}(S_t) \explain{By $\rho_t = 1$ when the behavior policy is $\pi$} \\
=& \E_{A_t \sim \pi_t}\left[\hat{q}_{\pi, t}(S_t,A_t) \mid S_t\right] - v^2_{\pi, t}(S_t).
\label{eq: variance onpolicy hat q T-2}
\end{align}
The last equality is the definition \eqref{eq def q hat} of $\hat q_{\pi, t}$, read for the target policy: the three summands of $\hat q_{\pi, t}$ are exactly the three terms in the bracket. 

Combining, for $t\in[T-2]$ and all $s$,
\begin{align}
&\V\left(\pdisg(\tau^{\pi_{t:T-1}}_{t:T-1})\mid S_t = s\right) - \V\left(\pdisg(\tau^{\hat{\mu}_{t:T-1}}_{t:T-1})\mid S_t = s\right)\\
=& \left(\E_{A_t \sim \pi_t}\left[\hat{q}_t(S_t,A_t) \mid S_t\right] - v^2_{\pi, t}(S_t)\right) - \left(\E_{A_t \sim \pi_t}\left[\sqrt{\hat{q}_t(S_t,A_t)} \mid S_t\right]^2 - v^2_{\pi, t}(S_t) - \delta^{\pi, \hat{\mu}}_t(s)\right) \explain{By \eqref{eq: variance onpolicy hat q T-2} and \eqref{eq: variance comparison hat q T-2}} \\
=& \E_{A_t \sim \pi_t}\left[\hat{q}_t(S_t,A_t) \mid S_t\right] - \E_{A_t \sim \pi_t}\left[\sqrt{\hat{q}_t(S_t,A_t)} \mid S_t\right]^2 + \delta^{\pi, \hat{\mu}}_t(s) \\
=& \V_{A_t \sim \pi_t}\qty(\sqrt{\hat{q}_t(S_t,A_t)} \mid S_t) + \delta^{\pi, \hat{\mu}}_t(s), \explain{Definition of variance}
\end{align}
which is the statement of Theorem~\ref{lem:one_step_var_smaller_stronger_optimal}. Together with the $t = T-1$ case above, in which $\delta^{\pi, \hat{\mu}}_{T-1}(s) = 0$ and both displayed relations are equalities, this establishes the identity for every $t$.

To show $\forall t, s, \delta^{\pi, \hat{\mu}}_t(s) \geq 0$.
We have $\forall t \in [T-2]$,
 $\forall s$,
\begin{align}
&\V\left(\pdisg(\tau^{\pi_{t+1:T-1}}_{t+1:T-1})\mid S_{t+1} = s\right) - \V\left(\pdisg(\tau^{\hat{\mu}_{t+1:T-1}}_{t+1:T-1})\mid S_{t+1} = s\right) \\
=& \V_{A_{t+1} \sim \pi_{t+1}}\qty(\sqrt{\hat{q}_{t+1}(S_{t+1},A_{t+1})} \mid S_{t+1}=s)  + \delta^{\pi, \hat{\mu}}_{t+1}(s) \\
\geq& 0 \label{eq: v pdis - v geq 0 hat}.
\end{align}

Thus, $\forall s$,
\begin{align}
&\delta^{\pi, \hat{\mu}}_t(s) \\
=&\E_{A_t \sim \hat{\mu}_t, S_{t+1}}\left[ \rho_t^2 \qty(\V\left(\pdisg(\tau^{\pi_{t+1:T-1}}_{t+1:T-1}) \mid S_{t+1}\right) - \V\left(\pdisg(\tau^{\hat{\mu}_{t+1:T-1}}_{t+1:T-1}) \mid S_{t+1}\right))  \mid S_t = s\right]\\
\geq&0.\explain{by \eqref{eq: v pdis - v geq 0 hat} and $\rho_t^2 \geq 0$}
\end{align}

This concludes the proof.

\end{proof}

\subsection{Variance Equality}\label{append: variance equality}

We express
$\V\left(\pdisg(\tau^{\pi_{t:T-1}}_{t:T-1}) \mid S_t = s\right) $,
the variance of the on-policy Monte Carlo estimator,
 in the form of a Bellman equation \citep{variance2016Tamar,o2017uncertainty,sherstan2018directly}, using the auxiliary reward $\tilde r_{\pi, t}$ \eqref{def:tilde-r} and the auxiliary value function $\tilde q_{\pi, t}$ \eqref{def:tilde-q} defined in Section~\ref{sec: learn behavior policies}. 
\begin{lemma}[Variance Equality]\label{lem:u_variance_eq}
\begin{align}
\V\left(\pdisg(\tau^{\pi_{t:T-1}}_{t:T-1}) \mid S_t = s\right) = \sum_a \pi_t(a|s)\tilde q_{\pi, t}(s, a) \quad \forall t,s.
\end{align}
\end{lemma}
\begin{proof}
We proceed via induction.
When $t = T-1$, we have
\begin{align}
&\V\left(\pdisg(\tau^{\pi_{t:T-1}}_{t:T-1})\mid S_t\right) \\
=& \V_{A_t}\left(\rho_t r(S_t, A_t)\mid S_t\right)  \\
=& \V_{A_t}\left(r(S_t, A_t)\mid S_t\right) \explain{By on-policy} \\
=& \V_{A_t}\left(q_{\pi,t}(S_t, A_t)\mid S_t\right) \\
=& \E_{A_t}\left[q_{\pi, t}^2(S_t, A_t) \mid S_t\right] -  v_{\pi, t}^2(S_t) \\
=& \sum_{a} \pi_t(a|S_t)\tilde q_{\pi, t}(S_t, a) \explain{By~\eqref{def:tilde-q} and $\omega_{\pi, T-1}(s, a) = 0$}.
\end{align}
For $t \in [T-2]$,
we have
\begin{align}
&\V\left(\pdisg(\tau^{\pi_{t:T-1}}_{t:T-1})\mid S_t\right) \\
=&\E_{A_t}\left[ \E_{S_{t+1}}\left[\V\left(\pdisg(\tau^{\pi_{t+1:T-1}}_{t+1:T-1}) \mid S_{t+1}\right) \mid S_t, A_t\right] +  q_{\pi, t}^2(S_t, A_t) + \omega_{\pi,t}(S_t, A_t) \mid S_t\right]
-  v_{\pi, t}^2(S_t) \explain{Lemma~\ref{lem:recursive-var} and on-policy} \\
=&\sum_{a} \pi_t(a|S_t) \left( \sum_{s'} p(s'|S_t, a) \V\left(\pdisg(\tau^{\pi_{t+1:T-1}}_{t+1:T-1}) \mid S_{t+1} = s' \right) + \tilde r_{\pi, t}(S_t, a) \right)\\
=&\sum_{a} \pi_t(a|S_t) \left( \sum_{s'} p(s'|S_t, a) \sum_{a'}\pi_{t+1}(a'|s')\tilde q_{\pi, t+1}(s', a') + \tilde r_{\pi, t}(S_t, a) \right) \explain{Inductive hypothesis}\\
=&\sum_{a} \pi_t(a|S_t) \tilde q_{\pi, t}(S_t, a), \explain{By~\eqref{def:tilde-q}}
\end{align}
which completes the proof.
\end{proof}

This proves \eqref{eq: variance equality main}. Here,
this $\tilde q$ is exactly the state-action value function of the target policy $\pi$ in the MDP w.r.t. to a new reward function $\tilde r$.
Manipulating~\eqref{eq def q hat} then yields
\begin{align}
\hat q_{\pi, t}(s, a) =& \sum_{s'}p(s'|s, a)\sum_{a'} \pi_{t+1}(a'|s')\tilde q_{\pi, t+1}(s' ,a') + \omega_{\pi, t}(s, a) + q_{\pi, t}^2(s, a) \\
=&\tilde q_{\pi, t}(s, a) + v_{\pi, t}^2(s),
\end{align}
which is \eqref{eq hat and tilde q}.

\subsection{Proof of Theorem~\ref{lem:hat-q-recursive}}\label{append:hat-q-recursive}

\begin{proof}
For $t = T-1$, we have
\begin{align}
 \hat{q}_{\pi, t}(s, a)  &=  q^2_{ \pi, t}( s, a) \explain{Definition of $\hat q_{\pi, t}$~\eqref{eq def q hat last step}} \\
&= \hat{r}_{\pi,t}(s,a). \explain{By $q_{\pi, T-1}(s, a) = r(s, a)$ and the definition of $\hat r_{\pi, t}$ \eqref{def: hat r}}
\end{align}
For $t\in [T-2]$, we have
\begin{align}
&\hat{q}_{\pi,t}(s,a) \\
=& \tilde{q}_{ \pi, t}( s, a) + v^2_{\pi, t}(s) \explain{By~\eqref{eq hat and tilde q}} \\
=&  \tilde r_{\pi,t}(s, a) + v^2_{\pi, t}(s)  + \sum_{s', a'} p(s'|s, a) \pi_{t+1}(a'|s') \tilde{q}_{\pi, t+1}(s', a')  \explain{Definition of $\tilde{q}$ \eqref{def:tilde-q}} \\
=& \tilde r_{\pi,t}(s, a) + v^2_{\pi, t}(s)  + \sum_{s', a'} p(s'|s, a) \pi_{t+1}(a'|s') (\tilde{q}_{\pi, t+1}(s', a') + v^2_{\pi, t+1}(s') -  v^2_{\pi, t+1}(s')) \\
=& \tilde r_{\pi,t}(s, a) + v^2_{\pi, t}(s)  + \sum_{s', a'} p(s'|s, a) \pi_{t+1}(a'|s') (\hat{q}_{\pi, t+1}(s', a') -  v^2_{\pi, t+1}(s')) \explain{By~\eqref{eq hat and tilde q}}\\
=& \omega_{\pi,t}(s, a) + q^2_{\pi, t}(s, a) - \sum_{s'} p(s'|s, a)   v^2_{\pi, t+1}(s') +  \sum_{s', a'} p(s'|s, a) \pi_{t+1}(a'|s') \hat{q}_{\pi, t+1}(s', a') \explain{Definition of $\tilde{r}$ \eqref{def:tilde-r}}\\
=& -(\E [ v_{\pi, t+1}(S_{t+1})\mid S_t=s, A_t=a ])^2 + q^2_{\pi, t}(s, a) + \sum_{s', a'} p(s'|s, a) \pi_{t+1}(a'|s') \hat{q}_{\pi, t+1}(s', a')  \explain{Definition of $\omega$ \eqref{def:nu}} \\
=& -(q_{\pi, t}(s, a) - r(s, a) )^2 + q^2_{\pi, t}(s, a) + \sum_{s', a'} p(s'|s, a) \pi_{t+1}(a'|s') \hat{q}_{\pi, t+1}(s', a') \\
=&  2r(s, a) q_{\pi, t}(s, a)- r^2(s, a)  + \sum_{s', a'} p(s'|s, a) \pi_{t+1}(a'|s') \hat{q}_{\pi, t+1}(s', a') \\
=&  \hat r_{\pi, t}(s, a)  + \sum_{s', a'} p(s'|s, a) \pi_{t+1}(a'|s') \hat{q}_{\pi, t+1}(s', a'), \explain{By the definition of $\hat r_{\pi, t}$ \eqref{def: hat r}}
\end{align}
which completes the proof.
\end{proof}

The Bellman representation also motivates studying TD-based
approximations, drawing on stability analyses under Markovian noise
\citep{liu2025ode,mahadevan2026convergence}
and softmax-Transformer implementations of TD updates
\citep{xie2026beyond}.

\subsection{Proof of Theorem~\ref{lem:error analysis}}\label{append:error analysis}
\begin{proof}

We first derive an important equality.
$\forall t$,
\begin{align}
&\E_{A_t\sim \hat \mu^+_t}\left[{\rho^+_t}^2  \hat q_{\pi, t}(S_t, A_t) \mid S_t\right] \\
=& \sum_{a \, : \, \pi_t(a|S_t) > 0} \frac{\pi_t^2(a|S_t)}{\hat{\mu}_t^+(a|S_t) }  \hat{q}_{\pi, t}(S_t, a) \\
=& \sum_{a \, : \, \pi_t(a|S_t) > 0} \frac{\pi_t^2(a|S_t)}{\frac{\pi_t(a|S_t)\sqrt{\hat q^+_{\pi, t}(S_t, a)}}{ \sum_b \pi_t(b|S_t) \sqrt{\hat q^+_{\pi, t}(S_t, b)}} }  \hat q_{\pi, t}(S_t, a) \explain{by \eqref{def: mu +}}\\
=& \left[\sum_a \pi_t(a|S_t) \sqrt{\hat{q}^+_{\pi, t}(S_t, a)}\right] \qty[\sum_a \pi_t(a|S_t) \frac{ \hat{q}_{\pi, t}(S_t, a)}{\sqrt{\hat{q}^+_{\pi, t}(S_t, a)}} ] \\
=& \left[\sum_a \pi_t(a|S_t) \sqrt{\eta_{\pi, t}(S_t, a)}\sqrt{\hat q_{\pi, t}(S_t, a)}\right] \times \label{eq: error sqrt eq} \\
& \qty[\sum_a \pi_t(a|S_t) \frac{1}{\sqrt{\eta_{\pi, t}(S_t, a)}} \sqrt{\hat q_{\pi, t}(S_t, a)}].\explain{By \eqref{def: eta}}
\end{align}
We proceed via induction.
For $t = T-1$, we  have
\begin{align}
&\V\left(\pdisg(\tau^{\hat{\mu}^+_{t:T-1}}_{t:T-1})\mid S_t\right) \\
=&\E_{A_t\sim \hat{\mu}^+_t}\left[{\rho^+_t}^2 q_{\pi, t}^2(S_t, A_t) \mid S_t\right] - v_{\pi, t}^2(S_t) \explain{Lemma~\ref{lem:recursive-var}}\\
=&\E_{A_t\sim \hat{\mu}^+_t}\left[{\rho^+_t}^2  \hat q_{\pi, t}(S_t, A_t) \mid S_t\right] - v^2_{\pi, t}(S_t) \explain{Definition of $\hat q$ \eqref{eq def q hat last step}} \\
=&\left[\sum_a \pi_t(a|S_t) \sqrt{\eta_{\pi, t}(S_t, a)}\sqrt{\hat{q}_{\pi, t}(S_t, a)}\right] \left[\sum_a \pi_t(a|S_t) \frac{1}{\sqrt{\eta_{\pi, t}(S_t, a)}} \sqrt{\hat{q}_{\pi, t}(S_t, a)}\right]  - v^2_{\pi, t}(S_t) \explain{By \eqref{eq: error sqrt eq}} \\
=& \sum_{a} \pi_t(a|S_t)\hat{q}_{\pi, t}(S_t, a) +  \left[\sum_a \pi_t(a|S_t) \sqrt{\eta_{\pi, t}(S_t, a)}\sqrt{\hat{q}_{\pi, t}(S_t, a)}\right] \left[\sum_a \pi_t(a|S_t) \frac{1}{\sqrt{\eta_{\pi, t}(S_t, a)}} \sqrt{\hat{q}_{\pi, t}(S_t, a)}\right] \\
&- \sum_{a}\pi_t(a|S_t)\hat{q}_{\pi, t}(S_t, a) - v_{\pi, t}^2(S_t) \\
=& \V\left(\pdisg(\tau^{\pi_{t:T-1}}_{t:T-1})\mid S_t\right)    \\
-& \left(\sum_{a}\pi_t(a|S_t)\hat{q}_{\pi, t}(S_t, a)-\left[\sum_a \pi_t(a|S_t) \sqrt{\eta_{\pi, t}(S_t, a)}\sqrt{\hat{q}_{\pi, t}(S_t, a)}\right] \left[\sum_a \pi_t(a|S_t) \frac{1}{\sqrt{\eta_{\pi, t}(S_t, a)}} \sqrt{\hat{q}_{\pi, t}(S_t, a)}\right] \right)\explain{By~\eqref{eq hat and tilde q} and Lemma~\ref{lem:u_variance_eq}} \\
=& \V\left(\pdisg(\tau^{\pi_{t:T-1}}_{t:T-1})\mid S_t\right) - \epsilon^+_t(S_t).  \explain{Definition of $\epsilon^+$ \eqref{def: epsilon +}}
\end{align}

For $t\in [T-2]$,
we have
\begin{align}
&\V\left(\pdisg(\tau^{\hat{\mu}^+_{t:T-1}}_{t:T-1})\mid S_t\right) \\
=&\E_{A_t\sim \hat{\mu}^+_t}\left[\rho_t^2 \left(\E_{S_{t+1}}\left[\V\left(\pdisg(\tau^{\hat{\mu}^+_{t+1:T-1}}_{t+1:T-1}) \mid S_{t+1}\right) \mid S_t, A_t\right] + \omega_{\pi,t}(S_t, A_t) + q_{\pi, t}^2(S_t, A_t)\right) \mid S_t\right] \\
&-  v_{\pi, t}^2(S_t) \explain{Lemma~\ref{lem:recursive-var}}\\
\leq&\E_{A_t\sim \hat{\mu}^+_t}\Big[\rho_t^2 \Big(\E_{S_{t+1}}\Big[ \sum_{a'} \pi_{t+1}(a'|S_{t+1})\tilde {q}_{\pi, t+1}(S_{t+1}, a') \mid S_t, A_t\Big] + \omega_{\pi,t}(S_t, A_t) \\
& +  q_{\pi, t}^2(S_t, A_t)\Big) \mid S_t\Big]  - v_{\pi, t}^2(S_t) - \E_{A_t \sim \hat{\mu}^+_t}\left[\rho_t^2\E_{S_{t+1}}\left[\epsilon^+_{t+1}(S_{t+1})\mid S_t, A_t\right]\right] \explain{Inductive hypothesis and Lemma~\ref{lem:u_variance_eq}} \\
=&\E_{A_t\sim \hat{\mu}^+_t}\left[\rho_t^2 \left(\tilde{q}_{\pi, t}(S_t, A_t) + v^2_{\pi, t}(S_t)\right) \mid S_t\right] - v^2_{\pi, t}(S_t)  - \E_{A_t \sim \hat{\mu}^+_t}\left[\rho_t^2\E_{S_{t+1}}\left[\epsilon^+_{t+1}(S_{t+1})\mid S_t, A_t\right]\right] \explain{Definition of $\tilde q $ \eqref{def:tilde-q}} \\
=&\E_{A_t\sim \hat{\mu}^+_t}\left[\rho_t^2 \hat q_{\pi, t}(S_t, A_t) \mid S_t\right] - v^2_{\pi, t}(S_t)  - \E_{A_t \sim \hat{\mu}^+_t}\left[\rho_t^2\E_{S_{t+1}}\left[\epsilon^+_{t+1}(S_{t+1})\mid S_t, A_t\right]\right] \explain{Definition of $\hat q$ \eqref{eq def q hat}} \\
=& \left[\sum_a \pi_t(a|S_t) \sqrt{\eta_{\pi, t}(S_t, a)}\sqrt{\hat{q}_{\pi, t}(S_t, a)}\right] \left[\sum_a \pi_t(a|S_t) \frac{1}{\sqrt{\eta_{\pi, t}(S_t, a)}} \sqrt{\hat{q}_{\pi, t}(S_t, a)}\right] - v_{\pi, t}^2(S_t)  \\
&- \E_{A_t \sim \hat{\mu}^+_t}\left[\rho_t^2\E_{S_{t+1}}\left[\epsilon^+_{t+1}(S_{t+1})\mid S_t, A_t\right]\right] \explain{By \eqref{eq: error sqrt eq}}\\
=& \sum_{a} \pi_t(a|S_t)\hat{q}_{\pi, t}(S_t, a)  - v^2_{\pi, t}(S_t)  \\
&+ \left[\sum_a \pi_t(a|S_t) \sqrt{\eta_{\pi, t}(S_t, a)}\sqrt{\hat{q}_{\pi, t}(S_t, a)}\right] \left[\sum_a \pi_t(a|S_t) \frac{1}{\sqrt{\eta_{\pi, t}(S_t, a)}} \sqrt{\hat{q}_{\pi, t}(S_t, a)}\right] - \sum_{a} \pi_t(a|S_t)\hat{q}_{\pi, t}(S_t, a)  \\
&- \E_{A_t \sim \hat{\mu}^+_t}\left[\rho_t^2\E_{S_{t+1}}\left[\epsilon^+_{t+1}(S_{t+1})\mid S_t, A_t\right]\right] \\
=& \V\left(\pdisg(\tau^{\pi_{t:T-1}}_{t:T-1})\mid S_t\right)  +  \left[\sum_a \pi_t(a|S_t) \sqrt{\eta_{\pi, t}(S_t, a)}\sqrt{\hat{q}_{\pi, t}(S_t, a)}\right] \left[\sum_a \pi_t(a|S_t) \frac{1}{\sqrt{\eta_{\pi, t}(S_t, a)}} \sqrt{\hat{q}_{\pi, t}(S_t, a)}\right]  \\
&- \sum_{a}\pi_t(a|S_t)\hat{q}_{\pi, t}(S_t, a)- \E_{A_t \sim \hat{\mu}^+_t}\left[\rho_t^2\E_{S_{t+1}}\left[\epsilon^+_{t+1}(S_{t+1})\mid S_t, A_t\right]\right]\explain{By~\eqref{eq hat and tilde q} and Lemma~\ref{lem:u_variance_eq}}\\
=& \V\left(\pdisg(\tau^{\pi_{t:T-1}}_{t:T-1})\mid S_t\right) - \epsilon^+_t(S_t)  \explain{Definition of $\epsilon^+$ \eqref{def: epsilon +}}.
\end{align}

\end{proof}

\subsection{Proof of Lemma~\ref{lem: second moment}}\label{append: second moment}
\begin{proof}
Fix $s$ and $a$. For $t = T-1$, we have
\begin{align}
\hat q_{\pi, T-1}(s, a)
&= q_{\pi, T-1}^2(s, a) \explain{Definition of $\hat q_{\pi, t}$ \eqref{eq def q hat last step}} \\
&= r^2(s, a) \explain{By $q_{\pi, T-1}(s, a) = r(s, a)$} \\
&= \E_\pi\left[G_{T-1}^2 \mid S_{T-1} = s, A_{T-1} = a\right], \label{eq: second moment terminal}
\end{align}
where the last equality holds because $G_{T-1} = R_T = r(s, a)$ is deterministic given $(S_{T-1}, A_{T-1}) = (s, a)$.

For $t \in [T-2]$, we first rewrite the third term of \eqref{eq def q hat}. When the behavior policy coincides with the target policy, every importance sampling ratio is one, so
\begin{align}
\pdisg(\tau^{\pi_{t+1:T-1}}_{t+1:T-1})
&= \sum_{k=t+1}^{T-1} \rho_{t+1:k} R_{k+1} \explain{Definition of PDIS \eqref{eq pdis def}} \\
&= \sum_{k=t+1}^{T-1} R_{k+1} \explain{By $\rho_k = \pi_k(A_k|S_k)/\pi_k(A_k|S_k) = 1$} \\
&= G_{t+1}, \label{eq: pdis onpolicy is return}
\end{align}
and therefore
\begin{align}
&\sum_{s'} p(s'|s, a) \V\left(\pdisg(\tau^{\pi_{t+1:T-1}}_{t+1:T-1}) \mid S_{t+1} = s'\right) \\
&\qquad = \E_{S_{t+1}}\left[\V_\pi\left(G_{t+1} \mid S_{t+1}\right) \mid S_t = s, A_t = a\right]. \label{eq: third term is expected variance}
\end{align}
Next we apply the law of total variance to $G_{t+1}$, conditioning on $S_{t+1}$:
\begin{align}
\V_\pi\left(G_{t+1} \mid S_t = s, A_t = a\right)
&= \E_{S_{t+1}}\left[\V_\pi\left(G_{t+1} \mid S_{t+1}\right) \mid s, a\right] + \V_{S_{t+1}}\left(\E_\pi\left[G_{t+1} \mid S_{t+1}\right] \mid s, a\right) \explain{Law of total variance} \\
&= \E_{S_{t+1}}\left[\V_\pi\left(G_{t+1} \mid S_{t+1}\right) \mid s, a\right] + \V_{S_{t+1}}\left(v_{\pi, t+1}(S_{t+1}) \mid s, a\right) \explain{Definition of $v_{\pi, t+1}$} \\
&= \E_{S_{t+1}}\left[\V_\pi\left(G_{t+1} \mid S_{t+1}\right) \mid s, a\right] + \omega_{\pi, t}(s, a). \label{eq: total variance step}
\end{align}
Since the reward is deterministic, $G_t = r(s, a) + G_{t+1}$ conditionally on $(S_t, A_t) = (s, a)$, and adding a constant leaves a variance unchanged, so
\begin{align}
\V_\pi\left(G_t \mid s, a\right) = \V_\pi\left(G_{t+1} \mid s, a\right). \label{eq: shift invariance}
\end{align}
Combining the pieces, for $t \in [T-2]$,
\begin{align}
\hat q_{\pi, t}(s, a)
&= q_{\pi, t}^2(s, a) + \omega_{\pi, t}(s, a) + \sum_{s'} p(s'|s, a) \V\left(\pdisg(\tau^{\pi_{t+1:T-1}}_{t+1:T-1}) \mid S_{t+1} = s'\right) \explain{Definition of $\hat q_{\pi, t}$ \eqref{eq def q hat}} \\
&= q_{\pi, t}^2(s, a) + \omega_{\pi, t}(s, a) + \E_{S_{t+1}}\left[\V_\pi\left(G_{t+1} \mid S_{t+1}\right) \mid s, a\right] \explain{By \eqref{eq: third term is expected variance}} \\
&= q_{\pi, t}^2(s, a) + \V_\pi\left(G_{t+1} \mid s, a\right) \explain{By \eqref{eq: total variance step}} \\
&= q_{\pi, t}^2(s, a) + \V_\pi\left(G_{t} \mid s, a\right) \explain{By \eqref{eq: shift invariance}} \\
&= \E_\pi\left[G_t \mid s, a\right]^2 + \V_\pi\left(G_{t} \mid s, a\right) \explain{Definition of $q_{\pi, t}$} \\
&= \E_\pi\left[G_t^2 \mid s, a\right], \explain{Definition of variance}
\end{align}
which together with \eqref{eq: second moment terminal} proves \eqref{eq: second moment}. Finally,
\begin{align}
\abs{G_t} = \abs{\sum_{i = t+1}^{T} R_i} \leq \sum_{i=t+1}^{T} \abs{r(S_{i-1}, A_{i-1})} \leq (T - t) R_{\max} = L_t R_{\max}, \label{eq: return range}
\end{align}
so $0 \leq \E_\pi[G_t^2 \mid s, a] \leq L_t^2 R_{\max}^2$, which completes the proof.
\end{proof}

\subsection{Proof of Corollary~\ref{cor: online complexity}}\label{append: online complexity}
\begin{proof}
Since $\mu \in \hat \Lambda$, we have
\begin{align}
\E\left[Y^\mu\right]
&= \sum_s p_0(s) \E\left[\pdisg(\tau^{\mu_{0:T-1}}_{0:T-1}) \mid S_0 = s\right] \explain{Definition of $Y^\mu$ \eqref{def: Y sigma}} \\
&= \sum_s p_0(s) v_{\pi, 0}(s) \explain{By Lemma~\ref{lemma: Lambda unbiasedness hat}} \\
&= J(\pi). \label{eq: Y unbiased}
\end{align}
The $Y_i^\mu$ are i.i.d.\ with mean $J(\pi)$ and variance $\sigma_\mu^2$, so
\begin{align}
\E\left[\left(\frac{1}{N}\sum_{i=1}^N Y_i^\mu - J(\pi)\right)^2\right]
&= \V\left(\frac{1}{N}\sum_{i=1}^N Y_i^\mu\right) \explain{By \eqref{eq: Y unbiased}, the estimator is unbiased} \\
&= \frac{\sigma_\mu^2}{N}, \explain{Independence}
\end{align}
which is \eqref{eq: mse identity}. For $\sigma_\mu^2 > 0$, the smallest integer $N$ with $\sigma_\mu^2/N \leq \varepsilon^2$ is $\lceil \sigma_\mu^2/\varepsilon^2 \rceil$.

For the last claim, let $\mu \in \Lambda^*$ be arbitrary. Then
\begin{align}
\sigma_\mu^2
&= \E_{S_0}\left[\V\left(\pdisg(\tau^{\mu_{0:T-1}}_{0:T-1}) \mid S_0\right)\right] + \V_{S_0}\left(v_{\pi, 0}(S_0)\right), \label{eq: sigma decomposition}
\end{align}
in which the second summand depends only on $\pi$ and $p_0$. Theorem~\ref{lem:rl-optimal} states that $\mu^*$ minimizes $\V(\pdisg(\tau^{\mu_{0:T-1}}_{0:T-1}) \mid S_0 = s)$ over $\Lambda^*$ simultaneously for every $s$, hence it minimizes the $p_0$-average of these conditional variances, hence it minimizes $\sigma_\mu^2$ over $\Lambda^*$ by \eqref{eq: sigma decomposition}. Since $\lceil \sigma_\mu^2/\varepsilon^2 \rceil$ is non-decreasing in $\sigma_\mu^2$, $\mu^*$ minimizes it as well.
\end{proof}

\subsection{Proof of Theorem~\ref{thm: robust local}}\label{append: robust local}
The proof of Theorem~\ref{thm: robust local} appears after Lemma~\ref{lem: kantorovich}, which records the inequality the argument rests on. We include a self-contained proof of the lemma; see \citet{horn2012matrix} for the standard statement of the Kantorovich inequality, of which \eqref{eq: kantorovich} is the scalar case.
\begin{lemma}[Kantorovich Inequality]\label{lem: kantorovich}
Let $Y$ be a random variable with $\Pr(\ell \leq Y \leq u) = 1$ for some $0 < \ell \leq u < \infty$. Then
\begin{align}\label{eq: kantorovich}
\E[Y]\,\E\!\left[Y^{-1}\right] \; \leq \; \frac{(u + \ell)^2}{4 u \ell}.
\end{align}
\end{lemma}
\begin{proof}
We have
\begin{align}
(Y - \ell)(u - Y) &\geq 0 \explain{By $\ell \leq Y \leq u$ almost surely} \\
\Longrightarrow \quad Y^2 - (\ell + u) Y + \ell u &\leq 0 \\
\Longrightarrow \quad Y + \ell u\, Y^{-1} &\leq \ell + u, \label{eq: kant pointwise}
\end{align}
almost surely. Taking expectations in \eqref{eq: kant pointwise} and then applying the arithmetic--geometric mean inequality,
\begin{align}
2\sqrt{\ell u\, \E[Y]\,\E\left[Y^{-1}\right]}
&\leq \E[Y] + \ell u\, \E\left[Y^{-1}\right] \explain{Arithmetic--geometric mean inequality} \\
&\leq \ell + u. \explain{By \eqref{eq: kant pointwise}}
\end{align}
Squaring both sides and dividing by $4\ell u > 0$ yields \eqref{eq: kantorovich}.
\end{proof}

We now prove Theorem~\ref{thm: robust local}.
\begin{proof}
Fix $t$ and $s$, and abbreviate
\begin{align}
X_a \doteq \sqrt{\hat q_{\pi, t}(s, a)} \geq 0, \qquad Y_a \doteq \sqrt{\eta_{\pi, t}(s, a)} > 0, \qquad Z \doteq Z_{\pi, t}(s) = \sum_a \pi_t(a|s) X_a, \label{def: XYZ shorthand}
\end{align}
where positivity of $Y_a$ is by \eqref{def: eta}.

\emph{Case $Z = 0$.} Then $\pi_t(a|s) X_a = 0$ for every $a$, hence $\pi_t(a|s)\hat q_{\pi, t}(s, a) = 0$ for every $a$, and so
\begin{align}
\Delta_{\pi, t}(s) = \sum_a \pi_t(a|s) X_a^2 - Z^2 = 0 - 0 = 0. \explain{Definition of $\Delta_{\pi, t}$ \eqref{def: Z Delta}}
\end{align}
Every summand of both bracketed factors in the definition of $c^+_t(s)$ in Theorem~\ref{lem:error analysis} carries a factor $\pi_t(a|s) X_a = 0$, so both factors vanish and $c^+_t(s) = 0 = \Delta_{\pi, t}(s) - \lambda(\alpha_m) Z^2$; the claim holds with equality.

\emph{Case $Z > 0$.} Define $w_a \doteq \pi_t(a|s) X_a / Z$. Then $w_a \geq 0$ and $\sum_a w_a = 1$ by \eqref{def: XYZ shorthand}, so $w$ is a probability distribution on $\fA$, supported on actions with $\pi_t(a|s) > 0$. Writing $\E_w$ for expectation under $w$,
\begin{align}
\sum_a \pi_t(a|s) \sqrt{\eta_{\pi, t}(s, a)} \sqrt{\hat q_{\pi, t}(s, a)}
&= \sum_a \pi_t(a|s) X_a Y_a \explain{By \eqref{def: XYZ shorthand}} \\
&= Z \sum_a w_a Y_a = Z\, \E_w[Y], \label{eq: first factor}
\end{align}
and likewise
\begin{align}
\sum_a \pi_t(a|s) \frac{1}{\sqrt{\eta_{\pi, t}(s, a)}} \sqrt{\hat q_{\pi, t}(s, a)}
&= \sum_a \pi_t(a|s) \frac{X_a}{Y_a} \explain{By \eqref{def: XYZ shorthand}} \\
&= Z \sum_a w_a Y_a^{-1} = Z\, \E_w\!\left[Y^{-1}\right]. \label{eq: second factor}
\end{align}
Therefore
\begin{align}
c^+_t(s)
&= \sum_a \pi_t(a|s) \hat q_{\pi, t}(s, a) \\
&\quad - \left(\sum_a \pi_t(a|s)\sqrt{\eta_{\pi, t}(s, a)}\sqrt{\hat q_{\pi, t}(s, a)}\right)\left(\sum_a \pi_t(a|s)\frac{1}{\sqrt{\eta_{\pi, t}(s, a)}}\sqrt{\hat q_{\pi, t}(s, a)}\right) \explain{Definition of $c^+_t$ in Theorem~\ref{lem:error analysis}} \\
&= \E_{A \sim \pi_t}\left[X_A^2\right] - Z^2\, \E_w[Y]\, \E_w\!\left[Y^{-1}\right]. \label{eq: c plus in w form}
\end{align}
On the event $\mathcal{E}_m$, Assumption~\ref{assum: relative error} gives $\eta_{\pi, t}(s, a) \in [1 - \alpha_m, 1 + \alpha_m]$ for every $a$ with $\pi_t(a|s) > 0$, and $w$ is supported on such actions, so under $w$ the variable $Y$ takes values in $[\ell, u]$ with $\ell \doteq \sqrt{1 - \alpha_m} > 0$ and $u \doteq \sqrt{1 + \alpha_m}$. Lemma~\ref{lem: kantorovich} therefore applies, and
\begin{align}
\frac{(u + \ell)^2}{4 u \ell}
&= \frac{(1 + \alpha_m) + (1 - \alpha_m) + 2\sqrt{(1 + \alpha_m)(1 - \alpha_m)}}{4 \sqrt{(1+\alpha_m)(1-\alpha_m)}} \explain{Expanding $(u+\ell)^2$} \\
&= \frac{2 + 2\sqrt{1 - \alpha_m^2}}{4\sqrt{1 - \alpha_m^2}} = \frac{1}{2}\left(\frac{1}{\sqrt{1 - \alpha_m^2}} + 1\right) \\
&= 1 + \lambda(\alpha_m). \label{eq: kant constant}
\end{align}
Substituting into \eqref{eq: c plus in w form},
\begin{align}
c^+_t(s)
&\geq \E_{A \sim \pi_t}\left[X_A^2\right] - \left(1 + \lambda(\alpha_m)\right) Z^2 \explain{By Lemma~\ref{lem: kantorovich} and \eqref{eq: kant constant}} \\
&= \left(\E_{A \sim \pi_t}\left[X_A^2\right] - \E_{A \sim \pi_t}\left[X_A\right]^2\right) - \lambda(\alpha_m) Z^2 \explain{By $\E_{A\sim\pi_t}[X_A] = Z$ from \eqref{def: XYZ shorthand}} \\
&= \V_{A \sim \pi_t}\left(X_A\right) - \lambda(\alpha_m) Z^2 \explain{Definition of variance} \\
&= \Delta_{\pi, t}(s) - \lambda(\alpha_m) Z_{\pi, t}(s)^2, \explain{Definition of $\Delta_{\pi, t}$ and $Z_{\pi, t}$ \eqref{def: Z Delta}}
\end{align}
which is \eqref{eq: robust local}. Finally, for $\abs{\alpha} < 1$ the binomial series gives $(1 - \alpha^2)^{-1/2} = 1 + \frac{1}{2}\alpha^2 + \frac{3}{8}\alpha^4 + O(\alpha^6)$, hence
\begin{align}
\lambda(\alpha) = \frac{1}{2}\left(\frac{1}{2}\alpha^2 + \frac{3}{8}\alpha^4 + O(\alpha^6)\right) = \frac{\alpha^2}{4} + \frac{3\alpha^4}{16} + O(\alpha^6). \explain{Definition of $\lambda$ \eqref{def: lambda alpha}}
\end{align}
\end{proof}

\subsection{Proof of Corollary~\ref{cor: retained fraction}}\label{append: retained fraction}
\begin{proof}
Fix $t$ and $s$ with $Z \doteq Z_{\pi, t}(s) > 0$ and write $\kappa^2 \doteq \kappa^2_{\pi, t}(s)$. We first isolate the condition on $\lambda(\alpha_m)$:
\begin{align}
c^+_t(s) \geq (1 - \zeta)\Delta_{\pi, t}(s)
&\Longleftarrow \; \Delta_{\pi, t}(s) - \lambda(\alpha_m) Z^2 \geq (1 - \zeta)\Delta_{\pi, t}(s) \explain{By Theorem~\ref{thm: robust local}} \\
&\Longleftrightarrow \; \lambda(\alpha_m) Z^2 \leq \zeta \Delta_{\pi, t}(s) \\
&\Longleftrightarrow \; \lambda(\alpha_m) \leq \zeta \kappa^2. \label{eq: lambda condition}
\end{align}
It therefore suffices to characterize $\qty{\alpha \in [0, 1) : \lambda(\alpha) \leq x}$ for $x \doteq \zeta \kappa^2 \geq 0$. Since $\alpha \mapsto (1 - \alpha^2)^{-1/2}$ is strictly increasing on $[0, 1)$, so is $\lambda$, and for $\alpha \in [0, 1)$,
\begin{align}
\lambda(\alpha) \leq x
&\Longleftrightarrow \frac{1}{\sqrt{1 - \alpha^2}} \leq 1 + 2x \explain{Definition of $\lambda$ \eqref{def: lambda alpha}} \\
&\Longleftrightarrow \sqrt{1 - \alpha^2} \geq \frac{1}{1 + 2x} \explain{Inverting two positive quantities} \\
&\Longleftrightarrow 1 - \alpha^2 \geq \frac{1}{(1 + 2x)^2} \explain{Squaring two non-negative quantities} \\
&\Longleftrightarrow \alpha^2 \leq 1 - \frac{1}{(1 + 2x)^2}. \label{eq: alpha threshold}
\end{align}
Every step is an equivalence, so combining \eqref{eq: lambda condition} and \eqref{eq: alpha threshold} with $x = \zeta\kappa^2$ gives \eqref{eq: retained fraction}. Taking $\zeta = 1$ gives $c^+_t(s) \geq 0 \cdot \Delta_{\pi, t}(s) = 0$, which is the non-negativity statement.

For the asymptotics, set $y \doteq 2\zeta\kappa^2$. Then
\begin{align}
1 - \frac{1}{(1+y)^2} = 1 - \left(1 - 2y + 3y^2 + O(y^3)\right) = 2y - 3y^2 + O(y^3) = 4\zeta\kappa^2 - 12\zeta^2\kappa^4 + O(\kappa^6),
\end{align}
so the threshold in \eqref{eq: retained fraction} equals $4\zeta\kappa^2_{\pi, t}(s) + O(\kappa^4_{\pi, t}(s))$ as $\kappa_{\pi, t}(s) \to 0$.
\end{proof}

\subsection{Proof of Theorem~\ref{thm: certificate}}\label{append: certificate}
\begin{proof}
We work throughout on the event $\mathcal{E}_m$ of Assumption~\ref{assum: relative error}, and prove by backward induction 
\begin{align}
\epsilon^+_t(s) \geq \underline{\epsilon}_{m, t}(s) \qquad \forall t, \forall s. \label{eq: eps dominates}
\end{align}

For $t = T-1$,
\begin{align}
\epsilon^+_{T-1}(s)
&= c^+_{T-1}(s) \explain{Definition of $\epsilon^+$ \eqref{def: epsilon +}} \\
&\geq \Delta_{\pi, T-1}(s) - \lambda(\alpha_m) Z_{\pi, T-1}(s)^2 \explain{By Theorem~\ref{thm: robust local}} \\
&= \underline{c}_{m, T-1}(s) = \underline{\epsilon}_{m, T-1}(s). \explain{Definition of $\underline \epsilon_{m, t}$ \eqref{def: underline epsilon}}
\end{align}
Now assume \eqref{eq: eps dominates} holds at $t+1$. For any $s$, the map
\begin{align}
\mathcal{L}_t^s f \doteq \E_{A_t \sim \hat \mu^+_t(\cdot|s)}\left[\rho_t^2 \, \E_{S_{t+1}}\left[f(S_{t+1}) \mid s, A_t\right]\right] = \sum_{a \, : \, \pi_t(a|s) > 0} \frac{\pi_t^2(a|s)}{\hat \mu^+_t(a|s)} \sum_{s'} p(s'|s, a) f(s') \label{def: T operator}
\end{align}
has only non-negative coefficients $\pi_t^2(a|s) p(s'|s, a)/\hat \mu^+_t(a|s)$, so it is a positive linear operator and is therefore monotone: $f \geq g$ pointwise implies $\mathcal{L}^s_t f \geq \mathcal{L}^s_t g$, \emph{with no sign restriction on $f$ or $g$}. The sum is restricted to $a$ with $\pi_t(a|s) > 0$: those are the only actions contributing to $\E_{A_t \sim \hat \mu^+_t}[\rho_t^2 \,\cdot\,]$, and on that set Assumption~\ref{assum: relative error} makes $\hat q^+_{\pi, t}(s,a) > 0$ and hence $\hat \mu^+_t(a|s) > 0$ by \eqref{def: mu +}, so no summand is of the form $0/0$. Consequently
\begin{align}
\epsilon^+_t(s)
&= c^+_t(s) + \mathcal{L}^s_t \epsilon^+_{t+1} \explain{Definition of $\epsilon^+$ \eqref{def: epsilon +} and \eqref{def: T operator}} \\
&\geq \underline{c}_{m, t}(s) + \mathcal{L}^s_t \epsilon^+_{t+1} \explain{By Theorem~\ref{thm: robust local}} \\
&\geq \underline{c}_{m, t}(s) + \mathcal{L}^s_t \underline{\epsilon}_{m, t+1} \explain{Inductive hypothesis and monotonicity of $\mathcal{L}^s_t$} \\
&= \underline{\epsilon}_{m, t}(s), \explain{Definition of $\underline \epsilon_{m, t}$ \eqref{def: underline epsilon}}
\end{align}
which completes the induction. In particular \eqref{eq: eps dominates} remains valid when $\underline{\epsilon}_{m, t+1}(s')$ is negative at some successor states. Combining with Theorem~\ref{lem:error analysis},
\begin{align}
\V\left(\pdisg(\tau^{\hat \mu^+_{t:T-1}}_{t:T-1}) \mid S_t = s\right)
&\leq \V\left(\pdisg(\tau^{\pi_{t:T-1}}_{t:T-1}) \mid S_t = s\right) - \epsilon^+_t(s) \explain{By Theorem~\ref{lem:error analysis}} \\
&\leq \V\left(\pdisg(\tau^{\pi_{t:T-1}}_{t:T-1}) \mid S_t = s\right) - \underline{\epsilon}_{m, t}(s). \label{eq: certificate statewise}
\end{align}

For the unconditional statement we first check that both policies are admissible. By \eqref{def: mu +} and the standing positivity $\hat q^+_{\pi, t} > 0$ assumed before Theorem~\ref{lem:error analysis}, $\hat \mu^+_t(a|s) = 0$ can occur only if $\pi_t(a|s) = 0$, which implies $\pi_t(a|s)\hat q_{\pi, t}(s, a) = 0$; under the uniform-distribution convention of \eqref{def hat mu} we instead have $\hat \mu^+_t(a|s) = 1/\na > 0$. Hence $\hat \mu^+ \in \hat \Lambda$ by \eqref{def: Lambda hat}. Also $\pi \in \Lambda^- \subseteq \hat \Lambda$ by Lemma~\ref{lemma: Lambda broadness hat q}. Therefore \eqref{eq:varaince-1} applies to both, giving
\begin{align}
\sigma_{\hat \mu^+}^2 &= \E_{S_0}\left[\V\left(\pdisg(\tau^{\hat \mu^+_{0:T-1}}_{0:T-1}) \mid S_0\right)\right] + \V_{S_0}\left(v_{\pi, 0}(S_0)\right), \label{eq: sigma mu plus decomp} \\
\sigma_{\pi}^2 &= \E_{S_0}\left[\V\left(\pdisg(\tau^{\pi_{0:T-1}}_{0:T-1}) \mid S_0\right)\right] + \V_{S_0}\left(v_{\pi, 0}(S_0)\right), \label{eq: sigma pi decomp}
\end{align}
with the \emph{same} second summand, which depends only on $\pi$ and $p_0$. Taking $t = 0$ in \eqref{eq: certificate statewise}, averaging over $S_0 \sim p_0$, and subtracting \eqref{eq: sigma pi decomp} from \eqref{eq: sigma mu plus decomp},
\begin{align}
\sigma_{\hat \mu^+}^2 - \sigma_\pi^2
&\leq -\E_{S_0}\left[\underline{\epsilon}_{m, 0}(S_0)\right] \explain{By \eqref{eq: certificate statewise}, \eqref{eq: sigma mu plus decomp}, and \eqref{eq: sigma pi decomp}} \\
&= -\Gamma_m, \explain{Definition of $\Gamma_m$ \eqref{def: Gamma m}}
\end{align}
which is \eqref{eq: sigma certificate}. Finally, if $\underline{c}_{m, t}(s) \geq 0$ at every $t$ and every $s$ reachable under $\hat \mu^+$, then every term of the recursion \eqref{def: underline epsilon} is non-negative on the reachable set, hence $\underline{\epsilon}_{m, 0}(s) \geq 0$ for every $s$ in the support of $p_0$ and $\Gamma_m \geq 0$ by \eqref{def: Gamma m}.
\end{proof}

\subsection{Range Bound and Proof of the High-Probability Online Requirement}\label{append: end to end}
Assumption~\ref{assum: relative error}, together with $\hat q_{\min} > 0$, bounds the one-step importance ratio under the learned behavior policy, and hence the range of a single-trajectory estimator. This is the ingredient the online concentration step requires.
\begin{lemma}[Explicit Ratio and Range Bounds]\label{lem: ratio bound}
Let
\begin{align}\label{def: q min max}
\hat q_{\min} \doteq \inf_{(t, s, a) \in \mathcal{K}} \hat q_{\pi, t}(s, a), \qquad \hat q_{\max} \doteq \sup_{(t, s, a) \in \mathcal{K}} \hat q_{\pi, t}(s, a),
\end{align}
suppose $\hat q_{\min} > 0$, and note that $\hat q_{\max} \leq T^2 R_{\max}^2$ by Lemma~\ref{lem: second moment}. On the event $\mathcal{E}_m$ of Assumption~\ref{assum: relative error},
\begin{align}\label{eq: rho bar}
\rho_t = \frac{\pi_t(A_t|S_t)}{\hat \mu^+_t(A_t|S_t)} \; \leq \; \bar \rho_m \doteq \sqrt{\frac{1 + \alpha_m}{1 - \alpha_m} \cdot \frac{\hat q_{\max}}{\hat q_{\min}}} \qquad \text{almost surely},
\end{align}
and consequently, with
\begin{align}\label{def: B b}
B_m \doteq R_{\max} \sum_{k=0}^{T-1} \bar \rho_m^{\,k+1}, \qquad b_m \doteq B_m + T R_{\max},
\end{align}
we have $\abs{Y^{\hat \mu^+}} \leq B_m$ and $\abs{Y^{\hat \mu^+} - J(\pi)} \leq b_m$ almost surely.
\end{lemma}
\begin{proof}
Let $a$ satisfy $\pi_t(a|s) > 0$. Then
\begin{align}
\rho_t
&= \frac{\pi_t(a|s)}{\hat \mu^+_t(a|s)} = \frac{\sum_b \pi_t(b|s)\sqrt{\hat q^+_{\pi, t}(s, b)}}{\sqrt{\hat q^+_{\pi, t}(s, a)}} \explain{By \eqref{def: mu +}} \\
&\leq \frac{\sum_b \pi_t(b|s)\sqrt{(1 + \alpha_m)\hat q_{\pi, t}(s, b)}}{\sqrt{(1 - \alpha_m)\hat q_{\pi, t}(s, a)}} \explain{By \eqref{eq: relative error event}} \\
&\leq \frac{\sqrt{(1 + \alpha_m)\hat q_{\max}}}{\sqrt{(1 - \alpha_m)\hat q_{\min}}} \sum_b \pi_t(b|s) = \bar \rho_m. \explain{Definitions of $\hat q_{\max}$, $\hat q_{\min}$ and $\sum_b \pi_t(b|s) = 1$}
\end{align}
For the range,
\begin{align}
\abs{Y^{\hat \mu^+}}
&= \abs{\sum_{k=0}^{T-1}\rho_{0:k} R_{k+1}} \explain{Definition of PDIS \eqref{eq pdis def}} \\
&\leq \sum_{k=0}^{T-1} \bar\rho_m^{\,k+1} R_{\max} = B_m, \explain{By \eqref{eq: rho bar}, $\rho_{0:k}$ being a product of $k+1$ ratios}
\end{align}
and the centered bound follows from $\abs{J(\pi)} \leq T R_{\max}$, which holds by \eqref{eq: return range}.
\end{proof}

We now prove Theorem~\ref{thm: end to end}.

\begin{proof}
Condition on $\mathcal{D}_m$. The learned behavior policy $\hat \mu^+$ is then a fixed element of $\hat \Lambda$ (established in the proof of Theorem~\ref{thm: certificate}), $N$ is a fixed integer, and $Y_1^{\hat\mu^+}, \dots, Y_N^{\hat\mu^+}$ are i.i.d. Applying Corollary~\ref{cor: online complexity} conditionally,
\begin{align}\label{eq: conditional mse exact}
\E\left[\left(\hat J_{m, N} - J(\pi)\right)^2 \;\middle|\; \mathcal{D}_m\right] = \frac{\sigma_{\hat \mu^+}^2}{N}.
\end{align}
On the event $\mathcal{E}_m$ we have $\sigma_{\hat \mu^+}^2 \leq \sigma_\pi^2 - \Gamma_m$ by Theorem~\ref{thm: certificate}, and substituting this into \eqref{eq: conditional mse exact} gives \eqref{eq: conditional mse}. Conditionally on $\mathcal{D}_m$, the summands $Y_i^{\hat \mu^+} - J(\pi)$ are i.i.d.\ and centered by \eqref{eq: Y unbiased}, are bounded in absolute value by $b_m$ by Lemma~\ref{lem: ratio bound}, and have variance $\sigma_{\hat \mu^+}^2$, so Bernstein's inequality gives
\begin{align}
\Pr\left(\abs{\hat J_{m, N} - J(\pi)} > \varepsilon \;\middle|\; \mathcal{D}_m\right)
&\leq 2 \exp\left(-\frac{N \varepsilon^2}{2 \sigma_{\hat \mu^+}^2 + \frac{2}{3} b_m \varepsilon}\right) \explain{Bernstein's inequality} \\
&\leq 2 \exp\left(-\frac{N \varepsilon^2}{2\left(\sigma_\pi^2 - \Gamma_m\right) + \frac{2}{3} b_m \varepsilon}\right) \explain{By Theorem~\ref{thm: certificate} on $\mathcal{E}_m$} \\
&\leq \delta_{\text{on}} \qquad \text{on } \mathcal{E}_m, \label{eq: conditional delta on}
\end{align}
the last step because \eqref{eq: N requirement} is equivalent to $N \varepsilon^2 \geq \left(2(\sigma_\pi^2 - \Gamma_m) + \frac{2}{3} b_m \varepsilon\right)\log(2/\delta_{\text{on}})$. Both $\Gamma_m$ and $N$ are determined by $\mathcal{D}_m$, as is the event $\mathcal{E}_m$, so \eqref{eq: conditional delta on} is a valid statement about the conditional law. Taking expectations over $\mathcal{D}_m$ and splitting on $\mathcal{E}_m$,
\begin{align}
&\Pr\left(\abs{\hat J_{m, N} - J(\pi)} > \varepsilon\right) \\
=& \E\left[\Pr\left(\abs{\hat J_{m, N} - J(\pi)} > \varepsilon \;\middle|\; \mathcal{D}_m\right) \mathbf{1}_{\mathcal{E}_m}\right] + \E\left[\Pr\left(\abs{\hat J_{m, N} - J(\pi)} > \varepsilon \;\middle|\; \mathcal{D}_m\right) \mathbf{1}_{\mathcal{E}_m^c}\right] \explain{Tower property} \\
\leq& \delta_{\text{on}} \Pr\left(\mathcal{E}_m\right) + \Pr\left(\mathcal{E}_m^c\right) \explain{By \eqref{eq: conditional delta on} and bounding a probability by $1$} \\
\leq& \delta_{\text{on}} + \delta_{\text{off}}, \explain{By $\Pr(\mathcal{E}_m) \leq 1$ and Assumption~\ref{assum: relative error}}
\end{align}
which completes the proof.
\end{proof}

Related offline-learning analyses use KL regularization in
zero-sum \citep{zhang2026pessimism,chen2026offline},
general-sum \citep{chen2026pessimism}, and
$\alpha$-potential games \citep{chen2026fast}.
Their guarantees concern equilibrium learning, whereas our bound
concerns Monte Carlo evaluation of a fixed target policy.

\section{Experiment Details}\label{append:experiment}

\subsection{GridWorld}\label{append:gridworld}

For a Gridworld with size $n$, its width, height,
and time horizon $T$ are all set to $n$.
There are four possible actions: up, down, left, and right. After taking an action, the agent has a $0.9$ probability of moving accordingly and a $0.1$ probability of moving uniformly at random.
If the agent runs into a boundary,
the agent stays in its current location.
The reward function $r(s, a)$ is randomly generated and fixed after generation.
We normalize the rewards across all $(s, a)$ such that $\max_{s,a} r(s,a) = 1$.
We consider a set of randomly generated target policies.
The ground truth policy performance is estimated using the on-policy Monte Carlo method by running each target policy for $10^6$ episodes.
We test two different sizes of the Gridworld with a number of $1,000$ and $27,000$ states.
The offline dataset contains $m = 10^5$ randomly generated tuples.
For a Gridworld of size $n$,
the total amount of possible $(s, t, a, r, s')$ tuples is
$n \times n \times n  \times 4 \times 4 = 16 n^3$.
For the Gridworld of size $27,000$ this gives $16 n^3 = 432{,}000$ possible tuples, so the offline sample size corresponds to a sample-to-support ratio of approximately $23.15\%$. Because tuples are drawn independently and may repeat, this ratio is not the fraction of distinct tuples observed.

We use a one-hot vector representing the position of the agent and a real number representing the current time step as features for the state.
We execute Algorithm \ref{alg: ODI algorithm} to approximate function $r$, $q$, and $\hat{q}$.
We train $r$ using supervised learning by batch stochastic gradient descent.
We train $q$ and $\hat{q}$ using fitted $Q$-learning \citep{le2019batch,chen2024efficient,liu2024efficientmul}.
We split the offline data into a training set and a test set. We
tune all hyperparameters offline
based on the supervised learning loss and fitted $Q$-learning loss on the test set.  With the Adam optimizer \citep{kingma2014adam},
we
search the learning rates in $\qty{2^{-20},2^{-18},\cdots,2^{0}}$ to minimize the loss on the offline data and
use the learning rate $2^{-10}$ on all learning processes.

For the behavior policy search (BPG, \citet{hanna2017data})
and robust on-policy sampling (ROS, \citet{zhong2022robust}) algorithms,
we use the reported parameters from  \citet{hanna2017data} and \citet{zhong2022robust},
since it is not clear how to do hyperparameter tuning for BPG and ROS with only offline data.

\subsection{Online Stochastic Matching}\label{append: online resource allocation}

We formulate the online stochastic matching problem as a finite-horizon Markov Decision Process (MDP; \citet{puterman2014markov}) with state space $\fS$, action space $\fA$, reward function $r:\fS \times \fA \to \mathbb{R}$, and transition kernel $p:\fS \times \fS \times \fA \to [0,1]$. Time evolves in discrete steps $t \in \{0,\ldots,T-1\}$, where $T$ denotes the total number of customer arrivals. The system consists of $n$ heterogeneous resource types with finite capacities.
The state is defined as
$
S_t = (I_t, C_t, t),
$
where $I_t \in \mathbb{Z}_{\ge0}^n$ denotes remaining resource capacities, $C_t \in \{1,\ldots,m\}$ is the arriving customer type, and $t$ is the time index. Each resource $j$ has initial capacity (i.e., number of available units) $\kappa_j$, so $(I_0)_j = \kappa_j$.

At each step, the decision-maker observes the arrival and chooses whether to reject or allocate one unit of an available resource. The feasible action set is
\[
\fA_t = \{0\} \cup \{ j \in \{1,\ldots,n\} \mid (I_t)_j > 0 \},
\]
where $A_t=0$ denotes rejection and $A_t=j>0$ assigns resource $j$.

Transitions decompose into deterministic inventory updates and stochastic arrivals. Given $I_t=i$ and action $A_t=a$, inventory evolves as
\begin{align}
i'_j =
\begin{cases}
i_j - 1, & \text{if } a=j,\\
i_j, & \text{otherwise},
\end{cases}
\qquad \forall j\in\{1,\ldots,n\},
\end{align}
while customer arrivals are sampled independently from a fixed distribution $p_0$, i.e.,
\begin{align}
p(C_{t+1}=c' \mid C_t=c, A_t=a) = p_0(c').
\end{align}

Rewards depend on both customer type and allocated resource,
\begin{align}
r(S_t,a) =
\begin{cases}
R(C_t,a), & a>0,\\
0, & a=0,
\end{cases}
\end{align}
where $R(c,j)$ denotes the utility of assigning resource $j$ to customer type $c$. Rewards between items and customers are randomly initialized.

We evaluate environments with $n=10$ resource types, $m=10$ customer types, and horizon $T=30$. Two capacity settings are considered: single-capacity ($\kappa_j=1$) and multi-capacity ($\kappa_j=2$). Policies are trained using PPO from random initialization, and 30 policies at different training stages are selected as target policies. Offline datasets contain $10$K trajectories generated during training.

The heterogeneous inventory induces a combinatorial state space scaling as $|\fS| \propto T \cdot m \cdot \prod_{j=1}^n (\kappa_j+1)$, a scalability challenge also encountered in multi-item inventory
control \citep{liu2026or}. In particular, the capacity-2 setting contains approximately $17.7$ million states. To handle this complexity, we approximate the state-action value function with a neural network $q_\pi(s,a;\theta)$ with hidden layers of size $(128,128,64)$. Extensions that adapt the value network to new tasks would also
need to account for neural-network plasticity
\citep{wang2026predicting}. Our algorithm is robust to hyperparameters.
All learning rates in Algorithm \ref{alg: ODI algorithm} are tuned offline and are the same $2^{-10}$ across all Gridworld, Online Stochastic Matching, and MuJoCo experiments.

\subsection{MuJoCo}\label{append:mujoco}
\begin{figure}[ht]
\begin{minipage}{0.18\textwidth}
\centering
\includegraphics[width=1\textwidth]{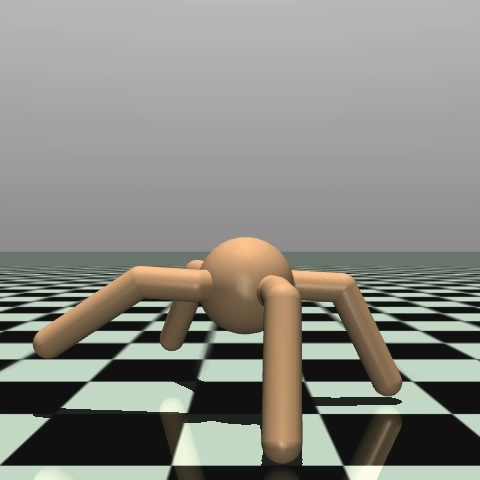}
\end{minipage}
\begin{minipage}{0.18\textwidth}
\centering \includegraphics[width=1\textwidth]{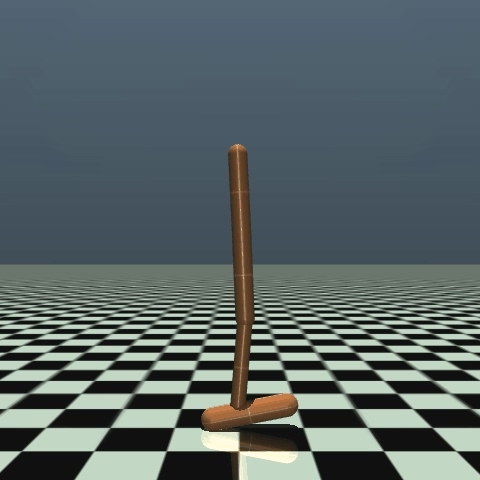}
\end{minipage}
\begin{minipage}{0.18\textwidth}
\centering \includegraphics[width=1\textwidth]{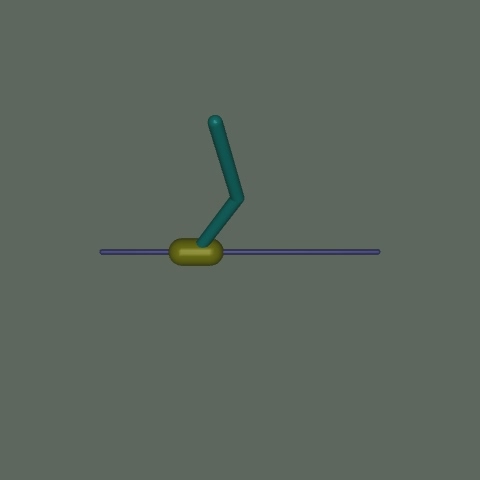}
\end{minipage}
\begin{minipage}{0.18\textwidth}
\centering \includegraphics[width=1\textwidth]{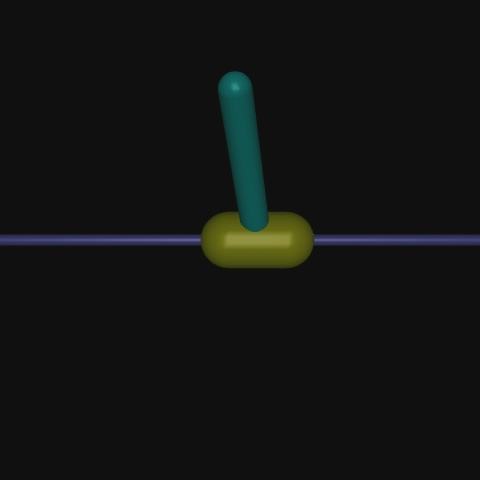}
\end{minipage}
\begin{minipage}{0.18\textwidth}
\centering \includegraphics[width=1\textwidth]{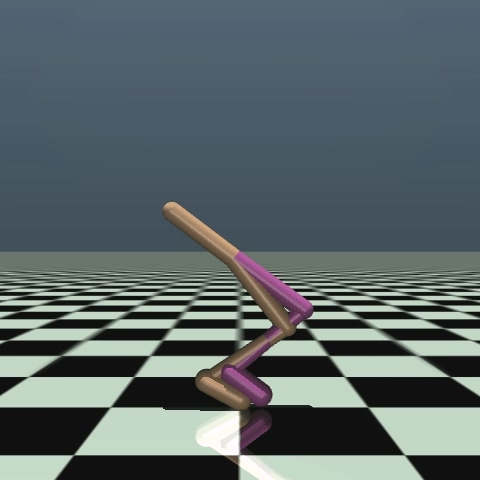}
\end{minipage}
\centering
\caption{
MuJoCo \citep{todorov2012mujoco} robot simulation tasks.
MuJoCo is a physics engine for robotics simulation and contains various stochastic environments. The goal in each environment is to control a robot to achieve different behaviors such as walking, jumping, and balancing.
Environments from left to right are Ant, Hopper, InvertedDoublePendulum,  InvertedPendulum, and Walker. We conducted experiments on those five environments with results reported in Section~\ref{sec:experiment}.
}
\label{fig:cart_pole_image}
\end{figure}

Figure \ref{fig:cart_pole_image} is an introduction to the MuJoCo environments.
We construct $150$ policies ($30$ policies in each environment) with a wide range of performance using
the proximal policy optimization (PPO) algorithm \citep{schulman2017proximal} and the default PPO implementation in \citet{huang2022cleanrl}.
Since our methods are designed for discrete action spaces,
we discretize the first dimension of MuJoCo action space in our experiments.
The remaining dimensions are controlled by the PPO policy and are viewed as part of the environment.
We run each compared algorithm $30$ times for each policy and compute the average and standard error to plot curves in Figure \ref{fig:mujoco}.
To generate offline data,
we add different levels of noise to the target policy and run noisy target policies for $2000$ episodes.
The noise is in the form of a uniformly random policy, and its weight is uniformly randomly sampled from $(0, 0.1]$.
This data generation process simulates the data generated during the training of a policy.
Notably, compared with previous works, we do not need data to be complete trajectories or generated by known policies.
We leave the investigation of entirely irrelevant offline data
in the MuJoCo domain, as well as robustness to transition shifts
\citep{liu2025efficientrobust, chen2026robust}, for future work.
Our algorithm is robust to hyperparameters.
All learning rates in Algorithm \ref{alg: ODI algorithm} are tuned offline and are the same $2^{-10}$ across all Gridworld, Online Stochastic Matching, and MuJoCo experiments.

In  MuJoCo, the episode length varies because of stochasticity in policies and environments.
Because the length of each episode is not fixed, episodes in off-policy estimation may be longer than episodes in on-policy estimation.
In the main text, we use episodes instead of steps
as the $x$-axis mainly to improve readability.
Because
after running $100$ steps,
we might already have a good estimate for a target policy with a length of $10$
but may still not finish a single episode for a target policy with a length of $250$.
Due to the diversity of our target policies,
averaging using steps as the $x$-axis makes the plot conceptually hard to interpret.

\begin{figure}[ht]
\includegraphics[width=1\textwidth]{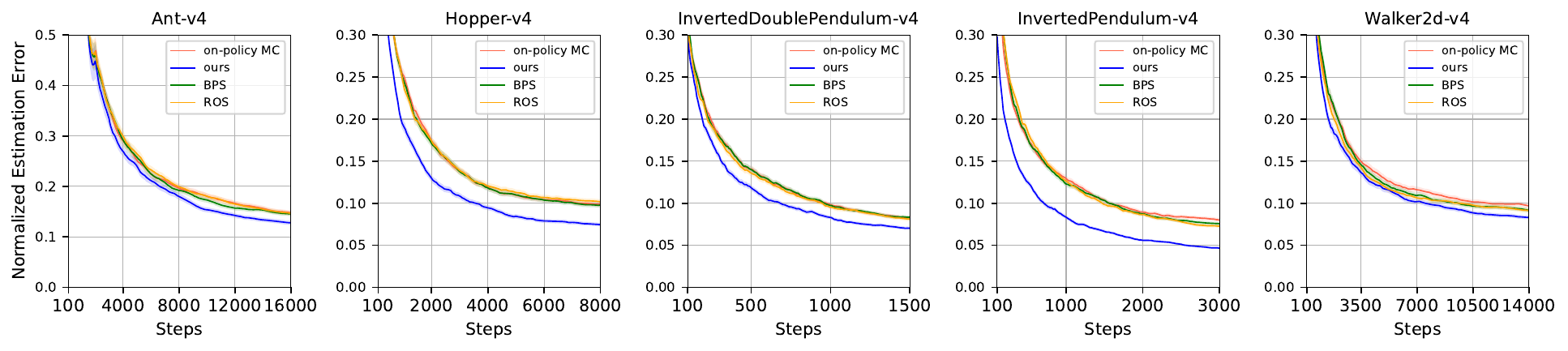}
\centering
\caption{
MuJoCo results using steps as the $x$-axis. We draw each curve from step $100$ because some policies need more than $100$ steps to finish the first episode. All curves are averaged over 900 trials (30 target policies, each having 30 independent runs).
The shaded regions denote standard errors and are invisible because they are too small.
}
\label{fig:carat_pole_step}
\end{figure}

We show the figure with steps as the $x$-axis
in Figure \ref{fig:carat_pole_step}.
Setting steps as the $x$-axis, we linearly interpolate the estimation error across episodes. At each step, we average the estimation error for all tests that have completed the first episode and, thus, have an estimate.
The estimation error is divided by the first estimate of the on-policy estimation to get the normalized estimation error.
Although the normalized estimation error for the on-policy estimation starts from $1$, it may be unstable until around $1000$ steps because different policies get the first estimate at different steps.
However, it is still clear that our off-policy estimator achieves the same accuracy with fewer online steps.

\clearpage

\end{document}